\documentclass{article}

\usepackage{times}
\usepackage{graphicx} %
\usepackage{natbib}
\usepackage{multibib} %
\newcites{sup}{Supplementary References}

\usepackage{times}
\usepackage{epsfig}
\usepackage{graphicx}
\usepackage{subcaption}
\usepackage{booktabs}       %
\usepackage{array}
 \usepackage{multirow}
 \usepackage[table,xcdraw]{xcolor}
\usepackage{algorithm,algcompatible} %
\algnewcommand\algorithmicinput{\textbf{input}}
\algnewcommand\INPUT{\item[\algorithmicinput]}
\algnewcommand{\algorithmicreturn}{\textbf{return}}%
\algnewcommand{\RETURN}[1]{\algorithmicreturn~\ref{#1}}%

\usepackage{enumitem}

\usepackage{amsmath, amssymb, amsthm, bm, color}
\usepackage{notation}
\usepackage{nicefrac}       %

\usepackage[nohyperref, accepted]{icml2016} %

\usepackage{calc} %

\usepackage{hyperref}

\definecolor{mydarkblue}{rgb}{0,0.08,0.45}

\hypersetup{ %
    pdftitle={Minding the Gaps for Block Frank-Wolfe Optimization of Structured SVMs},
    pdfauthor={Anton Osokin, Jean-Baptiste Alayrac, Isabella Lukasewitz, Puneet K. Dokania, Simon Lacoste-Julien},
    pdfsubject={Proceedings of the International Conference on Machine Learning 2016},
    pdfkeywords={adaptive sampling, structured output prediction, Frank-Wolfe, block-coordinate, structured SVM, machine learning, ICML, convex optimization},
    pdfborder=0 0 0,
    pdfpagemode=UseNone,
    colorlinks=true,
    linkcolor=mydarkblue,
    citecolor=mydarkblue,
    filecolor=mydarkblue,
    urlcolor=mydarkblue,
    pdfview=FitH
}

\newlength{\Sfloatsep}
\newlength{\Stextfloatsep}
\newlength{\Sintextsep}
\icmltitlerunning{Minding the Gaps for Block Frank-Wolfe Optimization of Structured SVMs}

\begin{document}

\twocolumn[
\icmltitle{Minding the Gaps for Block Frank-Wolfe Optimization of Structured SVMs}

\icmlauthor{Anton Osokin$^{*,1}$ \enspace Jean-Baptiste Alayrac$^{*,1}$}{first.lastname@inria.fr}
\vspace{1mm}
\icmlauthor{Isabella Lukasewitz$^{1}$ \enspace Puneet K. Dokania$^{2}$ \enspace Simon Lacoste-Julien$^{1}$}{}
\icmladdress{$^{1}\!$ INRIA -- \'Ecole Normale Sup\'erieure, Paris, France \quad $^{2}\!$ INRIA -- CentraleSup\'elec, Ch\^atenay-Malabry, France}
\vspace{-3mm}
\icmladdress{{\small $^*$ Both authors contributed equally.}}
\icmlkeywords{adaptive sampling, structured output prediction, Frank-Wolfe, block-coordinate, structured SVM, machine learning, ICML, convex optimization}

\vspace*{1.5em}
]

\begin{abstract}
In this paper, we propose several improvements on the block-coordinate Frank-Wolfe (BCFW) algorithm from~\citet{lacosteJulien13bcfw} recently used to optimize the structured support vector machine (SSVM) objective in the context of structured prediction, though it has wider applications. 
The key intuition behind our improvements is that the estimates of block gaps maintained by BCFW reveal the block suboptimality that can be used as an \emph{adaptive} criterion.
First, we sample objects at each iteration of BCFW in an adaptive non-uniform way via gap-based sampling.
Second, we incorporate pairwise and away-step variants of Frank-Wolfe into the block-coordinate setting.
Third, we cache oracle calls with a cache-hit criterion based on the block gaps.
Fourth, we provide the first method to compute an approximate regularization path for SSVM\@.
Finally, we provide an exhaustive empirical evaluation of all our methods on four structured prediction datasets.
\end{abstract}

\section{Introduction \label{sec:intro}}

One of the most popular learning objectives for structured prediction is the structured support vector machine~\citep{Taskar2003,Tsochantaridis2005}, which generalizes the classical binary SVM to problems with structured outputs. 
In this paper, we consider the $\ell_2$-regularized $\ell_1$-slack structured SVM, to which we will simply refer as SSVM.
The SSVM method consists in the minimization of the regularized structured hinge-loss on the labeled training set.
The optimization problem of SSVM is of significant complexity and, thus, hard to scale up.
In the literature, multiple optimization methods have been applied to tackle this problem, including cutting-plane methods~\citep{Tsochantaridis2005,Joachims:2009ex} and stochastic subgradient methods~\citep{Ratliff:2007subgradient}, among others. 

Recently, \citet{lacosteJulien13bcfw} proposed the block-coordinate Frank-Wolfe method (BCFW), which is currently one of the state-of-the-art algorithms for SSVM.\footnote{Independently, \citet{branson13} proposed their SVM-IS algorithm which is equivalent to BCFW in some scenarios.}
In contrast to the classical (batch) Frank-Wolfe algorithm~\citep{Frank:1956vp}, BCFW is a randomized block-coordinate method that works on \emph{block-separable} convex compact domains.
In the case of SSVM, BCFW operates in the dual domain and iteratively applies Frank-Wolfe steps on the blocks of dual variables corresponding to different objects of the training set.
Distinctive features of BCFW for SSVM include optimal step size selection leading to the absence of the step-size parameter, convergence guarantees for the primal objective, and ability to compute the duality gap as a stopping criterion.

Notably, the duality gap obtained by BCFW can be written as a sum of block gaps, where each block of dual variables corresponds to one training example.
In this paper, we exploit this property and improve the BCFW algorithm in multiple ways. First, we substitute the standard uniform sampling of objects at each iteration with an adaptive non-uniform sampling.
Our procedure consists in sampling objects with probabilities proportional to the values of their block gaps, giving one of the first fully \emph{adaptive} sampling approaches in the optimization literature that we are aware of.
This choice of sampling probabilities is motivated by the intuition that objects with higher block gaps potentially can provide more improvement to the objective.
We analyze the effects of the gap-based sampling on convergence and discuss the practical trade-offs.

Second, we apply pairwise~\citep{Mitchell:1974uy} and away~\citep{Wolfe:1970wy} steps of Frank-Wolfe to the block-coordinate setting.
This modification is motivated by the fact that batch algorithms based on these steps have linear convergence rates~\citep{LacosteJulien2015linearFW} whereas convergence of standard Frank-Wolfe is sublinear.

Third, we cache oracle calls and propose a gap-based criterion for calling the oracle (cache miss vs.\ cache hit).
Caching the oracle calls was shown do deliver significant speed-ups when the oracle is expensive, e.g., in the case of cutting-plane methods~\cite{Joachims:2009ex}.

Finally, we propose an algorithm to approximate the regularization path of SSVM, i.e., solve the problem for all possible values of the regularization parameter.
Our method exploits block gaps to construct the breakpoints of the path and leads to an $\epsilon$-approximate path.

\textbf{Contributions.} Overall, we make the following contributions:
(i) adaptive non-uniform sampling of the training objects;
(ii) pairwise and away steps in the block-coordinate setting;
(iii) gap-based criterion for caching the oracle calls;
(iv) regularization path for SSVM. The first three contributions are general to BCFW and thus could be applied to other block-separable optimization problems where BCFW could or have been used such as video co-localization~\citep{joulin2014video}, multiple sequence alignment~\citep[App.~B]{alayrac2016MSA}
or structured submodular optimization~\citep{jegelka2013reflection}, among others.

This paper is organized as follows. In Section~\ref{sec:background}, we describe the setup and review the BCFW algorithm. In Section~\ref{sec:contributions}, we describe our contributions: adaptive sampling (Section~\ref{sec:gap_sampling}), pairwise and away steps (Section~\ref{sec:pairwise_away_steps}), caching (Section~\ref{sec:caching}).
In Section~\ref{sec:reg_path}, we explain our algorithm to compute the regularization path.
We discuss the related work in the relevant sections of the paper. 
Section~\ref{sec:experiments} contains the experimental study of the methods.
The code and datasets are available at our project webpage.\footnote{\resizebox{0.92\columnwidth}{!}{ \url{http://www.di.ens.fr/sierra/research/gapBCFW/}}}

\section{Background \label{sec:background}}

\subsection{Structured Support Vector Machine (SSVM) \label{sec:ssvm}}
In structured prediction, we are given an input~$\inputvarv\in\inputdomain$, and the goal is to predict a structured object~$\outputvarv\in\outputdomain(\inputvarv)$ (such as a sequence of tags). In the standard setup for structured SVM (SSVM)~\citep{Taskar2003,Tsochantaridis2005}, we assume that prediction is performed with a linear model $h_{\weightv}(\inputvarv) = \argmax_{\outputvarv\in\outputdomain(\inputvarv)}\langle\weightv, \featuremapv(\inputvarv,\outputvarv) \rangle$ parameterized by the weight vector~$\weightv$ where the structured feature map $\featuremapv(\inputvarv,\outputvarv) \in \R^d$ encodes the relevant information for input/output pairs. We reuse below the notation and setup from~\citet{lacosteJulien13bcfw}. Given a labeled training set~$\data = \{(\inputvarv_i,\outputvarv_i)\}_{i=1}^n$, the parameters~$\weightv$ are estimated by solving a convex non-smooth optimization problem\vspace{-1.5mm}
\begin{equation}
    \label{eq:svmstruct_nslack_primal_nonsmooth}
    \min_{\weightv} \quad \frac{\regularizerweight}{2}\norm{\weightv}^2 +
    \frac1n \sum_{i=1}^n \tilde{H}_i(\weightv)\vspace{-1.5mm}
\end{equation}
where $\regularizerweight$ is the regularization parameter and~$\tilde{H}_i(\weightv)$ is the structured hinge loss defined using the \emph{loss-augmented decoding} subproblem (or \emph{maximization oracle}):
\begin{equation}\label{eq:subproblem_loss_augm}
\text{\parbox[t]{3em}{`max \\oracle'}} \quad 
  \tilde{H}_i(\weightv) := \max_{\outputvarv\in\outputdomain_i} \
    \underbrace{%
    \errorterm_i(\outputvarv)
    - \langle \weightv,
    \featuremapdiffv_i(\outputvarv)
    \rangle
    }_{=:\, H_i(\outputvarv;\weightv)}.
\end{equation}

Here $\featuremapdiffv_i(\outputvarv):= \featuremapv(\inputvarv_i,\outputvarv_i) - \featuremapv(\inputvarv_i,\outputvarv)$, $\outputdomain_i := \outputdomain(\inputvarv_i)$, and $\errorterm_i(\outputvarv) := \errorterm(\outputvarv_i,\outputvarv)$ denotes the task-dependent structured error of predicting an output~$\outputvarv$ instead of the observed output~$\outputvarv_i$ (e.g., a Hamming distance between two sequences).

\paragraph{Dual formulation.} The negative of a Fenchel dual for objective~\eqref{eq:svmstruct_nslack_primal_nonsmooth} can be written as\vspace{-1mm}
\begin{align}
    \label{eq:svmstruct_nslack_dual} %
    \min_{\substack{ \dualvarv\in\R^{m} \\  \dualvarv \succcurlyeq 0}} \quad  f(\dualvarv) \;:=&  \;\;
    \frac{\regularizerweight}{2}
    \big\| A\dualvarv \big\|^2
    - \bv^\transpose\dualvarv
    \\[-3mm]
    \text{s.t.} \quad &  \;
      \textstyle\sum_{\outputvarv \in \outputdomain_i}  \dualvar_i(\outputvarv) = 1 ~~~\forall i\in[n] \ \notag
\end{align}
where $\dualvar_i(\outputvarv)$, $i\in[n]:=\{1,\ldots,n\}$, $\outputvarv \in \outputdomain_i$ are the dual variables.
The matrix~$A\in\R^{d\times m}$ consists of the $m := \sum_i m_i = \sum_i |\outputdomain_i|$ columns $A := \SetOf{\frac1{\regularizerweight n} \featuremapdiffv_i(\outputvarv) \in\R^d}{i\in[n],\outputvarv \in \outputdomain_i}$, and the vector $\bv \in \R^m$ is given by
$\bv:= \left(\frac1n \errorterm_i(\outputvarv) \right)_{i\in[n],\outputvarv\in\outputdomain_i}$.

In SSVM (as for the standard SVM), the Karush-Kuhn-Tucker (KKT) optimality conditions can give the primal variables~\mbox{$
\weightv(\dualvarv) = A\dualvarv  = \sum_{i,\,\outputvarv \in \outputdomain_i} \dualvar_i(\outputvarv)  \frac{\featuremapdiffv_i(\outputvarv)}{\regularizerweight n}
$} corresponding to the dual variables $\dualvarv$ (see, e.g., \citep[App.~E]{lacosteJulien13bcfw}).
The gradient of $f$ then takes a simple form $\nabla f(\dualvarv) = \regularizerweight A^\transpose A\dualvarv - \bv = \regularizerweight A^\transpose \weightv - \bv$; its \mbox{$(i,\outputvarv)$-th} component equals $-\frac{1}{n} H_i(\outputvarv; \weightv)$.

\subsection{Block Coordinate Frank-Wolfe method (BCFW)\label{sec:bcfw}}

\begin{algorithm}[t]
    \caption{Block-Coordinate Frank-Wolfe (BCFW) algorithm for structured SVM}%
    \label{alg:FW_product_SVM}
\begin{algorithmic}[1]
        \STATE Let $\weightv^{(0)}\!:=\!{\weightv_i}^{(0)}\!:=\!\0$; \; $\ell^{(0)}\!:=\!{\ell_i}^{(0)}\!:=\!0$
       \FOR{$k:=0,\dots,\infty$}s
                \STATE Pick $i$ at random in $\{1,\ldots,n\}$ \label{alg:FW_product_SVM:randomSample}
                \STATE Solve $\outputvarv_i^* := \displaystyle\argmax_{\outputvarv\in\outputdomain_i} \ H_i(\outputvarv;\weightv^{(k)})$ \label{alg:FW_product_SVM:max_oracle}
                \STATE Let $\weightv_{\sv} := \frac1{\regularizerweight n} \featuremapdiffv_i(\outputvarv_i^*)$ \;
                and \; $\ell_{\sv} := \frac1n \errorterm_i(\outputvarv_i^*)$
               \STATE Let $g_i^{(k)} :=  \regularizerweight (\weightv_i^{(k)}-\weightv_{\sv})^\transpose \weightv^{(k)} - \ell_i^{(k)} + \ell_{\sv}$\label{alg:FW_product_SVM:block_gap}
               \STATE {\small Let $\stepsize := \frac{ g_i^{(k)}}{ \regularizerweight \|\weightv_i^{(k)}-\weightv_{\sv}\|^2}$~and clip to $[0,1]$}\label{alg:FW_product_SVM:line_search}
                \STATE Update ${\weightv_i}^{(k+1)}:= (1-\stepsize){\weightv_i}^{(k)}+\stepsize \,\weightv_{\sv}$
                \STATE {\small~~~~~~~and~ ${\ell_i}^{(k+1)}:= (1-\stepsize){\ell_i}^{(k)}+\stepsize\, \ell_{\sv}$}
                \STATE Update $\weightv^{(k+1)}\;:= \weightv^{(k)} + {\weightv_i}^{(k+1)} - {\weightv_i}^{(k)}$
                \STATE {\small~~~~~~~and~~ $\ell^{(k+1)}:= ~\ell^{(k)}+{\ell_i}^{(k+1)} \ \  - {\ell_i}^{(k)}$}
        \ENDFOR
\end{algorithmic}
\end{algorithm}

We give in Alg.~\ref{alg:FW_product_SVM} the BCFW algorithm from~\citet{lacosteJulien13bcfw} applied to problem~\eqref{eq:svmstruct_nslack_dual}. It exploits the block-separability of the domain~$\domain := \simplex_{|\outputdomain_1|}\times\mathellipsis\times\simplex_{|\outputdomain_n|}$ for problem~\eqref{eq:svmstruct_nslack_dual} and sequentially applies the Frank-Wolfe steps to the blocks of the dual variables~$\dualvarv_{(i)}\in\domain^{(i)} := \simplex_{|\outputdomain_i|}$.

While BCFW works on the dual~\eqref{eq:svmstruct_nslack_dual} of SSVM, it only maintains explicitly the primal variables via the relationship $\weightv(\dualvarv)$. 
Most importantly, the Frank-Wolfe linear oracle on block $i$ at iterate $\dualvarv^{(k)}$ is equivalent to the max oracle~\eqref{eq:subproblem_loss_augm} at the corresponding weight vector $\weightv^{(k)}:= A\dualvarv^{(k)}$ \citep[App.~B.1]{lacosteJulien13bcfw} (see line~\ref{alg:FW_product_SVM:max_oracle} of Alg.~\ref{alg:FW_product_SVM}):
\begin{equation}
\label{eq:block_max_oracle}
\!\!\!\!\displaystyle\max_{\sv_{(i)}\in \domain^{(i)}} \!\!\big\langle \sv_{(i)}, -\nabla_{(i)} f(\dualvarv^{(k)}) \big\rangle
=
\frac1n\max_{\outputvarv\in\outputdomain_i} H_i(\outputvarv;\weightv^{(k)}).
\end{equation}
Here, the operator~$\nabla_{(i)}$ denotes the partial gradient corresponding to the block~$i$, i.e., $\nabla f = (\nabla_{(i)} f)_{i=1}^n$. Note that each~$\argmax$ of the r.h.s. of~\eqref{eq:block_max_oracle}, $\outputvarv_{(i)}^*$, corresponds to a corner~$\sv_{(i)}^*$ of the polytope~$\domain^{(i)}$ maximizing the l.h.s. of~\eqref{eq:block_max_oracle}.

As the objective~\eqref{eq:svmstruct_nslack_dual} is quadratic, the optimal step size that yields the maximal improvement in the chosen direction $\sv_{(i)}^*- \dualvarv^{(k)}_{(i)}$ can be found analytically (Line~\ref{alg:FW_product_SVM:line_search} of Alg.~\ref{alg:FW_product_SVM}).

\subsection{Duality gap \label{sec:duality_gap}}
At each iteration, the batch Frank-Wolfe algorithm~\citep{Frank:1956vp}, \citep[Section~3]{lacosteJulien13bcfw} computes the following quantity, known as the \emph{linearization duality gap} or \emph{Frank-Wolfe gap}:
\begin{equation}\label{eq:duality_gap}
  g(\dualvarv) := \max_{\sv \in \domain} \,\langle \dualvarv - \sv, \nabla f(\dualvarv) \rangle
  = \langle \dualvarv - \sv^*, \nabla f(\dualvarv) \rangle .
\end{equation}
It turns out that this Frank-Wolfe gap exactly equals the Lagrange duality gap between the dual objective~\eqref{eq:svmstruct_nslack_dual} at a point~$\dualvarv$ and the primal objective~\eqref{eq:svmstruct_nslack_primal_nonsmooth} at the point~$\weightv(\dualvarv)=A\dualvarv$ \citep[App.~B.2]{lacosteJulien13bcfw}.

Because of the separability of~$\domain$, the Frank-Wolfe gap~\eqref{eq:duality_gap} can be represented here as a sum of block gaps~$g_i(\dualvarv)$, $g(\dualvarv) = \sum_{i=1}^n g_i(\dualvarv)$, where 
\begin{equation}
\label{eq:block_gap}
g_i(\dualvarv) := \max_{\sv_{(i)}\in \domain^{(i)}} \left\langle\dualvarv_{(i)} - \sv_{(i)}, \nabla_{(i)} f(\dualvarv) \right\rangle.
\end{equation}
Block gaps can be easily computed using the quantities maintained by Alg.~\ref{alg:FW_product_SVM} (see line~\ref{alg:FW_product_SVM:block_gap}).

Finally, we can rewrite the block gap in the form
\begin{equation}\label{eq:duality_gap_withH}
g_i(\dualvarv)
\!=\!
\frac{1}{n}  \bigg(\max_{\outputvarv \in \outputdomain_i} H_i(\outputvar; \weightv) -\!\!\sum_{\outputvarv \in \outputdomain_i}\dualvar_i(\outputvarv)H_i(\outputvarv; \weightv)  \bigg)
\end{equation}
providing understandable intuition of when the block gap equals zero. This is the case when all the support vectors, i.e., labelings corresponding to  $\dualvar_i(\outputvarv) > 0$, are tied solutions of the max oracle~\eqref{eq:block_max_oracle}.

\subsection{Convergence of BCFW \label{sec:convergence_bcfw}}
\citet{lacosteJulien13bcfw} prove the convergence of the BCFW algorithm at a rate $\bigO(\frac{1}{k})$.

\begin{theorem}[\citet{lacosteJulien13bcfw}, Theorem~2] \label{thm:convergence_FW_product}
For each $k\ge 0$, the iterate\footnote{Note that Alg.~\ref{alg:FW_product_SVM} does not maintain iterates~$\dualvarv^{(k)}$ explicitly.
They are stored in the form of~$\weightv^{(k)} = A \dualvarv^{(k)}$\!.} 
$\dualvarv^{(k)}$ of %
Alg.~\ref{alg:FW_product_SVM} satisfies
$
\E\big[f(\dualvarv^{(k)})\big] - f(\dualvarv^*) \le \frac{2n}{k+2n}\big(\CfTotal+ h_0\big) \, ,
$
where $\dualvarv^*\in \domain$ is a solution of the problem~(\ref{eq:svmstruct_nslack_dual}), $h_0 := f(\dualvarv^{(0)}) - f(\dualvarv^*)$ is the suboptimality at the starting point of the algorithm, $\CfTotal := \sum_{i=1}^n \Cf^{(i)}$ is the sum of the curvature constants\footnote{For the definition of curvature constant, see Definition~\ref{def:curvature_const} in App.~\ref{app:descent_lemma} or \citep[App.~A]{LacosteJulien2015linearFW}} of $f$ with respect to the domains~$\domain^{(i)}$ of individual blocks.
The expectation is taken over the random choice of the block~$i$ at iterations $1,\dots,k$ of the algorithm.
\end{theorem}

The proof of Theorem~\ref{thm:convergence_FW_product} crucially depends on a standard \emph{descent lemma} applied to a block, stating that at each iteration of BCFW, for any picked block~$i$ and any scalar~$\stepsize \in [0,1]$, the following inequality holds:
\begin{equation}
\label{eq:block_descent_line_search}
  f(\dualvarv^{(k+1)})
  \le f(\dualvarv^{(k)}) - \stepsize g_i(\dualvarv^{(k)}) + \frac{\stepsize^2}{2}  \Cf^{(i)}.
\end{equation}
We rederive inequality~\eqref{eq:block_descent_line_search} as Lemma~\ref{lem:block_step_improvement} in App.~\ref{app:descent_lemma}.
Note that~$\dualvarv^{(k+1)} \in \domain$ is defined by a line search, which is why the bound~\eqref{eq:block_descent_line_search} holds for any scalar~$\stepsize \in [0,1]$.

Taking the expectation of~\eqref{eq:block_descent_line_search} w.r.t.\ the random choice of block~$i$ (sampled uniformly on~$[n]$), we get the inequality
\begin{equation}
\label{eq:expected_descent_line_search}
\E\big[ f( {\scriptstyle \dualvarv^{(k+1)}})\,|\, {\scriptstyle \dualvarv^{(k)} }\big]
  \le f(\dualvarv^{(k)}) - \frac{\stepsize}{n} g(\dualvarv^{(k)}) + \frac{\stepsize^2}{2n} \CfTotal
\end{equation}
which can be used to get the convergence theorem.

\section{Block gaps in BCFW \label{sec:contributions}}
In this section, we propose three ways to improve the BCFW algorithm: adaptive sampling (Sec.~\ref{sec:gap_sampling}), pairwise and away steps (Sec.~\ref{sec:pairwise_away_steps}) and caching (Sec.~\ref{sec:caching}).

\subsection{Adaptive non-uniform sampling \label{sec:gap_sampling}}

\setlength{\textfloatsep}{\Stextfloatsep} %
\begin{figure*}
\centering
\begin{subfigure}[b]{0.28\textwidth}
\includegraphics[width=\textwidth]{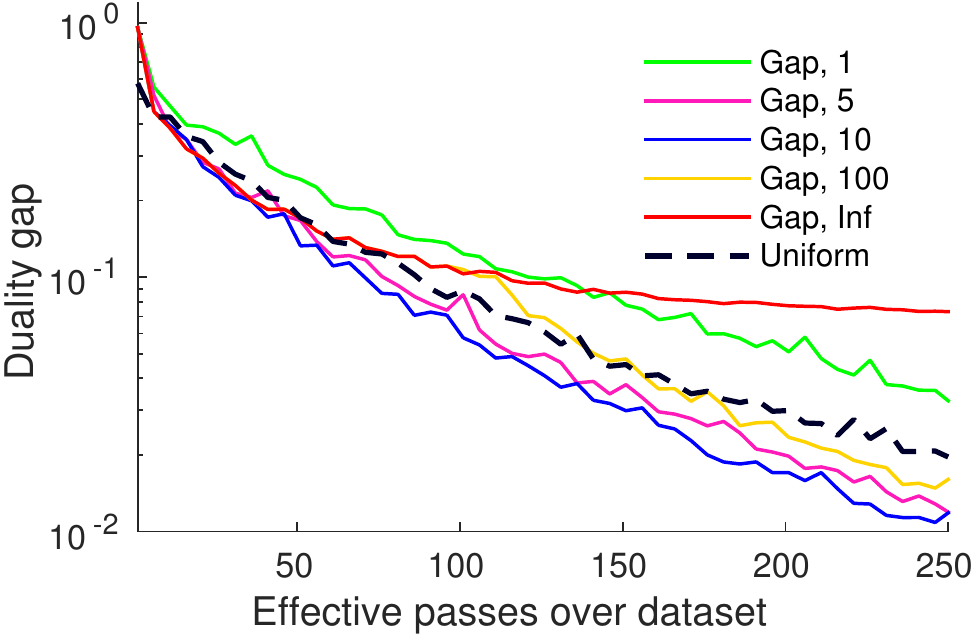}
\caption{Convergence plots\label{fig:gapSampling_a}\vspace{-3.5mm}}
\end{subfigure}
\qquad 
\begin{subfigure}[b]{0.28\textwidth}
\includegraphics[width=\textwidth]{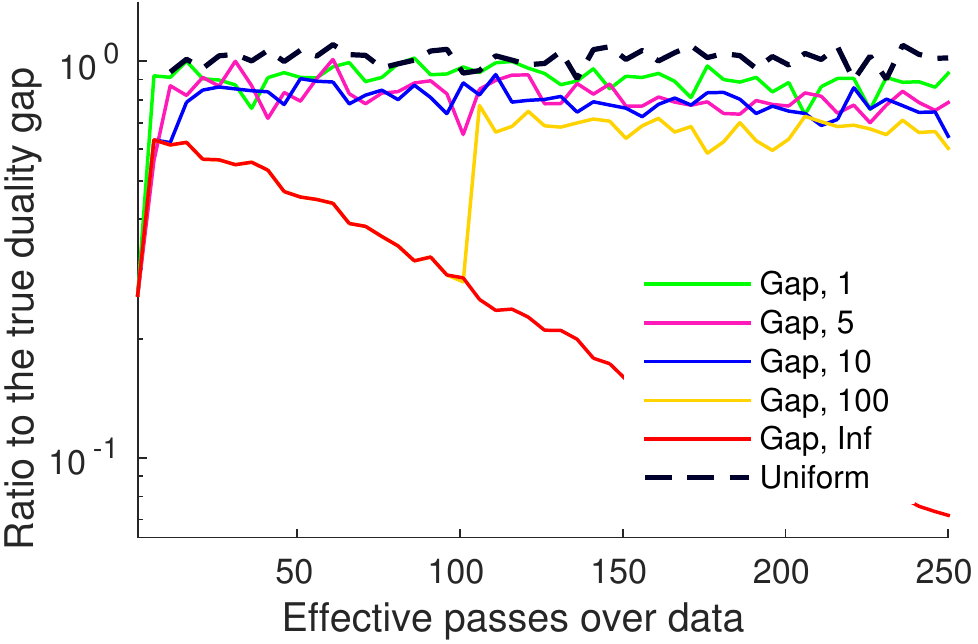}
\caption{Quality of gap estimates\label{fig:gapSampling_b}\vspace{-3.5mm}}
\end{subfigure}
\qquad 
\begin{subfigure}[b]{0.28\textwidth}
\includegraphics[width=\textwidth]{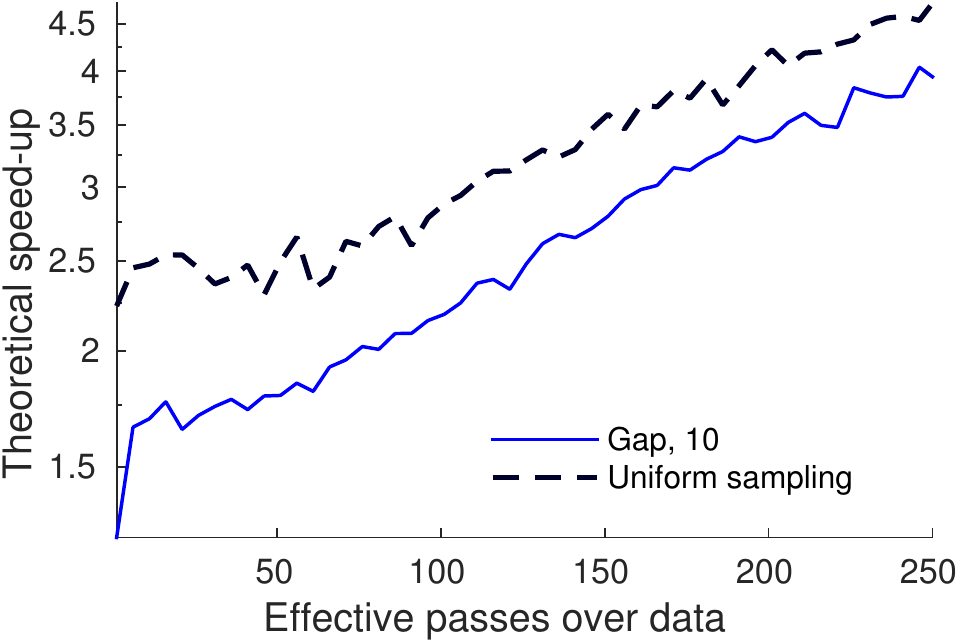}
\caption{Theoretical improvement\label{fig:gapSampling_c}\vspace{-3.5mm}}
\end{subfigure}
\caption{Plot~(a) shows exploitation/staleness trade-off for the gap sampling approach. We report the duality gap against the number of effective passes over the data for uniform sampling and for gap sampling with the different frequencies of batch passes updating the gap estimates (every pass over data, every 5, 10, 100 passes, no batch updates). Plot~(b) shows the quality of heuristic gap estimates obtained by the same methods. We report the ratio of the heuristic gap estimate to the true gap value. Plot~(c) shows the factor of improvement of exact gap sampling predicted by Theorem~\ref{thm:convTheoremGapSampling} for real gaps appearing during a run of BCFW with either uniform or gap sampling.\vspace{-3.5mm}
    \label{fig:gapSampling}}
\end{figure*}

\paragraph{Motivation.}
When optimizing finite sums such as~\eqref{eq:svmstruct_nslack_primal_nonsmooth}, it is often the case that processing some summands does not lead to significant progress of the algorithm. 
At each iteration, the BCFW algorithm selects a training object and performs the block-coordinate step w.r.t.\ the corresponding dual variables. 
If these variables are already close to being optimal, then BCFW does not make significant progress at this iteration. Usually, it is hard to identify whether processing the summand would lead to an improvement without actually doing computations on it. The BCFW algorithm obtains at each iteration the block gap~\eqref{eq:block_gap} quantifying the suboptimality on the block. In what follows, we use the block gaps to randomly choose a block (an object of the training set) at each iteration in such a way that the blocks with larger suboptimality are sampled more often (the sampling probability of a block is proportional to the value of the current gap estimate).

\paragraph{Convergence.}

Assume that at iteration~$k$ of Alg.~\ref{alg:FW_product_SVM}, we have the probability~$p_i^{(k)}$ 
of sampling block~$i$.
By minimizing the descent lemma bound~\eqref{eq:block_descent_line_search} w.r.t.\ $\stepsize$ for each~$i$ independently under the assumption that $g_i(\dualvarv^{(k)}) \leq \Cf^{(i)}$, and then taking the conditional expectation w.r.t.\ $i$, we get
\begin{equation}
\label{eq:expectedImprovement}
\E\big[ f( {\scriptstyle \dualvarv^{(k+1)}})\,|\, {\scriptstyle \dualvarv^{(k)} }\big]
  \le f(\dualvarv^{(k)}) - \frac{1}{2} \sum_{i=1}^n \p_i^{(k)} \frac{g_i^2(\dualvarv^{(k)})}{\Cf^{(i)}}.
\end{equation}
Intuitively, by adapting the probabilities $p_i^{(k)}$, we can obtain a better bound on the expected improvement of~$f$.
In the ideal scenario, one would choose deterministically the block~$i$ with the maximal value of $g_i^2(\dualvarv^{(k)})/\Cf^{(i)}$.

In practice, the curvature~$\Cf^{(i)}$ is unknown, and having access to all $g_i( {\scriptstyle \dualvarv^{(k)} })$'s at each step is prohibitively expensive.
However, the values of the block gaps obtained at the previous iterations can serve as estimates of the block gaps at the current iteration.
We use them in the following \emph{non-uniform} gap sampling scheme:
$
p_i^{(k)} \propto {g}_i(\dualvarv^{(k_i)}).
$ 
 where $k_i$ records the last iteration at which the gap~$i$ was computed.
Alg.~\ref{alg:FW_product_SVM_gapsampling} in App.~\ref{app:algorithms} summarizes the method.

We also motivate this choice by Theorem~\ref{thm:convTheoremGapSampling} below which shows that BCFW with (exact) gap sampling converges with a better constant in the rate than BCFW with uniform sampling when the gaps are non-uniform enough (and is always better when the curvatures~$\Cf^{(i)}$'s are uniform). See the proof and discussion in App.~\ref{app:proof_gap_sampling}.

\begin{theorem}\label{thm:convTheoremGapSampling}
Consider the same notation as in Theorem~\ref{thm:convergence_FW_product}.
Assume that at each iterate~$\dualvarv^{(k)}$, BCFW with gap sampling (Alg.~\ref{alg:FW_product_SVM_gapsampling}) has access to the exact values of the block gaps. 
Then, at each iteration, it holds that
$
\E \big[f(\dualvarv^{(k)})\big] - f(\dualvarv^*) \le \frac{2n}{k+2n}\big(\CfTotal \nonuniformityTotal + h_0\big)
$
where the constant~$\nonuniformityTotal$ is an upper bound on $\E \Big[ \frac{\nonuniformity(\Cf^{(:)})}{\nonuniformity(\bm{g}_{:}(\dualvarv^{(k)}))^3} \Big]$. The \emph{non-uniformity measure} $\nonuniformity(\bm{x})$ of a vector~$\bm{x}\in\R^n_+$ is defined as
$
\nonuniformity(\bm{x}) 
:=  
\sqrt{1+n^2\Var \big[\bm{p}\big]}
$
where~$\bm{p} := \frac{\bm{x}}{\Vert \bm{x} \Vert_1}$ is the probability vector obtained by normalizing~$\bm{x}$.
\end{theorem}

\vspace{-2mm}
\paragraph{Adaptive procedure.} 
Note that this procedure is \emph{adaptive}, meaning that the criterion for choosing an object to optimize changes during the optimization process.
Our adaptive approach differs from more standard techniques that sample proportional to the Lipschitz constants, as e.g., in~\citet{Nesterov:2012fa}.
In App.~\ref{app:toy_example}, we illustrate the advantage of this property by constructing an example where the convergence of gap sampling can be shown \emph{tightly} to be $n$ times faster than when using Lipschitz sampling.

\vspace{-2mm}
\paragraph{Exploitation versus staleness trade-off.}
In practice, having access to the exact block gaps is intractable because it requires a full pass over the dataset after every block update.
However, we have access to the estimates of the block gaps computed from past oracle calls on each block.
Notice that such estimates are outdated, i.e., might be quite far from the current values of the block gaps. We call this effect~``staleness''.
One way to compensate staleness is to refresh the block gaps by doing a full gap computation (a pass over the dataset) after several block-coordinate passes. These gap computations were often already done during the optimization process, e.g., to monitor convergence.

We demonstrate the exploitation/staleness trade-off in our exploratory experiment reported in Figure~\ref{fig:gapSampling}.
On the OCR dataset~\citep{Taskar2003}, we run the gap sampling algorithm with a gap computation pass after 1, 5, 10 and 100 block-coordinate passes (Gap 1, Gap 5, Gap 10, Gap 100) and without any gap computation passes (Gap Inf).
As a baseline, we use BCFW with uniform sampling (Uniform). Figure~\ref{fig:gapSampling_a} reports the duality gap after each number of effective passes over the data.\footnote{An effective pass consists in $n$ calls to the max oracle.}
Figure~\ref{fig:gapSampling_b} shows the ratio of the exact value of the duality gap to the heuristic gap estimate defined as the sum of the current gap estimates.
We observe that when the gap computation is never run, the gap becomes significantly underestimated and the algorithm does not converge.
On another extreme, when performing the gap computation after each pass of BCFW, the algorithm wastes too many computations and converges slowly. Between the two extremes, the method is not very sensitive to the parameter (we have tried 5, 10, 20, 50) allowing us to always use the value of~10.

Comparing adaptive methods to BCFW with uniform sampling, we observe a faster convergence.
Figure~\ref{fig:gapSampling_c} reports the improvement of gap sampling at each iteration w.r.t. uniform sampling that is predicted by Theorem~\ref{thm:convTheoremGapSampling}. 
Specifically, we report the quantity $\nicefrac{\nonuniformity(\bm{g}_{:}(\dualvarv^{(k)}))^3}{\nonuniformity(\Cf^{(:)})}$ with the block gaps estimated at the runs of BCFW with both uniform and gap sampling schemes. To estimate the curvature constants $\Cf^{(i)}$, we use the upper bounds proposed by~\citet[App.~A]{lacosteJulien13bcfw}: $\frac{4R_i^2}{\regularizerweight n^2}$ where
$R_i:= \max_{\outputvarv \in \outputdomain_i}\norm{\featuremapdiffv_i(\outputvarv)}_2$.
We approximate $R_i$ by picking the largest value~$\norm{\featuremapdiffv_i(\outputvarv)}_2$ corresponding to a labeling~$\outputvarv$ observed within the run of BCFW.

\paragraph{Related work.}
Non-uniform sampling schemes have been used over the last few years to improve the convergence rates of well known randomized algorithms~\cite{Nesterov:2012fa,needell2014,ZhaoImportanceSampling_ICML15}.
Most of these approaches use the Lipschitz constants of the gradients to sample more often functions for which gradient changes quickly.
This approach has two main drawbacks.
First, Lipschitz constants are often unknown and heuristics are needed to estimate them.
Second, such schemes are not adaptive to the current progress of the algorithm.
To the best of our knowledge, the only other approach that uses an \emph{adaptive} sampling scheme to guide the optimization with convergence guarantees is the one from~\citet{Csiba15adaSDCA}, in the context of the stochastic dual coordinate ascent (SDCA) algorithm.
A cyclic version of BCFW has been analyzed by~\citet{beck2015cyclicBCFW} while~\citet{wang2014parallelBCFW} analyzed its mini-batch form.

\subsection{Pairwise and away steps \label{sec:pairwise_away_steps}}
\paragraph{Motivation.}
In the batch setting, the convergence rate of the Frank-Wolfe algorithm is known to be sublinear
when the solution is on the boundary~\citep{Wolfe:1970wy}, as is the case for SSVM\@.
Several modifications have been proposed in the literature to address this issue.
All these methods replace (or complement) the FW step with a step of another type: pairwise step~\citep{Mitchell:1974uy}, away step~\citep{Wolfe:1970wy}, fully-corrective step~\citep{Holloway:1974:FCFW} (see~\citet{LacosteJulien2015linearFW} for a recent review and the proof that all these methods have a linear rate on the objective~\eqref{eq:svmstruct_nslack_dual} despite not being strongly convex).
A common feature of these methods is the ability to remove elements of the active set (support vectors in the case of SSVM) in order to reach the boundary, unlike FW which oscillates while never completely reaching the boundary. As we expect the solution of SSVM to be sparse, these variants seem natural in our setting.
In the rest of this section, we present the pairwise steps in the block-coordinate setting (the away-step version is described in Alg.~\ref{alg:aFW_product_SVM_away_steps} of App.~\ref{alg:aFW_product_SVM_away_steps}).

\paragraph{Pairwise steps.}
A (block) pairwise step consists in \emph{removing} mass from the \emph{away corner} on block~$i$ and transferring it to the 
\emph{FW corner} obtained by the max oracle~\eqref{eq:block_max_oracle}. The away corner is the element of the active set~$\activeS_i := \{\outputvarv \in \outputdomain_i \ | \ \alpha_i(\outputvarv)>0\} \subseteq \outputdomain_i$ worst aligned with the current descent direction, which can be found by solving~$\outputvarv_i^{\awayv} := \argmin\nolimits_{\outputvarv\in\activeS_i} H_i(\outputvarv;\weightv)$.
This does not require solving a combinatorial optimization problem because the size of the active set is typically small, e.g., bounded by the number of iterations performed on the block~$i$.
Analogously to the case of BCFW, the optimal step size~$\stepsize$ for the pairwise step can be computed explicitly by
clipping $\frac{\regularizerweight (\weightv_{\awayv}-\weightv_{\sv})^\transpose \weightv^{(k)} + \ell_{\sv} - \ell_{\awayv}}{ \regularizerweight \|\weightv_{\awayv}-\weightv_{\sv}\|^2}$ to the segment $[0,\dualvar_i^{(k)}(\outputvarv_i^{\awayv})]$ where the upper bound~$\dualvar_i^{(k)}(\outputvarv_i^{\awayv})$ corresponds to the mass of the away corner before the step and the quantities 
$\weightv_{\sv} := \frac1{\regularizerweight n} \featuremapdiffv_i(\outputvarv_i^*)$, $\ell_{\sv} := \frac1n \errorterm_i(\outputvarv_i^*)$ and 
$\weightv_{\awayv} := \frac1{\regularizerweight n} \featuremapdiffv_i(\outputvarv_i^a)$, $\ell_{\awayv} := \frac1n \errorterm_i(\outputvarv_i^a)$ represent the FW and away corners.
Alg.~\ref{alg:pFW_product_SVM_pairwise} in App.~\ref{app:algorithms} summarizes the block-coordinate pairwise Frank-Wolfe (BCPFW) algorithm.

In contrast to BCFW, the steps of BCPFW cannot be expressed in terms of the primal variables~$\weightv$ only, thus it is required to explicitly store the dual variables~$\dualvarv_i$. Storing the dual variables is feasible, because they are extremely sparse, but still can lead to computational overheads caused by the maintenance of the data structure.

The standard convergence analysis for pairwise and away-step FW cannot be easily extended to BCFW. We show the geometric decrease of the objective in Theorem~\ref{thm:BCPFW} of App.~\ref{app:BCPFWconvergence} only when \emph{no} block would have a drop step (a.k.a.\ `bad step'); a condition that cannot be easily analyzed due to the randomization of the algorithm. We believe that novel proof techniques are required here, even though we did observe empirically a linear convergence rate when $\regularizerweight$ is big enough.

\paragraph{Related work.}
\citet[Alg.~4]{nanculef2014} used the pairwise FW algorithm on the dual of binary SVM (in batch mode, however). It is related to classical working set algorithms, such as the SMO algorithm used to train SVMs~\cite{platt1999SMO}, also already applied on SSVMs in~\citet[Ch.~6]{taskar04thesis}.
\citet{franc2014fasole} recently proposed a version of pairwise FW for the block-coordinate setting.
Their SDA-WSS2 algorithm uses a different criterion for choosing the away corner than BCPFW:
instead of minimizing~$H_i$ over the active set~$\activeS_i$, they compute the improvement for all possible away corners and pick the best one.
Their FASOLE algorithm also contains a version of gap sampling in the form of variable shrinking: if a block gap becomes small enough, the block is not visited again, until all the counters are reset.

\subsection{Caching \label{sec:caching}}

\paragraph{Motivation.}
At each step, the BCFW and BCPFW algorithms call the max oracle to find the Frank-Wolfe corner.
In cases where the max oracle is expensive, this step becomes a computational bottleneck.
A natural idea to overcome this problem consists in using a ``cheaper oracle'' most of the time hoping that the resulting corner would be good enough.
Caching the results of the max oracle implements this idea by reusing the previous calls of the max oracle to store potentially promising corners.

\paragraph{Caching.}
The main principle of caching consists in maintaining a working set $\cache_i \subset \outputdomain_i$ of labelings/corners for each block $i$, where $|\cache_i|\ll|\outputdomain_i|$.
A \emph{cache oracle} obtains the \emph{cache corner} defined as a corner from the working set best aligned with the descent direction, i.e., $\outputvarv_i^c := \argmax_{\outputvarv\in\cache_i} \ H_i(\outputvarv;\weightv)$. 
If the obtained cache corner passes a \emph{cache hit criterion}, i.e., there is a \emph{cache hit}, we do a Frank-Wolfe (or pairwise) step based on the cache corner.
A step defined this way is equivalent to the corresponding step on the convex hull of the working set, which is a subset of the block domain~$\outputdomain_i$.
If a cache hit criterion is not satisfied, i.e., there is a \emph{cache miss}, we call the (possibly expensive) max oracle to obtain a Frank-Wolfe corner over the full domain $\outputdomain_i$.
Alg.~\ref{alg:FW_product_SVM_caching} in App.~\ref{app:algorithms} summarizes the BCFW method with caching.

Note that, in the case of BCPFW, the working set~$\cache_i$ is closely related to the active set~$\activeS_i$. 
On the implementation side, we maintain both sets in the same data structure and keep $\activeS_i \subseteq \cache_i$.

\paragraph{Cache hit criterion.}
An important part of a caching scheme is the criterion deciding whether the cache look up is sufficient or the max oracle needs to be called.
Intuitively, we want to use the cache whenever it allows optimization to make large enough progress.
We use as measure of potential progress the inner product between the candidate direction and the negative gradient (which would give the block gap $g_i$~\eqref{eq:block_gap} if the FW corner is used). For a cache step, it gives $\cachegap_i^{(k)} := \regularizerweight(\weightv_i^{(k)}-\weightv_{\cachev})^\transpose \weightv^{(k)}-{\ell_i}^{(k)}+\ell_{\cachev}$, which is defined by quantities~$\weightv_{\cachev} = \frac{ \featuremapdiffv_i(\outputvarv_i^c)}{\regularizerweight n}$, $\ell_{\cachev}=\frac1n L_i(\outputvarv_i^c)$ similar to the ones defining the block gap.
The quantity~$\cachegap_i^{(k)}$ is then compared to a \emph{cache hit threshold} defined as~$\max(\cacheF g_i^{(k_i)},\frac{\cacheN}{n} g^{(k_0)})$ where~$k_i$ identifies the iteration when the max oracle was last called for the block~$i$, $k_0$ is the index of the iteration when the full batch gap was computed, $\cacheF>0$ and $\cacheN>0$ are cache parameters. 

The following theorem gives a safety convergence result for BCFW with caching (see App.~\ref{app:caching_theorem} for the proof).
\begin{theorem} \label{thm:cacheTheorem}
Consider the same notation as in Theorem~\ref{thm:convergence_FW_product}. Let~$\tilde{\cacheN} := \frac1n\cacheN \leq 1$. The iterate $\dualvarv^{(k)}$ of %
Alg.~\ref{alg:FW_product_SVM_caching} satisfies
$
\E\big[f(\dualvarv^{(k)})\big] - f(\dualvarv^*) \le \frac{2n}{\tilde{\cacheN}k+2n}\big(\frac{1}{\tilde{\cacheN}}\CfTotal+ h_0\big)
$
for $k\geq0$.
\end{theorem}
Note that the convergence rate of Theorem~\ref{thm:cacheTheorem} differs from the original rate of  BCFW (Theorem~\ref{thm:convergence_FW_product}) by the  constant~$\tilde{\cacheN}$.
If~$\tilde{\cacheN}$ equals one the rate is the same, but the criterion effectively prohibits cache hits.
If~$\tilde{\cacheN} < 1$ then the convergence is slower, meaning that the method with cache needs more iterations to converge, but the oracles calls might be cheaper because of the cache hits.

\begin{figure}
    \centering
    \begin{subfigure}[b]{0.22\textwidth}
        \includegraphics[width=0.9\textwidth]{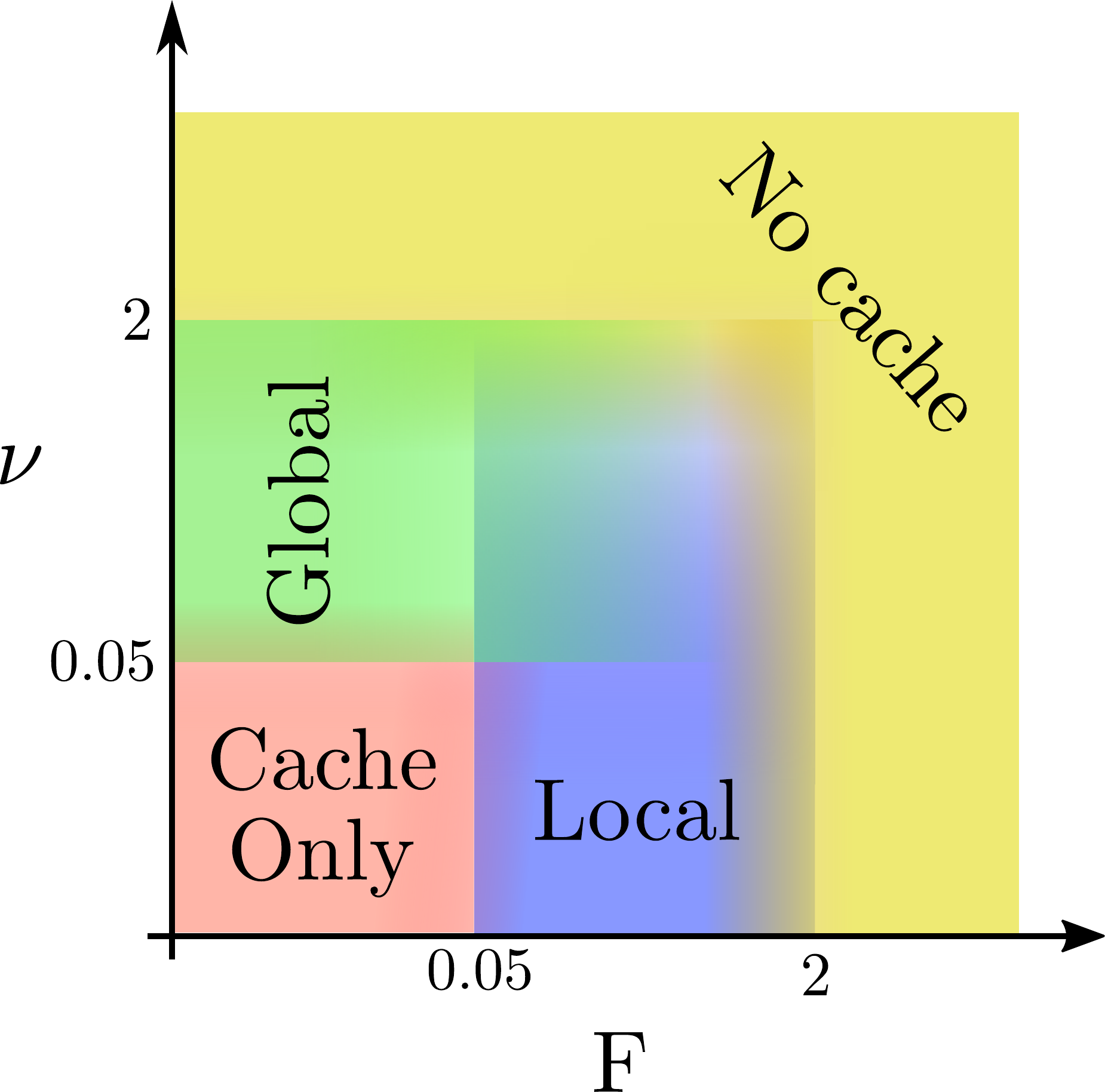}
        \caption{Regimes of the cache\label{fig:caching_a}\vspace{-2mm}}
    \end{subfigure}
    ~
    \begin{subfigure}[b]{0.22\textwidth}
        \includegraphics[width=0.9\textwidth]{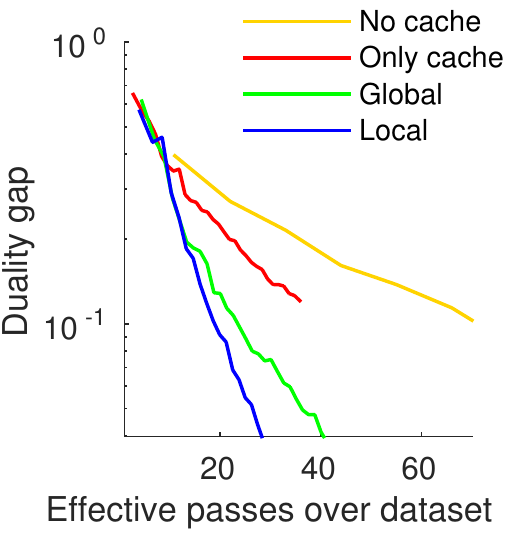}
        \caption{Convergence plots\label{fig:caching_b}\vspace{-2mm}}
    \end{subfigure}
    \caption{ Plot~(a) illustrates different regimes induced by the cache parameters~$\cacheF$ and~$\cacheN$. Plot~(b) shows the evolution of the duality gap within BCFW with gap sampling and with cache parameters in different regimes. \label{fig:caching} \vspace{-3mm}}
\end{figure}

\paragraph{Effect of $F$ and $\nu$.}
The parameter~$\nu$ controls the global component and acts as a safety parameter to ensure convergence (Theorem~\ref{thm:cacheTheorem}).
The parameter~$F$ controls, instead, the local (block-dependent) component of the criterion.
Figure~\ref{fig:caching} illustrates the effect of the parameters on OCR dataset~\citep{Taskar2003} and motivates their choice.
At one extreme, if either $F$ or $\nu$ are too large the cache is almost never hit.
At another extreme, if both values are small the cache is hit almost always, thus the method almost stops calling the oracle and does not converge.
Between the two extremes, one of the components usually dominates.
We observe empirically that the regime with the local component dominating leads to faster convergence.
Our experiments show that the method is not very sensitive to the choice of the parameters, so, in what follows, we use values $F=0.25$ and $\nu=0.01$.

\vspace{-1mm}
\paragraph{Related work.}
In the context of SSVM, the idea of caching was successfully applied to the cutting plane methods by~\citet{Joachims:2009ex}, and, recently, to BCFW by~\citet{shah2015caching}. 
In contrast to~\citet{shah2015caching}, our method chooses whether to call the oracle or to use the cache in an adaptive way by looking at the gap estimates of the current blocks.
In the extreme case, when just one block is hard and requires computation and all the rest are easy, our method would be able to call an oracle on the hard block and to use the cache everywhere else. This will result to $n$ times less oracle calls, compared to their strategy.

\section{Regularization path \label{sec:reg_path}}
According to the definition of~\citet{EfronRegPathLasso_04}, a regularization path is a set of minimizers of a regularized objective in the form of~\eqref{eq:svmstruct_nslack_primal_nonsmooth} for all possible values of the regularization parameter~$\regularizerweight$.
Similarly to LASSO and binary SVM, the general result of~\citet[Proposition~1]{Rosset07_regPathTheory} is applicable to the case of SSVM and implies that the exact regularization path is piecewise linear in~$1/\regularizerweight$.
However, recovering the exact path is, up to our knowledge, intractable in the case of SSVM.
In this paper, we construct an $\epsilon$-approximate regularization path, meaning that, for each feasible~$\regularizerweight$, we have a corresponding primal variables~$\weightv$ which is $\epsilon$-approximate, i.e., the suboptimality $f_{\regularizerweight}(\weightv) - f_{\regularizerweight}^*$ does not exceed~$\epsilon$. 
We use a piecewise \emph{constant} approximation except for the first piece which is linear. 
The approximation is represented by a set of breakpoints~$\{\regularizerweight_j\}_{j=0}^{J+1}$, $\regularizerweight_0 \!=\! +\infty$, $\regularizerweight_{J+1} \!=\! 0$, $\regularizerweight_{j+1} \!\leq\! \regularizerweight_j$, and a set of parameter vectors $\{\weightv^{j}\}_{j=1}^{J}$ with the following properties: for each $\regularizerweight \in [\regularizerweight_{j+1}, \regularizerweight_{j}]$, $j \geq 1$, the vector~$\weightv^{j}$ is $\epsilon$-approximate; for $\regularizerweight \geq \regularizerweight_{1}$, the vector $\frac{\regularizerweight_1}{\regularizerweight}\weightv^{1}$ is $\epsilon$-approximate.

Our algorithm consists of two steps:
(1) at the initialization step, we find the maximal finite breakpoint~$\regularizerweight^\infty := \regularizerweight_{1}$ and the vector $\weightv^\infty:=\weightv^{1}$;
(2) at the induction step, we compute a value~$\regularizerweight_{j+1}$ and a vector $\weightv^{j+1}$ given quantities~$\regularizerweight_{j}$ and $\weightv^{j}$. 
At both steps of our algorithm, we explicitly maintain dual variables $\dualvarv$ that correspond to $\weightv$.
Alg.~\ref{alg:rpAlgorithmSSVM} in App.~\ref{app:algorithms} presents the complete procedure.

\paragraph{Initialization of the regularization path.}
First, note that, for~$\regularizerweight = \infty$, the KKT conditions for \eqref{eq:svmstruct_nslack_primal_nonsmooth} and~\eqref{eq:svmstruct_nslack_dual} imply that $\weightv=\0$ is a solution of the problem~\eqref{eq:svmstruct_nslack_primal_nonsmooth}.
In what follows, we provide a finite value for~$\regularizerweight^\infty$ and explicitly construct~$\dualvarv^\infty$ and $\weightv^\infty$ such that $\frac{\lambda^\infty}{\lambda}\weightv^\infty$ is $\epsilon$-approximate for~$\regularizerweight \geq \regularizerweight^\infty$.

Let 
$
\tilde{\outputvarv}_i = \argmax_{\outputvarv \in \mathcal{Y}_i} H_i(\outputvarv; \0) = \argmax_{\outputvarv \in \mathcal{Y}_i} L_i(\outputvarv)
$
be the output of the max oracle for~$\weightv = \0$.
First, we construct a dual point~$\dualvarv^\infty \in \domain$ by setting $\dualvar^\infty_i(\tilde{\outputvarv}_i) = 1$. For any value of~$\regularizerweight^\infty$, the corresponding weight vector can be easily computed:
$
\weightv^\infty =  \frac{1}{\regularizerweight^\infty n} \sum_{i=1}^n \featuremapdiffv_i (\tilde{\outputvarv_i}).
$
Identity~\eqref{eq:duality_gap_withH} provides the duality gap:\\[-5mm]
\begin{align*}
g(\dualvarv^\infty, \regularizerweight^\infty, \weightv^\infty) = \frac1n \sum_{i=1}^{n} \Big( &\max_{\outputvarv \in \mathcal{Y}_i } \big( L_i(\outputvarv) - \langle\weightv^\infty, \featuremapdiffv(\outputvarv) \rangle \big) \nonumber \\
&- L_i(\tilde{\outputvarv_i}) + \langle(\weightv^\infty, \featuremapdiffv(\tilde{\outputvarv_i})\rangle \Big).
\end{align*}\\[-5mm]
The inequality $\max_x( f(x) + g(x) ) \leq \max_x f(x) + \max_x g(x)$ and the equality $ \max_{\outputvarv \in \mathcal{Y}_i } L_i(\outputvarv) =L_i(\tilde{\outputvarv_i})$ bound the gap:\\[-5mm]
\begin{multline}
g(\dualvarv^\infty, \regularizerweight^\infty, \weightv^\infty) \leq  \frac1n \sum_i \Big( \max_{\outputvarv \in \mathcal{Y}_i } (- \langle\weightv^\infty, \featuremapdiffv(\outputvarv)\rangle ) +\\  \langle\weightv^\infty, \featuremapdiffv(\tilde{\outputvarv_i})\rangle \Big) \nonumber 
= \frac{1}{n \regularizerweight^\infty} \sum_{i=1}^n \theta_i + \frac{1}{\regularizerweight^\infty} \norm{\tilde{\featuremapdiffv}}^2
\end{multline}\\[-4mm]
where the quantities $\theta_i = \max_{\outputvarv \in \outputdomain_i}  \big( - \langle\tilde{\featuremapdiffv}, \featuremapdiffv(\outputvarv)\rangle \big)$ and $\tilde{\featuremapdiffv} := \frac1n \sum_i \featuremapdiffv_i (\tilde{\outputvarv_i})$ are easily computable. 
To ensure that $g(\dualvarv^\infty, \regularizerweight, \frac{\regularizerweight^\infty}{\regularizerweight}\weightv^\infty) \leq\epsilon$ for $\regularizerweight \geq \regularizerweight^\infty$, we can now set\\[-3mm]
\begin{equation*}
\regularizerweight^\infty := \frac{1}{\epsilon} \bigg(\|\tilde{\featuremapdiffv}\|^2 + \frac1n \sum_{i=1}^n \theta_i \bigg).
\end{equation*}
 
\vspace{-3mm}\paragraph{Induction step.}

We utilize the intuition that the expression~\eqref{eq:duality_gap_withH} provides control on the Frank-Wolfe gap for different values of~$\regularizerweight$ if the primal variables~$\weightv$ and, consequently, the results of the max oracles stay unchanged. 
Proposition~\ref{thm:reg_path} formalizes this intuition.
\begin{proposition}
    \label{thm:reg_path}
    Assume that $L_i(\outputvarv_i) = 0$, $i=1,\dots,n$, i.e., the loss on the ground truth equals zero. Let $\regPathStep\!:=\!\frac{\regularizerweight^{\text{new}}}{\regularizerweight^{\text{old}}}\!<\!1$.
    Then, setting $\dualvar_i(\outputvarv) := \regPathStep \dualvar_i^{\text{old}}(\outputvarv)$, $\outputvarv \neq \outputvarv_i$,
    and $\dualvar_i(\outputvarv_i):=1-\sum_{\outputvarv \neq \outputvarv_i} \dualvar_i(\outputvarv)$, we then have $\weightv^{\text{new}} = \weightv^{\text{old}}$ and
    \begin{equation}
    g(\dualvarv,\: \regularizerweight^{\text{new}}) = g(\dualvarv^{\text{old}}\!\!\!,\: \regularizerweight^{\text{old}}) + (1 - \regPathStep) \Delta(\dualvarv^{\text{old}}\!\!\!,\: \regularizerweight^{\text{old}})
    \label{eq:ssvm_rp_gapnew}
    \end{equation}
    where\\[-4.5mm]
    \begin{equation}
    \notag
    \Delta(\dualvarv^{\text{old}}\!\!\!,\: \regularizerweight^{\text{old}}) := \frac1n \sum_{i=1}^n\sum_{\outputvarv \in \outputdomain_i} \dualvar_i^{\text{old}}(\outputvarv) H_i(\outputvarv; \weightv^{\text{old}}).
    \end{equation}
\end{proposition}
\begin{proof}
    \vspace{-2mm}Consider the problem~\eqref{eq:svmstruct_nslack_dual} for both~$\regularizerweight^{\text{new}}$ and~$\regularizerweight^{\text{old}}$\!\!\!.\:
    Since~$\featuremapdiffv_i(\outputvarv_i)=\bm{0}$ and $A^{\text{new}}=\frac{1}{\regPathStep}A^{\text{old}}$, we have that $\weightv^{\text{old}} = A^{\text{old}} \dualvarv^{\text{old}} = A^{\text{new}} \dualvarv = \weightv^{\text{new}}$.
    The assumption~$L_i(\outputvarv_i) = 0$ implies equalities~$H_i(\outputvarv_i; \weightv^{\text{old}})=0$.
    Under these conditions, the equation~\eqref{eq:ssvm_rp_gapnew} directly follows from the computation of $g(\dualvarv,\: \regularizerweight^{\text{new}})-g(\dualvarv^{\text{old}}\!\!\!,\: \regularizerweight^{\text{old}})$ and the equality~\eqref{eq:duality_gap_withH}.
\end{proof}

Assume that for the regularization parameter~$\regularizerweight^{\text{old}}$ the primal-dual pair $\dualvarv^{\text{old}}$\!, $\weightv^{\text{old}}$ is $\factorRegPath\epsilon$-approximate, $0<\factorRegPath<1$, i.e., $g(\dualvarv^{\text{old}}\!\!,\: \regularizerweight^{\text{old}}) \leq \factorRegPath\epsilon$.
Proposition~\ref{thm:reg_path} ensures that $g(\dualvarv,\: \regularizerweight^{\text{new}}) \leq \epsilon$ whenever\\[-0.15cm]
\begin{equation}\label{eq:reg_path_weight}
\regPathStep 
= 
1 - \frac{\epsilon - g(\dualvarv^{\text{old}}\!,\: \regularizerweight^{\text{old}})}{\Delta(\dualvarv^{\text{old}}\!,\: \regularizerweight^{\text{old}})}
\leq 1 - \frac{\epsilon(1-\factorRegPath)}{\Delta(\dualvarv^{\text{old}}\!,\: \regularizerweight^{\text{old}})}.
\end{equation}
Having~$\factorRegPath<1$ ensures that $\regPathStep < 1$, i.e., we get a new break point $\regularizerweight^{\text{new}} < \regularizerweight^{\text{old}}$.
If the equation~\eqref{eq:reg_path_weight} results in $\regPathStep\leq 0$ then we reach the end of the regularization path, i.e., $\weightv^{\text{old}}$ is $\epsilon$-approximate for all $0 \leq \regularizerweight < \regularizerweight^{\text{old}}$.

To be able to iterate the induction step, we apply one of the algorithms for the minimization of the SSVM objective for~$\regularizerweight^{\text{new}}$ to obtain~$\factorRegPath\epsilon$-approximate pair $\dualvarv^{\text{new}}$, $\weightv^{\text{new}}$\!. Initializing from~$\dualvarv$, $\weightv^{\text{old}}$ provides fast convergence in practice.

\paragraph{Related work.} Due to space constraints, see App.~\ref{app:relatedWorkRegPath}.

\begin{figure*}[t]
    \centering
    \begin{subfigure}[b]{\textwidth}
        \centering
        \includegraphics[width=0.85\textwidth]{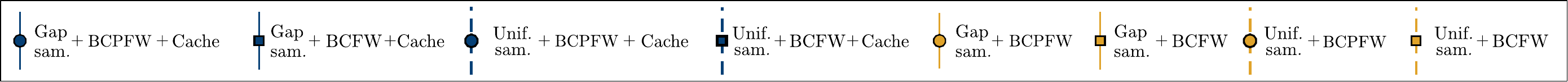}
    \end{subfigure}
    \begin{subfigure}[b]{0.29\textwidth}
        \includegraphics[width=\textwidth]{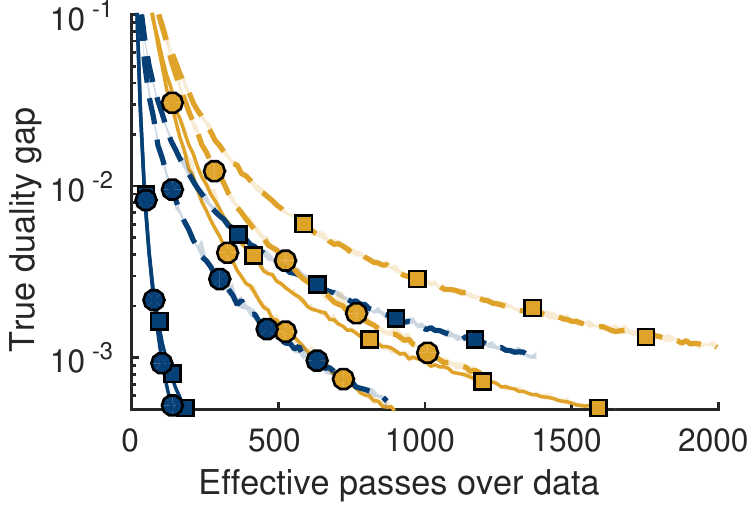}
        \includegraphics[width=\textwidth]{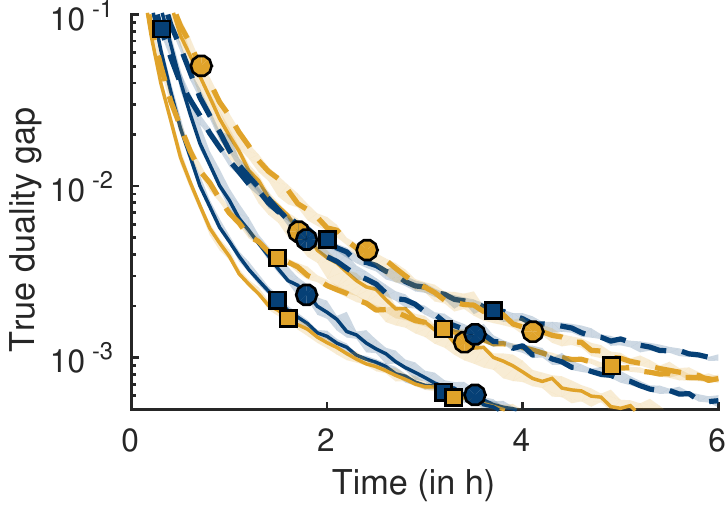}
        \caption{OCR-large, $\regularizerweight=0.001$\label{fig:exp1_main_a}\vspace{-2mm}}
    \end{subfigure}
    \qquad
    \begin{subfigure}[b]{0.29\textwidth}
        \includegraphics[width=\textwidth]{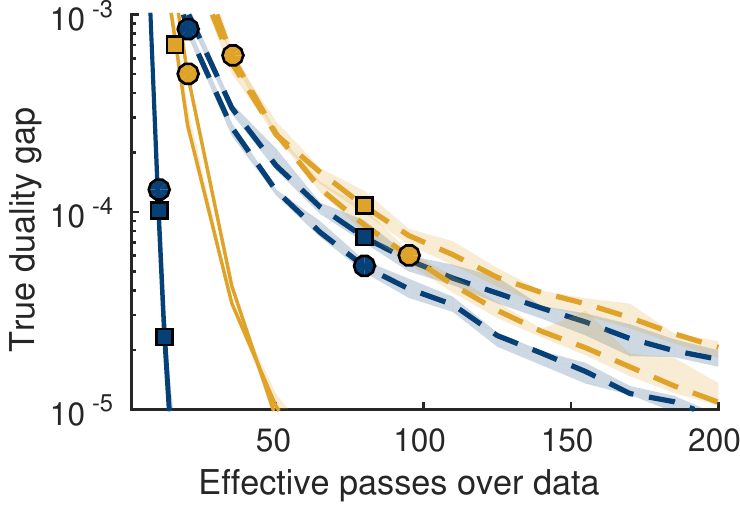}
        \includegraphics[width=\textwidth]{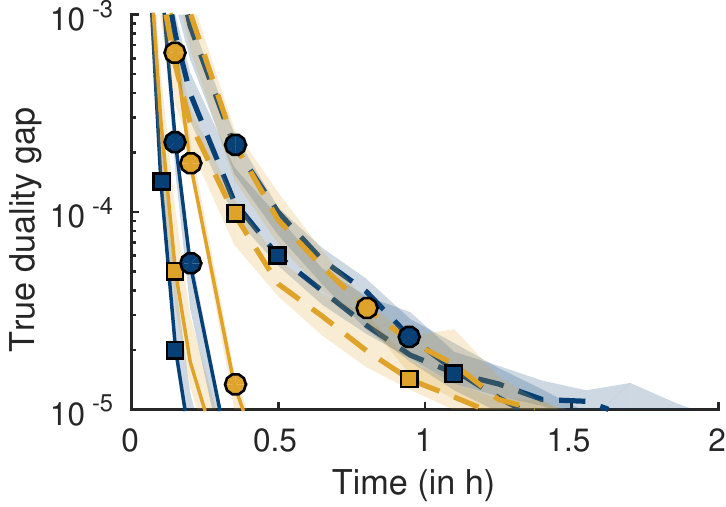}
        \caption{HorseSeg-medium, $\regularizerweight=10$\label{fig:exp1_main_b}\vspace{-2mm}}
    \end{subfigure}
    \qquad
    \begin{subfigure}[b]{0.29\textwidth}
        \includegraphics[width=\textwidth]{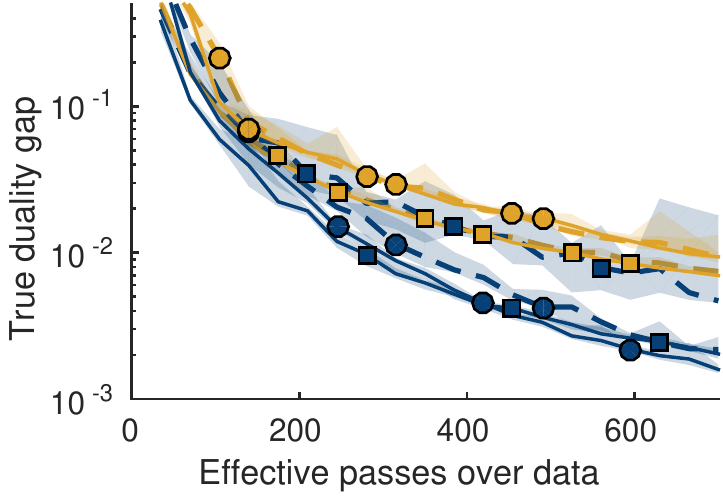}
        \includegraphics[width=\textwidth]{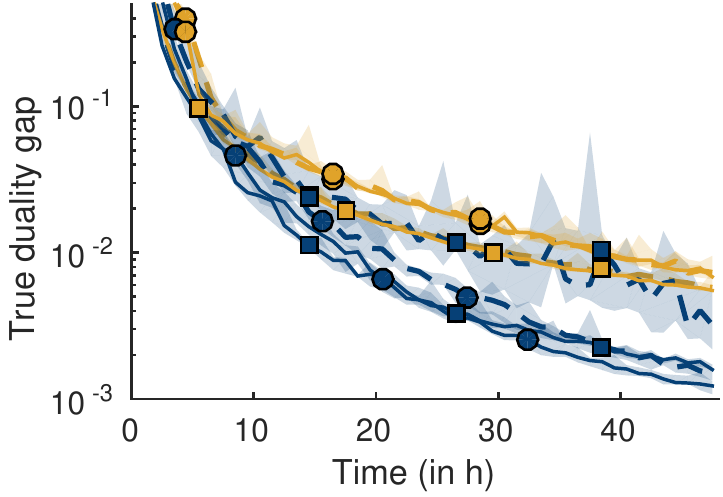}
        \caption{LSP-small, $\regularizerweight=100$\label{fig:exp1_main_c}\vspace{-2mm}}
    \end{subfigure}
    \caption{Summary of the results of Section~\ref{subsec:exp1}: the duality gap against the number of effective passes over data (top) and time (bottom).
        \label{fig:exp1_main} \vspace{-4.5mm}  
    }
\end{figure*}

\section{Experiments}\label{sec:experiments}
The experimental evaluation consists of two parts:
Section~\ref{subsec:exp1} compares the different algorithms presented in Section~\ref{sec:contributions};
Section~\ref{subsec:reg_path} evaluates our approach on the regularization path estimation. %

\textbf{Datasets.}
We evaluate our methods on four datasets for different structured prediction tasks: OCR~\citep{Taskar2003} for handwritten character recognition, CoNLL~\citep{Sang2000} for text chunking, HorseSeg~\citep{kolesnikov2014closed} for binary image segmentation and LSP~\citep{Johnson10} for pose estimation.
The models for OCR and CoNLL were provided by~\citet{lacosteJulien13bcfw}. We build our model based on the one by~\citet{kolesnikov2014closed} for HorseSeg, and the one by~\citet{Chen_NIPS14} for LSP.
For OCR and CoNLL, the max oracle consists of the Viterbi algorithm~\citep{Viterbi67}; for HorseSeg~-- in graph cut~\citep{boykov2004graphcut}, for LSP~-- in belief propagation on a tree with messages passed by a generalized distance transform~\citep{felzenszwalb2005distTrans}.
Note that the oracles of HorseSeg and LSP require positivity constraints on a subset of the weights in order to be tractable. 
The BCFW algorithm with positivity constraints is derived in App.~\ref{app:pos_const}.
We provide a detailed description of the datasets in App.~\ref{app:dataset_table} with a summary in Table~\ref{tab:dataset}.

The problems included in our experimental study vary in the number of objects~$n$ (from $100$ to $25,\!000$), in the number of features~$d$ (from $10^2$ to $10^6$), and in the computational cost of the max oracle (from $10^{-4}$ to $2$ seconds).

\subsection{Comparing the variants of  BCFW \label{subsec:exp1}}

In this section, we evaluate the three modifications of BCFW presented in Section~\ref{sec:contributions}.
We compare 8 methods obtained by all the combinations of three binary dimensions: 
gap-based vs.\ uniform sampling of objects, BCFW vs.\ BCPFW, caching oracle calls vs.\ no caching.

We report the results of each method on 6 datasets (including 3 sizes of HorseSeg) for three values of the regularization parameter~$\regularizerweight$: the value leading to the best test performance, a smaller and a larger value. For each setup, we report the duality gap against both number of oracle calls and elapsed time.
We run each method~$5$ times with different random seeds influencing the order of sampled objects and report the median (bold line), minimum and maximum values (shaded region).
We summarize the results in Figure~\ref{fig:exp1_main} and report the rest in App.~\ref{app:exp1}.

First, we observe that, aligned with our theoretical results, gap sampling always leads to faster convergence (both in terms of time and the number of effective passes).
The effect is stronger when $n$ is large (Figure~\ref{fig:exp1_main_b}).
Second, caching always helps in terms of number of effective passes, but an overhead caused by maintaining the cache is significant when the max oracle is fast (Figure~\ref{fig:exp1_main_a}).
In the case of expensive oracle (Figure~\ref{fig:exp1_main_c}), the cache overhead is negligible.
Third, the pairwise steps (BCPFW) lead to an improvement to get smaller values of duality gaps.
The effect is stronger when the problem is more strongly convex, i.e., $\regularizerweight$ is bigger. However, maintaining the active sets results in computational overheads, which sometimes are significant.
Note that the overhead of cache and active sets are shared, because they are maintained in the same data structure.
Using a cache also greatly limits the memory requirements of BCPFW, because, when the cache is hit, the active set is guaranteed not to grow.

\textbf{Recommendation.} For off-the-shelf usage, we recommend to use the BCPFW + gap sampling + cache method when oracle calls are expensive, and the BCFW + gap sampling method when oracle calls are cheap.

\vspace{-2mm}
\subsection{Regularization path}\label{subsec:reg_path}
In this section, we evaluate our regularization path algorithm presented in Section~\ref{sec:reg_path}.
\begin{figure}
\vspace{-1mm}
     \begin{tabular}{c@{$\:\:$}c}
         \includegraphics[width=0.23\textwidth]{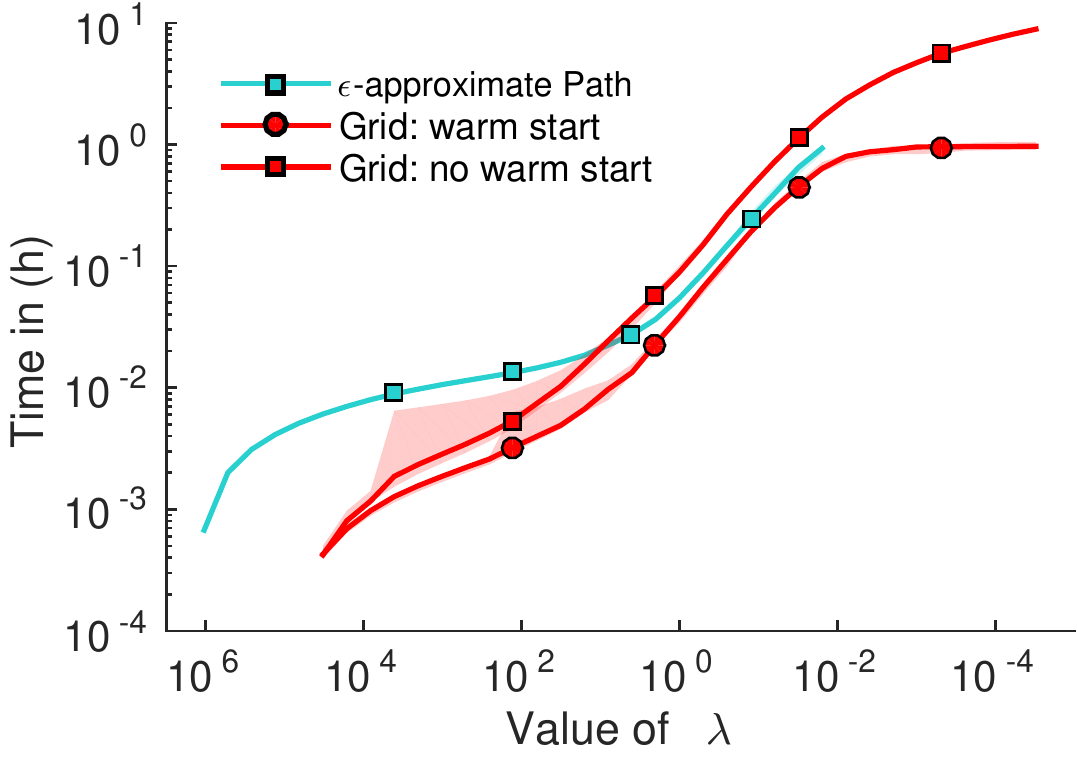} &
         \includegraphics[width=0.23\textwidth]{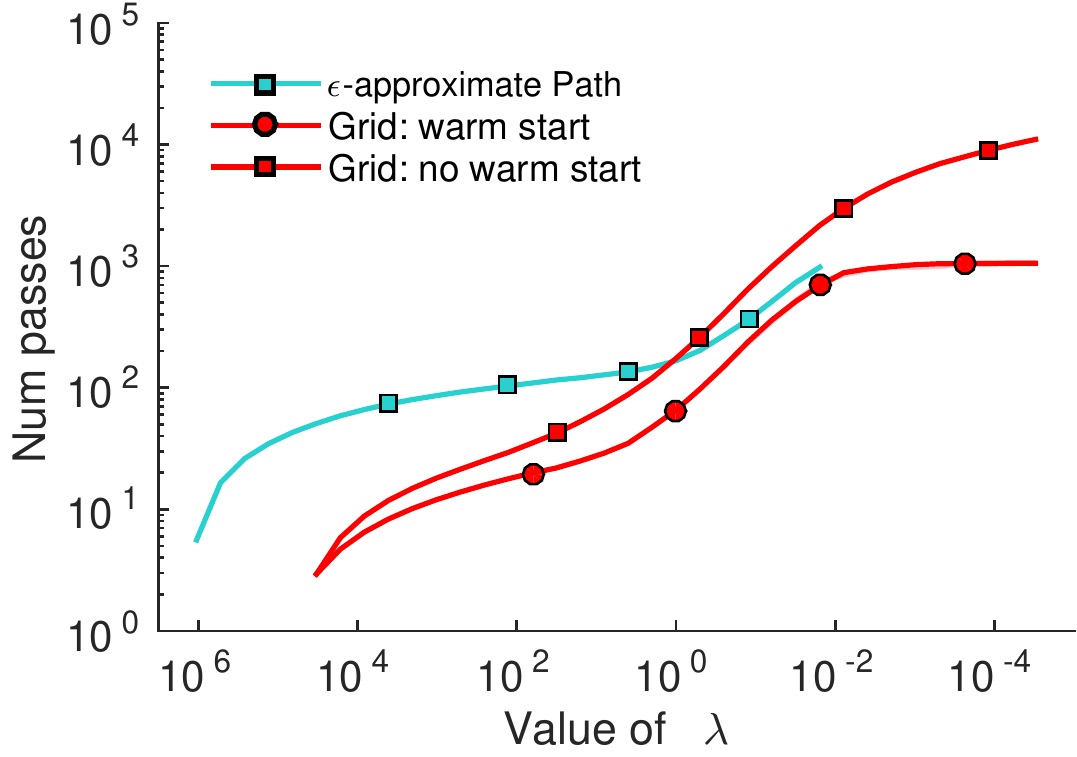}\\[-0.3cm]
        \end{tabular}
        \caption{
            \label{fig:exp2_horseSmall} On HorseSeg-small, we compare $\epsilon$-approximate regularization path against grid search with/without warm start. 
            We report the cumulative running time (left) and the cumulative effective number of passes (right) required to get to each value of the regularization parameter~$\regularizerweight$.
            \vspace{-2mm} 
        }
\end{figure}
We compare an $\epsilon$-approximate regularization path with $\epsilon=0.1$ against the standard grid search approach with/without warm start (we use a grid of 31 values of $\regularizerweight$: $2^{15},2^{14},\dots,2^{-15}$).
In Figure~\ref{fig:exp2_horseSmall}, we report the cumulative elapsed time and cumulative number of effective passes over the data required by the three methods to reach a certain value of~$\regularizerweight$ on the HorseSeg-small dataset (starting from the initialization value for the path method and the maximum values of the grid for the grid search methods).
The methods and additional experiments are detailed in App.~\ref{app:exp2}.

\paragraph{Interpretation.} First, we observe that warm start speeds up the grid search.
Second, the cost of computing the full regularization path is comparable with the cost of grid search.
However, the regularization path algorithm finds solutions for all values of $\regularizerweight$ without the need to predefine the grid.

\section*{Acknowledgments} 
This work was partially supported by the MSR-Inria Joint Center and a Google Research Award.
\bibliography{icml2016_bcfw}
\bibliographystyle{icml2016}

\clearpage
\appendix

\twocolumn[
\icmltitle{Supplementary Material\\
Minding the Gaps for Block Frank-Wolfe optimization of structured SVMs}
]

\paragraph{Outline} 
\begin{description}[itemsep=0mm,topsep=0cm,leftmargin=0cm,font=\normalfont]
\renewcommand\labelitemi{--}
\item[Appendix~\ref{app:relatedWorkRegPath}:] related work on regularization paths.
\item[Appendix~\ref{app:descent_lemma}:] review of the block descent lemma for BCFW.
\item[Appendix~\ref{app:toy_example}:] toy example showing that our adaptive gap sampling technique can be~$n$ times faster than alternatives.
\item[Appendix~\ref{app:algorithms}:] details of the proposed algorithms.
\item[Appendix~\ref{app:proof_gap_sampling}:] the proof of Theorem~\ref{thm:convTheoremGapSampling} on convergence of BCFW with gap sampling.
\item[Appendix~\ref{app:caching_theorem}:] the proof of Theorem~\ref{thm:cacheTheorem} on convergence of BCFW with caching.
\item[Appendix~\ref{app:BCPFWconvergence}:] the proof of Theorem~\ref{thm:BCPFW} on convergence of BCPFW and BCAFW (with conditions).
\item[Appendix~\ref{app:pos_const}:] generalization of BCFW to the case of SSVM with box constraints on the weights.
\item[Appendix~\ref{app:dataset_table}:] the detailed description of the datasets used in our experiments.
\item[Appendix~\ref{app:exp1}:] the full experimental comparison of different variants of BCFW.
\item[Appendix~\ref{app:exp2}:] the full experimental study of our regularization path method.
\end{description}

\section{Related work for regularization path} \label{app:relatedWorkRegPath}

\citet{EfronRegPathLasso_04}, in their seminal paper, introduced the notion of regularization path and showed that the regularization path of  LASSO~\citesup{TibshiraniLasso_96} is piecewise linear.
\citetsup{HastieRegPathSVM_04} proposed the path following method to compute the exact regularization path for the binary SVM with L2-regularization.
Exact path following algorithms suffer from numerical instabilities as they repeatedly invert a potentially badly-conditioned matrix~\citepsup{Allgower931_pathFollowing}.
In addition, \citetsup{Gartner12_regPathComplexity} show that, in the worst case, the regularization path for binary SVM contains an exponential number of break points.
Although the exact path following methods for SVM are still developed~\citepsup{Sentelle15_regPathSVM}, approximate methods might be more suited for practical use cases.
\citetsup{Karasuyama11_regPathSvm} proposed a method providing a trade-off between accuracy of the path and its computational cost.
Recently, \citetsup{Giesen12_regPathTheory} have developed a general framework to construct a piecewise-constant $\epsilon$-approximate path with at most~$\bigO(1/\sqrt{\epsilon})$ break points and applied it, e.g., to binary SVM. 

In contrast to binary SVM, regularization paths for multi-class SVMs are more complex and less studied.
\citetsup{Lee06_regPathMultiSVM} proposed a path following method for multi-class SVM in the MSVM formulation of~\citetsup{Lee04_MSVM}. \citetsup{Wang06_regPathMultiSVM} analyzed the regularization path for the L1 version of MSVM.
Finally, \citetsup{Jun-Tao10_regPathMultiSVM} constructed the regularization path for the multi-class SVM with huberized loss.
We are not aware of any work computing the regularization path for SSVM or, for its predecessor multi-class SVM in the formulation of \citetsup{Crammer01_multiClassSVM}.

The induction step of our method is similar to Alg.~1 of~\citepsup{Giesen12_regPathTheory} applied to the case of binary SVM.
They also construct a piecewise linear $\epsilon$-approximate path by alternating the SVM solver and a procedure to identify the region where the output of the solver is accurate enough. 

In contrast to our method, \citetsup{Giesen12_regPathTheory} construct the path only for the predefined segment of the values of~$\regularizerweight$. We do not require such a segment as input and are able to find the largest and smallest value automatically.
Another difference to~\citepsup{Giesen12_regPathTheory} consists in using the $\regularizerweight$-formulation of SVM instead of the $C$-formulation. In the two formulations, the accuracy parameter~$\epsilon$ is scaled differently for the different values of the regularization parameters. The $\regularizerweight$-formulation requests higher accuracy for the small values of~$\regularizerweight$ and, thus, creates more break points in that region.

\section{Block descent lemma for BCFW \label{app:descent_lemma}}

\begin{definition}[Block curvature constant]
    \label{def:curvature_const}
    Consider a convex function $f$ defined on a separable domain~$\domain=\domain^{(1)} \times\dots\times \domain^{(n)}$.
    The curvature constant~$\Cf^{(i)}$ of the function $f$ w.r.t.\ the individual block of coordinates $\domain^{(i)}$ is defined by
    \begin{align}
    \notag \Cf^{(i)} := \sup\; &\frac{2}{\stepsize^2}\left( f(\dualvarvtwo)-f(\dualvarv)-\langle \dualvarvtwo_{(i)}-\dualvarv_{(i)}, \nabla_{(i)} f(\dualvarv)\rangle \right) \\
    \notag \mbox{s.t.}\; & \dualvarv\in\domain,\,\sv_{(i)}\in \domain^{(i)}, \stepsize\in[0,1], \\
    \label{eqn:CfProduct} & \dualvarvtwo = \dualvarv+\stepsize(\sv_{[i]}-\dualvarv_{[i]}).
    \end{align}
\end{definition}
Here~$\sv_{[i]} \in \R^m$ and $\dualvarv_{[i]} \in \R^m$ are the zero-padded versions of~$\sv_{(i)}\in \domain^{(i)}$ and~$\dualvarv_{(i)}\in \domain^{(i)}$, respectively.
Note that, although $\sv_{[i]} \not\in \domain$ and $\dualvarv_{[i]} \not\in \domain$, we have that $\dualvarvtwo := \dualvarv+\stepsize(\sv_{[i]}-\dualvarv_{[i]}) \in \domain$.

In the case of SSVM, the curvature constant~$\Cf^{(i)}$ can be upper bounded (tightly in the worst case) with $\frac{4R_i^2}{\regularizerweight n^2}$ where
$R_i:= \max_{\outputvarv \in \outputdomain_i}\norm{\featuremapdiffv_i(\outputvarv)}_2$~\citep[Appendix~A]{lacosteJulien13bcfw}.\footnote{More generally, let $\|\cdot \|_i$ be some norm defined on $\domain^{(i)}$. Then suppose that $L_i$ is the Lipschitz-continuity constant with respect to this norm for $\nabla_{(i)} f(\dualvarv)$ when only $\dualvarv_{(i)}$ varies, i.e., $\| \nabla_{(i)} f(\dualvarv) - \nabla_{(i)} f(\dualvarv + \sv_{[i]} - \dualvarv_{[i]}) \|_i^* \leq L_i \| \sv_{(i)} -  \dualvarv_{(i)} \|_i$ for all $\dualvarv \in \domain$, $\sv_{(i)} \in \domain^{(i)}$, where $\|\cdot\|_i^*$ is the dual norm of $\|\cdot\|_i$. Then similarly to Lemma~7 in~\citetsup{Jaggi:2013wg}, we have $\Cf^{(i)} \leq L_i \big(\diam_{\|\cdot\|_i} \domain^{(i)} \big)^2$.
\label{foot:Cfi}}

For reference, we restate below the key descent lemma used for the proof of convergence of BCFW and its variants. We note in passing that this is an affine invariant analog of the standard descent lemmas that use the Lipschitz continuity of the gradient function to show progress during first order optimization algorithms.

\begin{lemma}[Block descent lemma]\label{lem:block_step_improvement}
    For any~$\dualvarv \in \domain$ and block~$i$, let $\sv_{(i)}\in \domain^{(i)}$ be the Frank-Wolfe corner selected by the max oracle of block~$i$ at~$\dualvarv$. Let~$\dualvarv_{LS}$ be obtained by the line search between $\dualvarv_{(i)} \in \domain^{(i)}$ and $\sv_{(i)}$, i.e.,  $f(\dualvarv_{LS}) = \min_{\gamma\in[0,1]} f(\dualvarv_\stepsize)$ where $\dualvarv_\stepsize := \dualvarv+\stepsize(\sv_{[i]}-\dualvarv_{[i]})$. Then, it holds that for each $\stepsize\in[0,1]$: \\[-2mm]
    \begin{equation}
    \label{eq:block_descent_line_search_lemma}
    f(\dualvarv_{LS})
    \le f(\dualvarv) - \stepsize g_i(\dualvarv) + \frac{\stepsize^2}{2} \Cf^{(i)}\\[-1mm]
    \end{equation}
    where~$\Cf^{(i)}$ is the curvature constant of the function $f$ over the factor $\domain^{(i)}$ and $g_i(\dualvarv)$ is the block gap at the point~$\dualvarv$ w.r.t. the block~$i$.
\end{lemma}
\begin{proof}
    From Definition~\ref{def:curvature_const} of the curvature constant and the expression~\eqref{eq:block_gap} for the block gap, we have
    \begin{align*}
    f(\dualvarv_\stepsize) = & f(\dualvarv+\stepsize(\sv_{[i]}-\dualvarv_{[i]})) \\
    \le & f(\dualvarv) + \stepsize \langle \sv_{(i)}-\dualvarv_{(i)}, \nabla_{(i)} f(\dualvarv)\rangle + \frac{\stepsize^2}{2}\Cf^{(i)} \\
    = & f(\dualvarv) - \stepsize g_i(\dualvarv) + \frac{\stepsize^2}{2}\Cf^{(i)} \ .
    \end{align*}
    The inequality~$f(\dualvarv_{LS}) \leq f(\dualvarv_\stepsize)$ completes the proof.
\end{proof}

\section{Toy example for gap sampling \label{app:toy_example}}
In this section, we construct a toy example of the structured SVM problem where the adaptive gap-based sampling is $n$ times faster than non-adaptive sampling schemes such as uniform sampling or curvature-based sampling (the latter being the affine invariant analog of Lipschitz-based sampling).

\paragraph{General idea.} The main idea is to consider a training set where there are $n-1$ ``easy'' objects that need to be visited only once to learn to classify them, and one ``hard'' object that requires at least $K \gg 1$ visits in order to get the optimal parameter. We can design the example in such a way that the curvature or Lipschitz constants are non-informative about which example is hard, and which is easy. The non-adaptive sampling schemes will thus have to visit the easy objects as often as the hard object, whereas the gap sampling technique can adapt to focus only on the single hard object after having visited the easy objects once, thus yielding an overall $\min\{n,K\}$-times speedup.

Note that large-scale learning datasets could have analogous features as this toy example: a subgroup of objects might be easier to learn than another, and moreover, they might share similar information, so that after visiting a subset, we do not need to linger on the other ones from the same subset as all the information has already been extracted. We cannot know in advance which subgroups are these subsets, and thus an adaptive scheme is needed. 

\subsection{Explicit construction}

For simplicity, we set the weight of the regularizer $\regularizerweight$ to $\nicefrac{1}{n}$ so that the scaling factor defining~$A$ in Problem~\eqref{eq:svmstruct_nslack_dual} is $\nicefrac{1}{\regularizerweight n} = 1$. The matrix~$A$ thus consists of the difference feature maps, i.e., $A := \SetOf{\featuremapdiffv_i(\outputvarv):= \featuremapv(\inputvarv_i,\outputvarv_i) - \featuremapv(\inputvarv_i,\outputvarv) \in\R^d}{i\in[n],\outputvarv \in \outputdomain_i}$.
In our example, we use feature maps of dimensionality $d:=K+1 := |\outputdomain_i|$.
Let $\outputdomain_i:=\lbrace 0, 1, \dots, K \rbrace$ be the set of labels for the object~$i$ and the label~$0$ be the correct label.
We consider the zero-one loss, i.e., $\errorterm_i(0)=0$ and  $\errorterm_i(k)=1$ for $k\geq 1$. In the following, let $\{\basis_j\}_{j=1}^d$ be the standard basis vectors for $\R^d$.

\paragraph{Hard and easy objects.}
We construct the feature map~$\featuremapv(\inputvarv,\outputvarv)$ so that only the last coordinate of the parameter vector is needed to classify correctly the easy object, whereas all the other coordinates are needed to classify correctly the hard object. By using a different set of coordinates between the easy and the hard objects, we simplify the analysis as the optimization for both block types decouples (become independent).
Specifically, we set the feature map for the correct label to be $\featuremapv(\inputvarv_i,0):=\bm{0}$ for all objects~$i$. We let $i=1$ be the hard object and we set $\featuremapv(\inputvarv_i,k):={-\frac1{\sqrt{2}} \basis_{k}}$ for $k = 1,\dots,K$. For the easy object, $i\in \{2,\dots, n\}$, we use the constant $\featuremapv(\inputvarv_i,k) := {-\basis_{K+1}}$ for all $k\geq 1$.
The normalization of the feature maps is made so that the curvature constants for all the objects are equal (see below). Note also here that $\featuremapdiffv_1(k) \perp \featuremapdiffv_i(l)$ for any labels~$k, l$, and thus the optimization over block~$1$ decouples with the one for the other blocks $i = 2, \ldots, n$. The SSVM dual~\eqref{eq:svmstruct_nslack_dual} takes here the following simple form:
\begin{align}
    \label{eq:SSVMdualSimple}
    \min_{\substack{ \dualvarv\in\R^{m} \\  \dualvarv \succcurlyeq 0}} \quad  & \;\;
    \frac{1}{n} \sum_{k=1}^K \big(\frac{1}{4}{\dualvar_1(k)}^2 - \dualvar_1(k) \big) 
 	+ \frac{1}{n} (\frac{1}{2} u^2 - u ) \\
    \text{s.t.} ~~ & \;\; \sum_{k=0}^K  \dualvar_i(k) = 1 ~~~\forall i\in[n] \, , ~~  u = \sum_{i=2}^n \sum_{k=1}^K \dualvar_i(k)\, , \notag
\end{align}
where we have introduced the auxiliary variable~$u$ to highlight the simple structure for the optimization over the easy blocks. The unique\footnote{Uniqueness can be proved by noticing that the objective is strongly convex in $\dualvarv_1$ after removing $\dualvar_1(0)$ and replacing the equality constraint with an inequality.} solution for the first (hard) block is easily seen to be $\dualvar_1^*(k) = \frac{1}{K}$ for $k \geq 1$ and $\dualvar^*_1(0) = 0$. For the easy blocks, any feasible combination of dual variables that gives $u^* = 1$ is a solution. This gives the optimal parameter $\weightv^* = A \dualvarv^* = \basis_{K+1}+\frac{1}{K} \sum_{k=1}^K \basis_k$.

\paragraph{Optimization on the hard object.}
The objective for the hard object (block~1) in~\eqref{eq:SSVMdualSimple} is similar to the one used to show a lower bound of $\Omega(1/t)$ suboptimality error after~$t$ iterations for the Frank-Wolfe algorithm for $t$ smaller than the dimensionality (e.g., see Lemma~3 and~4 in~\citetsup{Jaggi:2013wg}), hence showing that the optimization is difficult on this block.
The BCFW algorithm is initialized with $\weightv = \bm{0}$, which corresponds to putting all the mass on the correct label, i.e., $\dualvar_i(0)=1$ and $\dualvar_i(k)=0$, $k\geq 1$. 
At each iteration of BCFW, the mass can be moved only towards one corner, and all the corners (of the simplex) have exactly one non-zero coordinate.
This means that after~$t$ iterations of BCFW on the first block, at most~$t$ non-ground truth dual variables can be non-zero. Minimizing the objective~\eqref{eq:SSVMdualSimple} over the first block with the constraint that at most~$t$ of these variables are non-zero give the similar solution~$\dualvar_1(k)=1/t$ for $k=1, \ldots, t$, which gives a suboptimality of~$\frac{1}{4n}(\frac{1}{t} - \frac{1}{K})$ for $t \leq K$. 
Similarly, this also yields the smallest FW gap\footnote{Recall that the FW gap here is the same as the Lagrangian duality gap (see Section~\ref{sec:duality_gap}), and so if one cares about the SSVM primal suboptimality, one needs a small FW gap.} 
possible for this block after~$t$ iterations, which is~$\frac{1}{n}\frac{1}{2t}$. This means that in order to get a suboptimality error smaller than~$\epsilon$, one needs at least
\begin{equation} \label{eq:HardBlockComplexity}
t \geq \Omega(\min\{K, \frac{1}{n\epsilon}\} ) 
\end{equation} 
BCFW iterations on the first 
block.\footnote{In fact, BCFW also has a $\mathcal{O}(\frac{1}{nt})$ gap after~$t$ iterations on the first block by the standard FW convergence theorem, as $\Cf^{(1)} = \frac{1}{n}$ as we show in~\eqref{eq:Cf1HardBlock}.}

\paragraph{Optimization on easy objects.}
Finally, we now show that after one iteration on \emph{any} easy object, the gaps $g_i$ on all easy objects become zero (i.e., they are all optimal and then stay optimal as the optimization is decoupled with the first block). After this iteration, BCFW with gap sampling visits all the easy objects exactly once and sets their gap estimates to zero, thus never revisiting them again.

Note that before visiting any easy object~$i$, we have $\langle \weightv, \featuremapdiffv_i (k) \rangle = 0$ for all~$k$ as the features for the hard object are orthogonal and $\weightv$ is initialized to zero. Thus, at the first visit of an easy object $i \in \{2,\dots,n\}$, we have $H_i(0;\weightv) = 0$ and $H_i(k;\weightv) = 1$, $k\geq 1$, and the max oracle returns some (any) label~$k \in \{1,\dots,K\}$.
Following the steps of Algorithm~\ref{alg:FW_product_SVM}, we have $\weightv_{\sv} := \frac{1}{\regularizerweight n} \featuremapdiffv (k) = \basis_{K+1}$ and $\ell_s=\frac{1}{n}\errorterm_i(k)=\frac{1}{n}$.
Then $g_i = \frac{1}{n}$ and $\stepsize = 1$ as $\weightv_i = \0$. 
The assignment ${\weightv_i} =  \basis_{K+1}$ implies the update $\weightv \leftarrow \weightv + \basis_{K+1}$ of the parameter vector.
After such an update, at all iterations, for all easy objects $i\in \{2,\dots,n\}$ and for all labels $k \in \{1,\dots,K\}$, we have 
\begin{equation}
\label{eq:easyblockHiszero}
H_i(k;\weightv) = \errorterm_i(k)-\langle \weightv, \basis_{K+1} \rangle=0 
\end{equation}
because the coordinate $\weight_{K+1}$ is never updated again. According to~\eqref{eq:duality_gap_withH}, the equalities~\eqref{eq:easyblockHiszero} imply that the block gaps~$g_i$ equal zero for all the easy objects.

\paragraph{Curvature constants.}
The simple structure of the matrix~$A$ allows us to explicitly compute the curvature constants~$\Cf^{(i)}$ corresponding to both easy and hard objects. 

The SSVM dual~\eqref{eq:svmstruct_nslack_dual} is a quadratic function with a constant Hessian~$H:=  \regularizerweight A^\transpose A$, so the second-order Taylor expansion of $f$ at a point $\dualvarv$ allows us to rewrite the definition~\ref{eqn:CfProduct} as
\begin{align*}
\notag
\Cf^{(i)} &= \sup_{\substack{\dualvarv \in \domain\\ \sv_{(i)} \in \domain^{(i)}}}
(\sv_{[i]}-\dualvarv_{[i]})^\transpose   H (\sv_{[i]}-\dualvarv_{[i]}) \\ %
 &= \regularizerweight\!\!\!\! \sup_{\substack{\dualvarv \in \domain\\ \sv_{(i)} \in \domain^{(i)}}}
\!\!\!\! \Vert A (\sv_{[i]}-\dualvarv_{[i]}) \Vert_2^2 \\
 &= \regularizerweight  \max_{k \, , \, l}  \| \featuremapv(\inputvarv_i,k)-\featuremapv(\inputvarv_i,l) \|_2^2 .
\end{align*}
The last line uses the property that the maximum of a convex function over a convex set is obtained at a vertex.

In the case of the hard object, we can get
\begin{equation} \label{eq:Cf1HardBlock}
\Cf^{(1)} = \regularizerweight  \Vert \featuremapv(\inputvarv_1,1)-\featuremapv(\inputvarv_1,2)\Vert_2^2 = \frac{1}{n}.
\end{equation}
In the case of an easy object, we can get 
\begin{equation}  \label{eq:CfiEasyBlock}
\Cf^{(i)} = \regularizerweight  \Vert \featuremapv(\inputvarv_i,1)-\featuremapv(\inputvarv_i,0)\Vert_2^2 = \frac{1}{n}.
\end{equation}

\paragraph{Adaptive and non-adaptive sampling.}
Let~$t$ be the number of steps needed on the hard block. By~\eqref{eq:HardBlockComplexity}, we need $t \geq \Omega(\min\{K, \frac{1}{n\epsilon}\} )$ to get a suboptimality smaller than~$\epsilon$. 
The uniform sampling scheme visits all the objects with the same probability.
In the setting constructed above, it makes, on average, $t$ visits to each easy object prior to visiting the hard object $t$ times.
Thus, the overall scheme will call the max-oracle $O(nt)$ times.
All the curvature constants~$\Cf^{(i)}$ are equal, so the sampling proportional to the curvature constants is equivalent to uniform sampling.

The adaptive sampling scheme visits each easy object only once after the first visit to any of them.
After such a visit to any easy object, its local gap estimate equals zero and this object is never visited again.
The gap sampling scheme thus makes an overall $O(n + t)$ oracle calls.
The adaptive scheme is thus approximately $\min\{n,t\} = \min\{n, K, \frac{1}{n \epsilon} \}$ times faster than the non-adaptive ones.
The speed-up can be made arbitrary large by setting both $n$ and $K$ large enough, and $\epsilon$ small enough.

\paragraph{Lipschitz and curvature constants.}
The non-uniform sampling techniques used in the work of~\citet{Nesterov:2012fa,needell2014,ZhaoImportanceSampling_ICML15} use Lipschitz constants of partial derivatives to obtain the sampling probabilities.
In our discussion above, we use the curvature constants.
\citet[Appendix C]{LacosteJulien2015linearFW} note that the curvature constants are affine invariant quantities and, thus, are more suited for the analysis of Frank-Wolfe methods compared to Lipschitz constants (which depend on a choice of norm).
We illustrate this point on our toy example by explicitly computing the Lipschitz constants over blocks for the $\ell_2$ and $\ell_1$ norm.
For both easy and hard blocks, the Lipschitz constant of the gradient with respect to the $\ell_2$ norm equals the largest eigenvalues of the corresponding block Hessians.
For an easy object, the block Hessian is a rank one matrix with the only non-zero eigenvalue equal to $\regularizerweight(|\outputdomain_i|-1) = \frac{K}{n}$.
For the hard object, the block Hessian is a diagonal matrix with non-zero entries equal to~$\frac{\regularizerweight}{2} = \frac{1}{2n}$. 
Here the Lipschitz constant for the easy block is about $K$ times bigger than the one for the hard block, and thus for a large number of labels~$K$, sampling according to Lipschitz constants can be much slower than sampling according to the curvature constants, which was itself slower than the adaptive sampling scheme.

This poor scaling of the Lipschitz constants is partly due to the bad choice of norm in relationship to the optimization domain. \citetsup{aspremont2013optimalFW} suggests to use the atomic norm of the domain~$\domain$ for the analysis. In the case of the simplex, we get the $\ell_1$ norm to measure the diameter of the domain, and its dual norm ($\ell_\infty$) to measure the Lipschitz constant of the gradient. With this norm, the Lipschitz constant stays as $\frac{1}{2n}$ for the hard block, but decreases to the more reasonable $\frac{1}{n}$ for the easy blocks. As explained in footnote~\ref{foot:Cfi}, we can use the bound \mbox{$\Cf^{(i)} \leq L_i \big(\diam_{\|\cdot\|_i} \domain^{(i)} \big)^2$} for the curvature constant. As the diameter for the simplex measured with the $\ell_1$-norm is~$2$, we get the bound~$\Cf^{(1)} \leq \frac{2}{n}$ for the hard block, very close to its exact value of~$\frac{1}{n}$ as derived in~\eqref{eq:Cf1HardBlock}. The $\ell_1$ norm thus appears as a more appropriate choice for this problem.

\begin{algorithm}[t]
    \caption{Block-coordinate Frank-Wolfe (BCFW) algorithm with gap sampling for structured SVM}%
    \label{alg:FW_product_SVM_gapsampling}
\begin{algorithmic}[1]
        \STATE Let $\weightv^{(0)}\!\!:=\!{\weightv_i}^{(0)}\!\!:=\!\0$; \: $\ell^{(0)}\!\!:=\!{\ell_i}^{(0)}\!\!:=\!0$; \: $g_i^{(0)}\!\!:=\!+\infty$;
        \STATE $k_i\!:=\!0$
        \COMMENT{the last time $g_i$ was computed}
       \FOR{$k:=0,\dots,\infty$}
                \STATE Pick $i$ at random with probability $\propto~g_i^{(k_i)}$
                \STATE Solve $\outputvarv_i^* := \displaystyle\argmax_{\outputvarv\in\outputdomain_i} \ H_i(\outputvarv;\weightv^{(k)})$  \label{alg:FW_product_SVM_gap_sampling:max_oracle} %
                \STATE Let $k_i := k$
                \STATE Let $\weightv_{\sv} := \frac1{\regularizerweight n} \featuremapdiffv_i(\outputvarv_i^*)$ \;
                and \; $\ell_{\sv} := \frac1n \errorterm_i(\outputvarv_i^*)$
               \STATE Let $g_i^{(k_i)} :=  \regularizerweight (\weightv_i^{(k)}-\weightv_{\sv})^\transpose \weightv^{(k)} - \ell_i^{(k)} + \ell_{\sv}$\label{alg:FW_product_SVM_gap_sampling:block_gap}
               \STATE {\small Let $\stepsize := \frac{ g_i^{(k_i)}}{ \regularizerweight \|\weightv_i^{(k)}-\weightv_{\sv}\|^2}$~and clip to $[0,1]$}\label{alg:FW_product_SVM_gap_sampling:line_search}
                \STATE Update ${\weightv_i}^{(k+1)}:= (1-\stepsize){\weightv_i}^{(k)}+\stepsize \,\weightv_{\sv}$
                \STATE {\small~~~~~~~and~ ${\ell_i}^{(k+1)}:= (1-\stepsize){\ell_i}^{(k)}+\stepsize\, \ell_{\sv}$}
                \STATE Update $\weightv^{(k+1)}\;:= \weightv^{(k)} + {\weightv_i}^{(k+1)} - {\weightv_i}^{(k)}$
                \STATE {\small~~~~~~~and~~ $\ell^{(k+1)}:= ~\ell^{(k)}+{\ell_i}^{(k+1)} \ \  - {\ell_i}^{(k)}$}
              \IF {update global gap}
                \FOR{$i:=1,\dots,n$} 
                    \STATE Let $k_i:= k+1$
                    \STATE Solve $\outputvarv_i^* := \displaystyle\argmax_{\outputvarv\in\outputdomain_i} \ H_i(\outputvarv;\weightv^{(k_i)})$ 
                    \STATE Let $\weightv_{\sv} := \frac1{\regularizerweight n} \featuremapdiffv_i(\outputvarv_i^*)$ \; 
                 and \; $\ell_{\sv} := \frac1n \errorterm_i(\outputvarv_i^*)$
                    \STATE $g_i^{(k_i)} := \regularizerweight (\weightv_i^{(k_i)}-\weightv_{\sv})^\transpose \weightv^{(k_i)} - \ell_i^{(k_i)} + \ell_{\sv}$
                \ENDFOR
              \ENDIF
        \ENDFOR
\end{algorithmic}
\end{algorithm}

\section{Detailed algorithms. \label{app:algorithms}}
In this section, we give the detailed versions of our BCFW variants applied to the SSVM objective presented in the main paper. Algorithm~\ref{alg:FW_product_SVM_gapsampling} describes BCFW with adaptive gap sampling. We give the block-coordinate version of pairwise FW (BCPFW) in Algorithm~\ref{alg:pFW_product_SVM_pairwise}, and of away-step FW (BCAFW) in Algorithm~\ref{alg:aFW_product_SVM_away_steps}. We note that these two algorithms are simply the \emph{blockwise} application of the PFW and AFW algorithms as described in~\citet{LacosteJulien2015linearFW}, but in the context of SSVM which complicates the notation.  Algorithm~\ref{alg:FW_product_SVM_caching} presents the BCFW algorithm with caching. Algorithm~\ref{alg:rpAlgorithmSSVM} presents our method for computing the regularization path and Algorithm~\ref{alg:rpInitSSVM} presents the initialization of the regularization path.

Note that the three modifications proposed in our paper (gap sampling, caching, pairwise/away steps) can be straightforwardly put together in any combination. In our experimental study, we evaluate all the possibilities.

When using gap sampling or caching and to guarantee convergence, we have to do a full pass over the data every so often to refresh the global gap estimates and to compensate for the staleness effect.
In the experiments, we perform this computation every 10 passes over the data (this is the ``update global gap'' condition in the algorithms). This global gap can also be used as a certificate (upper bound) on the current suboptimality. We thus use the same frequency of global gap computation (every 10 passes) when we run a SSVM solver with a specific convergence tolerance threshold. This is used in our regularization path algorithm which runs a SSVM solver up to a fixed convergence tolerance at each breakpoint.

In our description of the regularization path algorithms (Algorithm~\ref{alg:rpInitSSVM} and Algorithm~\ref{alg:rpAlgorithmSSVM}), we explicitly describe how to update the active sets over the dual variables when the regularization parameter is updated. This is needed when using a SSVM solver that requires the active set over the dual variables (such as BCPFW or BCAFW). When using the simpler BCFW solver, then lines~\ref{alg:rpInitSSVM:activeSetBegin}--\ref{alg:rpInitSSVM:activeSetEnd} of Algorithm~\ref{alg:rpInitSSVM}
and lines~\ref{alg:rpAlgorithmSSVM:activeSetBegin}--\ref{alg:rpAlgorithmSSVM:activeSetEnd} of Algorithm~\ref{alg:rpAlgorithmSSVM} can simply be omitted.

\begin{algorithm}[t]
    \caption{Block-coordinate pairwise Frank-Wolfe (BCPFW) algorithm for structured SVM}%
    \label{alg:pFW_product_SVM_pairwise}
\begin{algorithmic}[1]
        \STATE Let $\weightv^{(0)}:= {\weightv_i}^{(0)}:= \0$; \; $\ell^{(0)}:={\ell_i}^{(0)}:=0$; \;  
        \STATE ${\activeS_i}^{(0)}:=\{\outputvarv_i\}$;
        \COMMENT{active sets}
        \STATE $\dualvar_i^{(0)}(\outputvarv):=0$, $\outputvarv\neq \outputvarv_i$; \; $\dualvar_i^{(0)}(\outputvarv_i):=1$
       \FOR{$k:=0,\dots,\infty$}
                \STATE Pick $i$ at random in $\{1,\ldots,n\}$ 
                \STATE Solve $\outputvarv_i^* := \displaystyle\argmax_{\outputvarv\in\outputdomain_i} \ H_i(\outputvarv;\weightv^{(k)})$ %
                     \COMMENT{FW corner}
                \STATE Let $\weightv_{\sv} := \frac1{\regularizerweight n} \featuremapdiffv_i(\outputvarv_i^*)$ \;
                and \; $\ell_{\sv} := \frac1n \errorterm_i(\outputvarv_i^*)$
                \STATE Solve $\outputvarv_i^a := \displaystyle\argmin_{\outputvarv\in\activeS_i^{(k)}} \ H_i(\outputvarv;\weightv^{(k)})$ \COMMENT{away corner}
                \STATE Let $\weightv_{\awayv} := \frac1{\regularizerweight n} \featuremapdiffv_i(\outputvarv_i^a)$ \;
                and  $\ell_{\awayv} := \frac1n \errorterm_i(\outputvarv_i^a)$  
                \STATE Let $\weightv_{\bm{d}} :=  \weightv_{\sv}-\weightv_{\awayv}$ \;
                 and \; $\ell_{\bm{d}}=\ell_{\sv}-\ell_{\awayv}$
               \STATE Let $\stepsize := \frac{-\regularizerweight {\weightv_{\bm{d}}}^\transpose \weightv^{(k)} + \ell_{\bm{d}}}{ \regularizerweight \|\weightv_{\bm{d}}\|^2}$~and clip to $[0,\dualvar_i^{(k)}(\outputvarv_i^a)]$
                \STATE Update ${\weightv_i}^{(k+1)}:= {\weightv_i}^{(k)}+\stepsize \,\weightv_{\bm{d}}$
                \STATE {\small~~~~~~~and~ ${\ell_i}^{(k+1)}:= {\ell_i}^{(k)}+\stepsize\, \ell_{\bm{d}}$}
                \STATE Update $\weightv^{(k+1)}\;:= \weightv^{(k)} + {\weightv_i}^{(k+1)} - {\weightv_i}^{(k)}$
                \STATE {\small~~~~~~~and~~ $\ell^{(k+1)}:= ~\ell^{(k)}+{\ell_i}^{(k+1)} \ \  - {\ell_i}^{(k)}$}
                \STATE Update ${\activeS_i}^{(k+1)}\;:= {\activeS_i}^{(k)} \cup \{\outputvarv_i^* \}$ 
                \STATE {\small~~~~~~~and~~ $\dualvar_i^{(k+1)}(\outputvarv_i^a):= \dualvar_i^{(k)}(\outputvarv_i^a)-\stepsize$}
                 \STATE {\small~~~~~~~and~~ $\dualvar_i^{(k+1)}(\outputvarv_i^*):= \dualvar_i^{(k)}(\outputvarv_i^*)+\stepsize$}
               \IF{$\stepsize=\dualvar_i^{(k)}(\outputvarv_i^a)$}
                 \STATE Set ${\activeS_i}^{(k+1)}:= {\activeS_i}^{(k+1)} \setminus \{\outputvarv_i^a \}$ \COMMENT{drop step} 
                \ENDIF
        \ENDFOR
\end{algorithmic}
\end{algorithm}

\begin{algorithm}[t]
    \caption{Block-coordinate away-step Frank-Wolfe (BCAFW) algorithm for structured SVM}%
    \label{alg:aFW_product_SVM_away_steps}
\begin{algorithmic}[1]
        \STATE Let $\weightv^{(0)}:= {\weightv_i}^{(0)}:= \0$; \; $\ell^{(0)}:={\ell_i}^{(0)}:=0$;
        \STATE $\activeS_i^{(0)}:=\{\outputvarv_i\}$;
        \COMMENT{active sets} 
        \STATE $\dualvar_i^{(0)}(\outputvarv):=0$, $\outputvarv\neq \outputvarv_i$; \; $\dualvar_i^{(0)}(\outputvarv_i):=1$
       \FOR{$k:=0,\dots,\infty$}
                \STATE Pick $i$ at random in $\{1,\ldots,n\}$ 
                \STATE Solve $\outputvarv_i^* := \displaystyle\argmax_{\outputvarv\in\outputdomain_i} \ H_i(\outputvarv;\weightv^{(k)})$ %
                     \COMMENT{FW corner}
                \STATE Let $\weightv_{\sv} := \frac1{\regularizerweight n} \featuremapdiffv_i(\outputvarv_i^*)$ \;
                and \; $\ell_{\sv} := \frac1n \errorterm_i(\outputvarv_i^*)$
                 \STATE Solve $\outputvarv_i^a := \displaystyle\argmin_{\outputvarv\in\activeS_i^{(k)}} \ H_i(\outputvarv;\weightv^{(k)})$ \COMMENT{away corner}
                \STATE Let $\weightv_{\awayv} := \frac1{\regularizerweight n} \featuremapdiffv_i(\outputvarv_i^a)$ \;
                and \; $\ell_{\awayv} := \frac1n \errorterm_i(\outputvarv_i^a)$
				\STATE Let $g_i^{FW} :=  \regularizerweight (\weightv_i^{(k)}-\weightv_{\sv})^\transpose \weightv^{(k)} - \ell_i^{(k)} + \ell_{\sv}$ \label{alg:FW_product_SVM_away_steps:block_gap}
                \STATE Let $g_i^{A} :=  \regularizerweight (\weightv_{\awayv}-\weightv_i^{(k)})^\transpose \weightv^{(k)} + \ell_i^{(k)} - \ell_{\awayv}$\label{alg:FW_product_SVM_away_steps:away_gap}\vspace{1.5mm}
				\IF{$g_i^{FW} > g_i^{A}$
} \label{alg:aFW_product_SVM_startA}  \COMMENT{FW step}
                \STATE Let $\stepsize := \frac{g_i^{FW}}{ \regularizerweight \|\weightv_i^{(k)}-\weightv_{\sv}\|^2}$~and clip to $[0,1]$
		  			\STATE Update ${\activeS_i}^{(k+1)}\;:= {\activeS_i}^{(k)} \cup \{\outputvarv_i^* \}$
		  			\STATE {\small~~~~~~~and~~ $\dualvar_i^{(k+1)}(\outputvarv):= (1-\stepsize)\dualvar_i^{(k)}(\outputvarv)$}
                     \STATE {\small~~~~~~~and~~ $\dualvar_i^{(k+1)}(\outputvarv_i^*):= \dualvar_i^{(k+1)}(\outputvarv_i^*)+\stepsize$}
                     \STATE Set ${\activeS_i}^{(k+1)}\;:= \{\outputvarv_i^* \}$ if $\stepsize=1$
		  		\ELSE \COMMENT{away step}
                    \STATE Let $\stepsize := \frac{g_i^{A}}{ \regularizerweight \|\weightv_i^{(k)}-\weightv_{\awayv}\|^2}$~and clip to $[0,\frac{\dualvar_i(\outputvarv_i^a)}{1-\dualvar_i(\outputvarv_i^a)}]$
		  			\STATE Update {\small~ $\dualvar_i^{(k+1)}(\outputvarv):= (1+\stepsize)\dualvar_i^{(k)}(\outputvarv)$}
                     \STATE {\small~~~~~~~and~~ $\dualvar_i^{(k+1)}(\outputvarv_i^a):= \dualvar_i^{(k+1)}(\outputvarv_i^a)-\stepsize$}		  							 \STATE {\small~~~~~~~~and~~ ${\activeS_i}^{(k+1)}\;:= {\activeS_i}^{(k)} \setminus \{\outputvarv_i^a \}$ if $\dualvar_i^{(k+1)}(\outputvarv_i^a)=0$}	
		  			
		  		\ENDIF \label{alg:aFW_product_SVM_endA}
                \STATE Update ${\weightv_i}^{(k+1)}:= {\weightv_i}^{(k)}+\stepsize \,\weightv_{\bm{d}}$
                \STATE {\small~~~~~~~and~ ${\ell_i}^{(k+1)}:= {\ell_i}^{(k)}+\stepsize\, \ell_{\bm{d}}$}
                \STATE Update $\weightv^{(k+1)}\;:= \weightv^{(k)} + {\weightv_i}^{(k+1)} - {\weightv_i}^{(k)}$
                \STATE {\small~~~~~~~and~~ $\ell^{(k+1)}:= ~\ell^{(k)}+{\ell_i}^{(k+1)} \ \  - {\ell_i}^{(k)}$}
        \ENDFOR
\end{algorithmic}
\end{algorithm}

\section{Proof of Theorem~\ref{thm:convTheoremGapSampling} (convergence of BCFW with gap sampling) \label{app:proof_gap_sampling}}

\begin{lemma}[Expected block descent lemma] \label{lemma:convGapSampling:expectedDescent}
	Let $g_j(\dualvarv^{(k)})$ be the block gap for block~$j$ for the iterate $\dualvarv^{(k)}$. Let $\dualvarv^{(k+1)}$ be obtained by sampling a block~$i$ with probability $p_i$ and then doing a (block) FW step with line-search on this block, starting from $\dualvarv^{(k)}$. Consider \emph{any} set of scalars $\stepsize_j \in [0,1]$, $j = 1, \ldots, n$, which do not depend on the chosen block $i$.
    Then in conditional expectation over the random choice of block $i$ with probabilities~$p_i$, it holds:
    \begin{align}  
    \notag\E\big[ f(\dualvarv^{(k+1)})\,|\, \dualvarv^{(k)}\big]
    &\le f(\dualvarv^{(k)}) -  \sum_{i=1}^n{\gamma_i p_i g_i(\dualvarv^{(k)})} \\[-0.1cm]
    \label{mineq} &\quad + \frac{1}{2}\sum_{i=1}^n{\gamma_i^2 p_i \Cf^{(i)}}.
    \end{align}\\[-1.0cm]
\end{lemma}
\begin{proof}
    The proof is analogous to the proof of Lemma~\ref{lem:block_step_improvement}, but being careful with the expectation. Let block~$i$ be the chosen one that defined $\dualvarv^{(k+1)}$ and let $\dualvarv_\stepsize := \dualvarv+\stepsize(\sv_{[i]}-\dualvarv_{[i]})$, where $\mathbf{s}_{(i)}\in \domain^{(i)}$ is the FW corner on block~$i$ and~$\mathbf{s}_{[i]} \in \R^m$ is its zero-padded version. By the line-search, we have $f(\dualvarv^{(k+1)}) \leq f(\dualvarv_\stepsize)$ for any $\stepsize \in [0,1]$. By using $\stepsize = \stepsize_i$ in the bound~\eqref{eqn:CfProduct} provided in the curvature Definition~\ref{def:curvature_const}, and by the definition of the Frank-Wolfe gap, we get:
    \begin{align*}
    f(\dualvarv^{(k+1)}) \leq  f(\dualvarv_{\stepsize_i})
    &= 
    f(\dualvarv^{(k)} + \gamma_i (\mathbf{s}_{[i]} - \dualvarv^{(k)}_{[i]})) \\
    & \leq f(\dualvarv^{(k)}) + \gamma_i g_i(\dualvarv^{(k)}) + \frac{\gamma_i^2}{2}\Cf^{(i)}.
    \end{align*}
    Taking the expectation of the bound with respect to $i$, conditioned on $\dualvarv^{(k)}$, proves the lemma.
\end{proof}

\begin{definition}
    \label{def:nonUnifMeasure}
    The \emph{nonuniformity measure} $\nonuniformity(\bm{x})$ of a vector~$\bm{x}\in\R^n_+$ is defined as:
    \begin{equation*}
    \nonuniformity(\bm{x}) 
    :=  
    \sqrt{1+n^2\Var \big[\bm{p}\big]}
    \end{equation*}
    where~$\bm{p} := \frac{\bm{x}}{\Vert \bm{x} \Vert_1}$ is the probability vector obtained by normalizing~$\bm{x}$.
\end{definition}

\begin{lemma}
    \label{lemma:L1L2}
    Let $\bm{x}\in\R^n_+$.
    The following relation between its $\ell_1$-norm and $\ell_2$-norm holds:
    \begin{equation*}
    \Vert \bm{x} \Vert_2 =  \frac{\nonuniformity(\bm{x})}{\sqrt{n}}\Vert \bm{x} \Vert_1 \, .
    \end{equation*}
\end{lemma}
\begin{proof}
    We have that\\[-0.2cm]
    \begin{equation}
    \label{eq:lemma_L1L2_proof1}
    \Var \big[\bm{p}\big] = \E \big[ \bm{p}^2 \big] - \E \big[ \bm{p} \big]^2 = \frac{1}{n} \| \bm{p} \|_2^2 - \frac{1}{n^2}.
    \end{equation}
    Combining~\eqref{eq:lemma_L1L2_proof1} and Definition~\ref{def:nonUnifMeasure} we prove the lemma.
\end{proof}

\textbf{Remark.} For any~$\bm{x} \in \R^n_+$, the quantity $\nonuniformity(\bm{x})$ always belongs to the segment~$[1,\sqrt{n}]$. 
We have $\nonuniformity(\bm{x})=1$ when all the elements of~$\bm{x}$ are equal and~$\nonuniformity(\bm{x})=\sqrt{n}$ when all the elements, except one, equal zero.

\begin{reptheorem}{thm:convTheoremGapSampling}
    Assume that at each iterate~$\dualvarv^{(k)}$, $k\ge 0$, BCFW with gap sampling (Algorithm~\ref{alg:FW_product_SVM_gapsampling}) has access to the exact values of the block gaps. 
    Then, at each iteration, it holds that
    $
    \E \big[f(\dualvarv^{(k)})\big] - f(\dualvarv^*) \le \frac{2n}{k+2n}\big(\CfTotal \nonuniformityTotal + h_0\big)
    $
    where $\dualvarv^*\in \domain$ is a solution of problem~(\ref{eq:svmstruct_nslack_dual}), $h_0 := f(\dualvarv^{(0)}) - f(\dualvarv^*)$ is the suboptimality at the starting point of the algorithm, the constant~$\CfTotal := \sum_{i=1}^n \Cf^{(i)}$ is the sum of the curvature constants, and the constant~$\nonuniformityTotal$ is an upper bound on $\E \Big[ \frac{\nonuniformity(\Cf^{(:)})}{\nonuniformity(\bm{g}_{:}(\dualvarv^{(k)}))^3} \Big]$, which quantities the amount of non-uniformity of the $\Cf^{(i)}$'s in relationship to the non-uniformity of the gaps obtained during the algorithm.
    The expectations are taken over the random choice of the sampled block at iterations $1,\dots,k$ of the algorithm.
\end{reptheorem}

\begin{proof}
    Starting from Lemma~\ref{lemma:convGapSampling:expectedDescent} with $\stepsize_i:=\stepsize$ for some $\stepsize$ to be determined later and $p_i := \frac{g_i}{g}$ where~$g_i:=g_i(\dualvarv^{(k)})$ and~$g:=g(\dualvarv^{(k)})$, we get\\[-0.4cm]
    \begin{align}  
    \notag
    \E\big[ f(\dualvarv^{(k+1)})\,|\, \dualvarv^{(k)}\big]
    &\le f(\dualvarv^{(k)}) -  \gamma \sum_{i=1}^n \frac{g_i^2}{ g } \\[-0.1cm]
    \label{eq:thrmGapSampling_proof_descent1}
    & + \frac{\gamma^2}{2}\sum_{i=1}^n  \Cf^{(i)} \frac{g_i}{ g }.
    \end{align}\\[-0.3cm]
    The Cauchy-Schwarz inequality bounds the dot product between the vectors of curvature constants~$\bm{c} := \Cf^{(:)} := (\Cf^{(i)})_{i=1}^n$ and block gaps~$\bm{g}:= \bm{g}_{:}(\dualvarv^{(k)}) :=(g_i)_{i=1}^n$\\[-0.3cm]
    \begin{equation}
    \label{eq:thrmGapSampling_proof_CS}
    \sum_{i=1}^n  \Cf^{(i)} g_i
    \leq 
    \| \bm{c} \|_2 \;  \| \bm{g} \|_2 \,.
    \end{equation}\\[-0.2cm]
    Combining~\eqref{eq:thrmGapSampling_proof_CS} and the result of Lemma~\ref{lemma:L1L2} for the vectors of curvature constants and block gaps (with $\| \bm{g} \|_1 = g$ and $\| \bm{c} \|_1 = \CfTotal$), we can further bound~\eqref{eq:thrmGapSampling_proof_descent1}: 
    \begin{align}  
    \notag
    \E\big[ f(\dualvarv^{(k+1)})\,|\, \dualvarv^{(k)}\big]
    &\le f(\dualvarv^{(k)}) -  \frac{\gamma g}{n} \nonuniformity( \bm{g} )^2 \\[-0.1cm]
    \label{eq:thrmGapSampling_proof_descent2}
    & + \frac{\gamma^2}{2n} \nonuniformity( \bm{g} ) \, \nonuniformity( \bm{c} ) \, \CfTotal.
    \end{align}
    Subtracting the minimal function value $f(\dualvarv^*)$ from both sides of~\eqref{eq:thrmGapSampling_proof_descent2} and by using 
    $
    h(\dualvarv^{(k)}) := f(\dualvarv^{(k)}) - f(\dualvarv^*) \leq g
    $,
    we bound the conditional expectation of the suboptimality $h$ with\\[-0.4cm]
    \begin{align}
    \notag \E [ h(\dualvarv^{(k+1)}) \mid \dualvarv^{(k)}] 
    \le\;& 
    h(\dualvarv^{(k)}) - \frac{\stepsize}{n} \nonuniformity( \bm{g} )^2\, h(\dualvarv^{(k)}) 
    \\
    \label{eq:thrmGapSampling_proof_descent3}
    +\;&
    \frac{\stepsize^2}{2n} \nonuniformity( \bm{g} ) \, \nonuniformity( \bm{c} ) \,\CfTotal 
    \end{align}\\[-0.3cm]
    which is analogous to~\citep[Eq.~(20)]{lacosteJulien13bcfw}.
    In what follows, we use the modified induction technique of~\citep[Proof of Theorem~C.1]{lacosteJulien13bcfw}.
    
    By induction, we are going to prove the following upper bound on the unconditional expectation of the suboptimality $h$:\\[-0.2cm]
    \begin{equation}
    \label{eq:thrmGapSampling_proof_inductionAssumption}
    \E \big[ h(\dualvarv^{(k)}) \big] \leq \frac{2n C}{k+2n} , \quad \text{for\;$k\geq0$,}
    \end{equation}
    that corresponds to the statement of the theorem with~$C:= \CfTotal \nonuniformityTotal + h_0$.
    
    The basis of the induction $k=0$ follows immediately from the definition of~$C$, given that~$\CfTotal \geq 0$ and $\nonuniformityTotal > 0$.
    
    Consider the induction step.
    Assume that~\eqref{eq:thrmGapSampling_proof_inductionAssumption} is satisfied for~$k \geq 0$.  
    With a particular choice of step size~$\gamma:= \frac{2n}{\nonuniformity( \bm{g} )^2 (k+2n)} \in[0,1]$ (which does not depend on the picked~$i$), we rewrite the bound~\eqref{eq:thrmGapSampling_proof_descent3} on the conditional expectation as\\[-0.4cm]
    \begin{align}
    \notag \E [ h(\dualvarv^{(k+1)}) \mid \dualvarv^{(k)}] 
    \le\;& 
    \Bigl( 1 - \frac{2}{k+2n} \Bigr) \, h(\dualvarv^{(k)}) 
    \\
    \label{eq:thrmGapSampling_proof_inductionStep1}
    +\;&
    \frac{2n}{(k+2n)^2} \frac{ \nonuniformity( \bm{c} ) \,\CfTotal }{\nonuniformity( \bm{g} )^3} \, .
    \end{align}\\[-0.3cm]
    Taking the unconditional expectation of~\eqref{eq:thrmGapSampling_proof_inductionStep1}, then the induction assumption~\eqref{eq:thrmGapSampling_proof_inductionAssumption} and the definition of~$\nonuniformityTotal$ give us the deterministic inequality\\[-0.6cm]
    \begin{align}
    \notag \E [ h(\dualvarv^{(k+1)}) ] 
    \le\;& 
    \Bigl( 1 - \frac{2}{k+2n} \Bigr) \, \frac{2nC}{k+2n} 
    \\
    \label{eq:thrmGapSampling_proof_inductionStep2}
    +\;&
    \frac{2n}{(k+2n)^2} \nonuniformityTotal \,\CfTotal.
    \end{align}\\[-0.3cm]
    Bounding~$\nonuniformityTotal \,\CfTotal$ by~$C$ and rearranging the terms gives\\[-0.4cm]
    \begin{align}
    \notag \E [ h(\dualvarv^{(k+1)}) ]  
    \le&
    \frac{2nC}{k+2n} \left(1 - \frac{2}{k+2n} + \frac{1}{k+2n}\right)
    \\
    \notag=&
    \frac{2nC}{k+2n} \frac{k+2n-1}{k+2n}
    \\
    \notag\le&
    \frac{2nC}{k+2n} \frac{k+2n}{k+2n+1}
    \\
    \notag=&
    \frac{2nC}{(k+1)+2n} \ ,
    \end{align}\\[-0.3cm]
    which completes the induction proof.  
\end{proof}

\paragraph{Comparison with uniform sampling.} 
We now compare the rates obtained by Theorem~\ref{thm:convTheoremGapSampling} for BCFW with gap sampling and by Theorem~\ref{thm:convergence_FW_product} for BCFW with uniform sampling. The only difference is in the constants: Theorem~\ref{thm:convTheoremGapSampling} has~$\CfTotal \nonuniformityTotal$ and Theorem~\ref{thm:convergence_FW_product} has~$\CfTotal$. 

Recall that by definition\\[-0.4cm]
\begin{equation*}
\nonuniformityTotal = \max_k \;\E \Big[ \frac{\nonuniformity(\Cf^{(:)})}{\nonuniformity(\bm{g}_{:}(\dualvarv^{(k)}))^3} \Big]
\end{equation*}\\[-0.3cm]
In the best case for gap sampling, the curvature constants are uniform, $\nonuniformity(\Cf^{(:)})=1$, and the gaps are nonuniform~$\nonuniformity(\bm{g}_{:}(\dualvarv^{(k)})) \approx \sqrt{n}$. Thus, $\nonuniformityTotal \approx \frac{1}{n\sqrt{n}}$.

In the worst case for gap sampling, the curvature constants are very non-uniform, $\nonuniformity(\Cf^{(:)}) \approx \sqrt{n}$. The constant for gap sampling is still better if the gaps are non-uniform enough, i.e., $\nonuniformity(\bm{g}_{:}(\dualvarv^{(k)})) \geq n^{\frac{1}{6}}$. 

We note that to design a sampling scheme that always dominates uniform sampling (in terms of bounds at least), we would need to include the $\Cf^{(i)}$'s in the sampling scheme (as was essentially done by~\citet{Csiba15adaSDCA} for SDCA). Unfortunately, computing good estimates for $\Cf^{(i)}$'s is too expensive for structured SVM, thus motivating our simpler yet practically efficient scheme. See also the discussion after~\eqref{eq:expectedImprovement}.

\section{Proof of Theorem~\ref{thm:cacheTheorem} (convergence of BCFW with caching)\label{app:caching_theorem}}
\begin{reptheorem}{thm:cacheTheorem} \label{thm:convergence_cache_supp}
Let~$\tilde{\cacheN} := \frac1n\cacheN \leq 1$. 
Then, for each $k\ge 0$, the iterate $\dualvarv^{(k)}$ of %
Algorithm~\ref{alg:FW_product_SVM_caching} satisfies
$
\E\big[f(\dualvarv^{(k)})\big] - f(\dualvarv^*) \le \frac{2n}{\tilde{\cacheN}k+2n}\big(\frac{1}{\tilde{\cacheN}}\CfTotal+ h_0\big)
$
where $\dualvarv^*\in \domain$ is a solution of problem~(\ref{eq:svmstruct_nslack_dual}), $h_0 := f(\dualvarv^{(0)}) - f(\dualvarv^*)$ is the suboptimality at the starting point of the algorithm, $\CfTotal := \sum_{i=1}^n \Cf^{(i)}$ is the sum of the curvature constants (see Definition~\ref{def:curvature_const}) of $f$ with respect to the domains~$\domain^{(i)}$ of individual blocks.
The expectation is taken over the random choice of the sampled blocks at iterations $1,\dots,k$ of the algorithm.
\end{reptheorem}
\begin{proof}
The key observation of the proof consists in the fact that the combined oracle (the cache oracle in the case of a cache hit and the max oracle in the case of a cache miss) closely resembles an oracle with multiplicative approximation error~\citep[Eq.~(12) of Appendix~C]{lacosteJulien13bcfw}.

In the case of a cache hit, Definition~\ref{def:curvature_const} of curvature constant for any step size $\gamma \in [0,1]$ gives us
\begin{align*}
 &f(\dualvarv^{(k+1)}_\stepsize) := f(\dualvarv^{(k)}+\stepsize(\cachev_{[i]}-\dualvarv^{(k)}_{[i]})) \\
    \leq \; & f(\dualvarv^{(k)}) + \stepsize \langle \cachev_{(i)}\!-\!\dualvarv^{(k)}_{(i)}, \nabla_{(i)} f(\dualvarv^{(k)})\rangle + \frac{\stepsize^2}{2}\Cf^{(i)} \\
     = \;& f(\dualvarv^{(k)})- \stepsize \cachegap_i^{(k)} + \frac{\stepsize^2}{2}\Cf^{(i)} \\
     \leq \;& f(\dualvarv^{(k)}) - \gamma \tilde{\cacheN} g^{(k_0)} + \frac{\stepsize^2}{2}\Cf^{(i)}
\end{align*}
where the corner~$\cachev_{(i)} \in \domain^{(i)}$ and its zero-padded version~$\cachev_{[i]} \in \R^m$ are provided by the cache oracle, and $\tilde{\cacheN} = \frac1n \cacheN$ is the constant controlling the global part of the cache-hit criterion.
In the case of a cache miss, similarly to Lemma~\ref{lem:block_step_improvement}, we get
\begin{equation*}
    f(\dualvarv^{(k+1)}_\stepsize) \leq f(\dualvarv^{(k)}) - \stepsize g_i^{(k)} + \frac{\stepsize^2}{2}\Cf^{(i)} \ .
\end{equation*}

Combining the two cases we get 
\begin{equation}
\label{eq:cache_proof:block_descent}
    f(\dualvarv^{(k+1)}_\stepsize) \leq f(\dualvarv^{(k)}) - \stepsize \tilde{g}_i^{(k)} + \frac{\stepsize^2}{2}\Cf^{(i)}
\end{equation}
where
\[
\tilde{g}_i^{(k)}\! := [i \text{ is a cache miss}] g_i^{(k)}  +   [i \text{ is a cache hit}] \tilde{\cacheN} g^{(k_0)}.
\]

\begin{algorithm}[t!]
    \caption{Block-coordinate Frank-Wolfe (BCFW) algorithm with cache for structured SVM \label{alg:FW_product_SVM_caching}}%
    \begin{algorithmic}[1]
        \STATE Let $\weightv^{(0)}\!:=\!{\weightv_i}^{(0)}\!:=\!\0$; \; $\ell^{(0)}\!:=\!{\ell_i}^{(0)}\!:=\!0$; \; $\cache_i\!:=\!\{\outputvarv_i\}$;
        \STATE $g^{(0)}\!:=\!g_i^{(0)}\!=\!+\infty$
        \STATE $k_0\!:=\!k_i\!:=\!0$ ; \; \COMMENT{the last time $g$ / $g_i$ was computed}
        \FOR{$k:=0,\dots,\infty$}
        \STATE Pick $i$ at random in $\{1,\ldots,n\}$ \label{alg:FW_product_SVM_caching:randomSample}
        \COMMENT{either uniform or \\ \hfill with probability $\propto g_i^{(k_i)}$ for gap sampling}
        \STATE Solve $\outputvarv_i^c := \argmax_{\outputvarv\in\cache_i} \ H_i(\outputvarv;\weightv)$ 
        	\COMMENT{cache corner}
        \STATE Let $\weightv_{\cachev} := \frac{ \featuremapdiffv_i(\outputvarv_i^c)}{\regularizerweight n}$ \; 
        and \; $\ell_{\cachev}:=\frac1n L_i(\outputvarv_i^c)$
        \STATE Let $\cachegap_i^{(k)} := \regularizerweight(\weightv_i^{(k)}-\weightv_{\cachev})^\transpose \weightv^{(k)}-{\ell_i}^{(k)}+\ell_{\cachev}$
        \IF{$\cachegap_i^{(k)} \geq \max(\cacheF g_i^{(k_i)},\frac{\cacheN}{n} g^{(k_0)})$} \COMMENT{cache hit}
        \STATE $\weightv_{\sv}:= \weightv_{\cachev}$, $\ell_{\sv} := \ell_{\cachev}$, $\cachegap_i := \cachegap_i^{(k)}$
        \ELSE \COMMENT{cache miss}
        \STATE Solve $\outputvarv_i^* := \displaystyle\argmax_{\outputvarv\in\outputdomain_i} \ H_i(\outputvarv;\weightv^{(k)})$  \COMMENT{FW corner}
        \STATE Let $\weightv_{\sv} := \frac1{\regularizerweight n} \featuremapdiffv_i(\outputvarv_i^*)$ \;
        and\; $\ell_{\sv} := \frac1n \errorterm_i(\outputvarv_i^*)$
        \STATE Let $g_i^{(k)} :=  \regularizerweight (\weightv_i^{(k)}-\weightv_{\sv})^\transpose \weightv^{(k)} - \ell_i^{(k)} + \ell_{\sv}$\label{alg:FW_product_SVM_caching:block_gap}
        \STATE Set $k_i := k$, $\cachegap_i := g_i^{(k)}$
        \STATE Update ${\cache_i}\;:= {\cache_i} \cup \{\outputvarv_i^* \}$  
        \ENDIF 
        \STATE Let $\stepsize := \frac{ \cachegap_i }{ \regularizerweight \|\weightv_i^{(k)}-\weightv_{\sv}\|^2}$~and clip to $[0,1]$ \label{alg:FW_product_SVM_caching:line_search}
        \STATE Update ${\weightv_i}^{(k+1)}:= (1-\stepsize){\weightv_i}^{(k)}+\stepsize \,\weightv_{\sv}$
        \STATE {\small~~~~~~~and~ ${\ell_i}^{(k+1)}:= (1-\stepsize){\ell_i}^{(k)}+\stepsize\, \ell_{\sv}$}
        \STATE Update $\weightv^{(k+1)}\;:= \weightv^{(k)} + {\weightv_i}^{(k+1)} - {\weightv_i}^{(k)}$
        \STATE {\small~~~~~~~and~~ $\ell^{(k+1)}:= ~\ell^{(k)}+{\ell_i}^{(k+1)} \ \  - {\ell_i}^{(k)}$}
        \IF {update global gap}
        \STATE Let $g^{(k_0)} := 0$, $k_0 := k+1$
        \FOR{$i:=1,\dots,n$} 
        \STATE Solve $\outputvarv_i^* := \displaystyle\argmax_{\outputvarv\in\outputdomain_i} \ H_i(\outputvarv;\weightv^{(k_0)})$ 
        \STATE Let $\weightv_{\sv} := \frac1{\regularizerweight n} \featuremapdiffv_i(\outputvarv_i^*)$ \;
        and \; $\ell_{\sv} := \frac1n \errorterm_i(\outputvarv_i^*)$
        \STATE $g^{(k_0)} +\!\!= \regularizerweight (\weightv_i^{(k_0)}-\weightv_{\sv})^\transpose \weightv^{(k_0)} - \ell_i^{(k_0)} + \ell_{\sv}$
        \STATE Set $k_i:=k_0$
        \ENDFOR
        \ENDIF
        \ENDFOR
    \end{algorithmic}
\end{algorithm}

Subtracting~$f(\dualvarv^*)$ from both sides of~\eqref{eq:cache_proof:block_descent} and taking the expectation of~\eqref{eq:cache_proof:block_descent} w.r.t.\ the block index~$i$ we get
\begin{equation}
\label{eq:cache_proof:block_descent_expected_h}
    \E\big[ h(\dualvarv^{(k+1)}_\stepsize)\,|\, \dualvarv^{(k)}\big] \leq h(\dualvarv^{(k)}) - \frac{\stepsize}{n} \tilde{g}^{(k)} + \frac{\stepsize^2}{2n}\CfTotal
\end{equation}
where~$h(\dualvarv):=f(\dualvarv)-f(\dualvarv^*)$ is the suboptimality of the function~$f$ and $\tilde{g}^{(k)} := \sum_{i=1}^n\tilde{g}_i^{(k)}$.
We know that the duality gap upper-bounds the suboptimality, i.e., $g(\dualvarv)\geq h(\dualvarv)$, and that cache miss steps, as well as cache hit steps, always decrease suboptimality, i.e., $h(\dualvarv^{(k)}) \leq h(\dualvarv^{(k_0)})$. 

If at iteration~$k$ there is at least one cache hit, then we can bound the quantity~$\tilde{g}^{(k)}$ from below:\\[-2mm]
\begin{equation}
\label{eq:cache_proof:case_one_cache_hit}
\tilde{g}^{(k)} \geq \tilde{\cacheN} g(\dualvarv^{(k_0)}) \geq \tilde{\cacheN} h(\dualvarv^{(k_0)}) \geq \tilde{\cacheN} h(\dualvarv^{(k)}).
\end{equation}
In the case of no cache hits, we have\\[-2mm]
\[
\tilde{g}^{(k)} = g(\dualvarv^{(k)}) \geq h(\dualvarv^{(k)}) \geq \tilde{\cacheN} h(\dualvarv^{(k)})
\]
where the last inequality holds because~$\tilde{\cacheN} \leq 1$. Applying the lower bound on~$\tilde{g}^{(k)}$ to~\eqref{eq:cache_proof:block_descent_expected_h}, we get
\begin{equation}
\begin{aligned}
\label{eq:cache_proof:block_descent_expected_h_bounded}
    \E\big[ h({ \dualvarv^{(k+1)}_\stepsize})\,|\, { \dualvarv^{(k)}}\big] &\leq h(\dualvarv^{(k)}) - \frac{\stepsize\tilde{\cacheN}}{n} h(\dualvarv^{(k)}) \\ &\quad+ \frac{\stepsize^2}{2n}\CfTotal.
\end{aligned}
\end{equation}

\begin{algorithm}[t!]
    \caption{{\sc init-reg-path}: Initialization of the regularization path for structured SVM}
    \label{alg:rpInitSSVM}
    \begin{algorithmic}[1]
        \INPUT $\factorRegPath$, tolerance $\epsilon$ %
        \STATE  $\weightv := \weightv_i := 0$; \; $\ell := \ell_i := 0$; \; $\tilde{\featuremapdiffv} := 0$
        \FOR{$i := 1, \dots , n$}
        \STATE $\tilde{\outputvarv}_i := \argmax_{\outputvarv \in \outputdomain_i} H_i (\outputvarv; \0)$
        \STATE $\ell_i := \frac1n \errorterm(\outputvarv_i, \tilde{\outputvarv}_i)$
        \STATE $\ell := \ell + \ell_i$
        \STATE $\tilde{\featuremapdiffv} := \tilde{\featuremapdiffv} + \frac1n \featuremapdiffv(\tilde{\outputvarv}_i)$
        \ENDFOR
        \FOR{$i := 1, \dots, n$}
        \STATE $\theta_i := \max_{\outputvarv \in \mathcal{Y}_i}  \big( - \tilde{\featuremapdiffv}^{\transpose} \featuremapdiffv(\outputvarv) \big)$
        \ENDFOR
        \STATE Let $\regularizerweight^\infty := \frac{1}{\factorRegPath\epsilon}\bigg( \|\tilde{\featuremapdiffv}\|^2 + \frac1n \sum_{i=1}^n \theta_i \bigg)$
        \STATE Let $\weightv := \frac{1}{\regularizerweight^\infty} \tilde{\featuremapdiffv}$; \; $\weightv_i := \frac{1}{n \regularizerweight^\infty} \featuremapdiffv(\tilde{\outputvarv}_i)$
        \STATE \textbf{for} $i:=1,\dots,n$ \;\textbf{do}\; $g_i := \frac{1}{\numsamples \regularizerweight^\infty}\theta_i + \regularizerweight^\infty \weightv_i^{\transpose} \weightv$
        \FOR{$i:=1,\dots,n$} \COMMENT{optional} \label{alg:rpInitSSVM:activeSetBegin}
        \STATE $\activeS_i := \{\tilde{\outputvarv}_i\}$ 
        \STATE $\dualvar_i(\tilde{\outputvarv}_i):=1$ and $\dualvar_i(\outputvarv):=0$ for $\outputvarv \neq \tilde{\outputvarv}_i$
        \ENDFOR \label{alg:rpInitSSVM:activeSetEnd}
        \STATE \textbf{return} $\weightv$, $\weightv_i$, $\ell$, $\ell_i$, $g_i$, $\regularizerweight^\infty$, $\activeS_i$, $\dualvarv$
    \end{algorithmic}
\end{algorithm}

\begin{algorithm}[t!]
    \caption{Regularization path for structured SVM}
    \label{alg:rpAlgorithmSSVM}
    \begin{algorithmic}[1]
        \INPUT $\factorRegPath$, tolerance $\epsilon$, $\regularizerweight_{min}$
        \STATE Initialize regularization path using Algorithm~\ref{alg:rpInitSSVM}. $\{\weightv^0\!\!, \weightv_i^0, \ell^0\!\!, \ell_i^0, g_i^0, \regularizerweight^0\!\!, \activeS_i^0, \dualvarv^0 \}$ \!:=\! {\sc init-reg-path} $(\factorRegPath, \epsilon)$
        \STATE $J := 0$
        \REPEAT
        \STATE \textbf{For}\; $i := 1,\dots,n$ \;\textbf{do}\; $\delta_i := \ell_i^J -\regularizerweight^J \langle\weightv^J, \weightv_i^J\rangle $
        \STATE Compute excess gap\; $\tau := \epsilon - \sum_{i=1}^n g_i^J$
        \STATE Let\; $\Delta := \sum_i^{n} \delta_i$
        \IF{$\Delta \leq  \tau$}
        \STATE Let $\regPathStep := 1 - \frac{\tau}{\Delta}$
        \ELSE
        \STATE $\weightv^J$ is $\epsilon$-approximate for any $\regularizerweight < \regularizerweight^J$
        \STATE \textbf{return}\quad$\{\regularizerweight^j\}_{j=0}^J$, $\{\weightv^j\}_{j=0}^J$
        \ENDIF
        \STATE Let\; $\regularizerweight^{J+1} := \regPathStep\regularizerweight^{J}$,\; $\ell^{J+1} := \regPathStep\ell^{J}$,\; 
        \FOR{$i:=1,\dots,n$} \COMMENT{update gaps using~\eqref{eq:ssvm_rp_gapnew}}
        \STATE Let\; $g^{J+1}_i := g^{J}_i + (1 - \regPathStep ) \delta_i$ \; and \; $\ell_i^{J+1} := \regPathStep\ell_i^{J}$
        \ENDFOR
        \FOR{$i:=1,\dots,n$} \COMMENT{optional: update duals} \label{alg:rpAlgorithmSSVM:activeSetBegin}
            \STATE $\activeS_i^{J+1} := \activeS_i^{J} \cup \{\outputvarv_i\}$
            \STATE $\dualvar_i^{J+1}(\outputvarv) := \regPathStep \dualvar_i^{J}(\outputvarv)$ for  $\outputvarv \in \activeS_i^{J+1} \setminus \{\outputvarv_i\}$
            \STATE $\dualvar_i^{J+1}(\outputvarv_i):=1-\sum_{\outputvarv \in \activeS_i^{J+1}\setminus \{\outputvarv_i\}} \dualvar_i^{J+1}(\outputvarv)$
        \ENDFOR \label{alg:rpAlgorithmSSVM:activeSetEnd}
        \STATE Run SSVM-optimizer with tolerance $\factorRegPath\,\epsilon$ \label{alg:rpAlgorithmSSVM:SSVM} \\ \quad\; to update $\weightv^{J+1}$\!\!, $\weightv_i^{J+1}$\!\!, $\ell^{J+1}$\!\!, $\ell_i^{J+1}$\!\!, $g_i^{J+1}$\!\!, $\activeS_i^{J+1}$\!\!, $\dualvarv^{J+1}$\!\! 
        \COMMENT{to have $\epsilon$-appr. path, gaps $g_i^{J+1}$ have to be exact}
        \STATE $J := J+1$
        \UNTIL{$\regularizerweight^{J+1} < \regularizerweight_{min}$}
        \STATE \textbf{return}\quad$\{\regularizerweight^j\}_{j=0}^J$, $\{\weightv^j\}_{j=0}^J$
    \end{algorithmic}
\end{algorithm}

Inequality~\eqref{eq:cache_proof:block_descent_expected_h_bounded} is identical to the inequality~\citep[Eq.~(20)]{lacosteJulien13bcfw} in the proof of convergence of BCFW with a multiplicative approximation error in the oracle.
We recopy their argument below for reference to finish the proof.
First, we take the expectation of~\eqref{eq:cache_proof:block_descent_expected_h_bounded} w.r.t.\ the choice of previous blocks:
\begin{equation}
\label{eq:cache_proof:block_descent_expected_h_bounded2}
    \E\big[ h(\dualvarv^{(k+1)}_\stepsize)] \leq (1-\frac{\stepsize\tilde{\cacheN}}{n})\E\big[h(\dualvarv^{(k)})]+ \frac{\stepsize^2}{2n}\CfTotal.
\end{equation}

Following the proof of Theorem~C.1 in~\citet{lacosteJulien13bcfw}, we prove the bound of Theorem~\ref{thm:cacheTheorem} by induction. The induction hypothesis consists in inequality
\begin{equation*}
\E\big[h(\dualvarv^{(k)})] \le \frac{2nC}{\tilde{\cacheN}k+2n} \, \text{ for } k\geq 0
\end{equation*}
where $C:=\big(\frac{1}{\tilde{\cacheN}}\CfTotal+ h_0\big)$. 

The base-case $k=0$ follows directly from $C\geq h_0$.
We now prove the induction step. Assume that the hypothesis is true for a given $k\geq 0$. 
Let us now prove that the hypothesis is true for $k+1$.
We use inequality~\eqref{eq:cache_proof:block_descent_expected_h_bounded} with the step size $\stepsize_k:=\frac{2n}{\tilde{\mapprox} k+2n} \in[0,1]$:
\begin{align*}
\E\big[ h(\dualvarv^{(k+1)}_{\stepsize_k})] \le&\: (1- \frac{\stepsize_k\tilde{\mapprox}}{n}) \E\big[h(\dualvarv^{(k)})] + (\stepsize_k)^2 \frac{C\tilde{\mapprox}}{2n} \\[3pt]
=& \: (1-\frac{2\tilde{\mapprox}}{\tilde{\mapprox} k+2n}) \E\big[h(\dualvarv^{(k)})] + (\frac{2n}{\tilde{\mapprox} k+2n})^2 \frac{C\tilde{\mapprox}}{2n}  \\[3pt]
\le& \: (1-\frac{2\tilde{\mapprox}}{\tilde{\mapprox} k+2n}) \frac{2n C}{\tilde{\mapprox} k+2n} + (\frac{1}{\tilde{\mapprox} k+2n})^2 2nC\tilde{\mapprox}
\end{align*}
where, in the first line, we use inequality $\CfTotal \le C \tilde{\mapprox}$, and, in the last line, we use the induction hypothesis for $\E\big[h(\dualvarv^{(k)})]$.

By rearranging the terms, we have
\begin{align*}
    \E\big[ h(\dualvarv^{(k+1)})] \le& \frac{2nC}{\tilde{\mapprox} k+2n} \left(1 - \frac{2\tilde{\mapprox}}{\tilde{\mapprox} k+2n} + \frac{\tilde{\mapprox}}{\tilde{\mapprox} k+2n}\right) \\[3pt]
                      =& \frac{2nC}{\tilde{\mapprox} k+2n} \frac{\tilde{\mapprox} k+2n-\tilde{\mapprox}}{\tilde{\mapprox} k+2n} \\[3pt]
                     \le& \frac{2nC}{\tilde{\mapprox} k+2n} \frac{\tilde{\mapprox} k+2n}{\tilde{\mapprox} k+2n+\tilde{\mapprox}} \\[3pt]
                      =& \frac{2nC}{\tilde{\mapprox} (k+1)+2n} \ ,
\end{align*}  
which finishes the proof.
\end{proof}

\section{Convergence of BCPFW and BCAFW} \label{app:BCPFWconvergence}

In this section, we prove Theorem~\ref{thm:BCPFW} that states that the suboptimality error on~\eqref{eq:svmstruct_nslack_dual} decreases geometrically in expectation for BCPFW and BCAFW for the iterates at which \emph{no} block would have a drop step, i.e., when no atom would be removed from the active sets. We follow closely the notation and the results from~\citet{LacosteJulien2015linearFW} where the global linear convergence of the (batch) pairwise FW (PFW) and away-step FW (AFW) algorithms was shown. The main insight to get our result is that the ``pairwise FW gap'' decomposes also as a sum of block gaps.
We give our result for the following more general setting (the block-separable analog of the setup in Appendix~F of~\citet{LacosteJulien2015linearFW}):
\begin{equation}
\begin{aligned} \label{eq:blockPFWproblem}
\min_{\bx \in \domain} \, f(\bx) \quad 
& \text{with} \quad f(\bx) := q(A \bx) + \bv^\top \bx \\
& \text{and} \quad \domain=\domain^{(1)} \times\dots\times \domain^{(n)},
\end{aligned}
\end{equation}
where $q$ is a strongly convex function, and $\domain^{(i)} := \conv(\atoms^{(i)})$ for each $i$, where $\atoms^{(i)} \subseteq \R^{m_i}$ is a \emph{finite} set of vectors (called \emph{atoms}). In other words, each $\domain^{(i)}$ is a polytope. For the example of the dual SSVM objective~\eqref{eq:svmstruct_nslack_dual}, $q(\cdot) := \frac{\regularizerweight}{2} \| \cdot \|^2$ and $\atoms^{(i)}$ are the corners of a probability simplex in $m_i := |\outputdomain_i|$ dimensions.

Suppose that we maintain an active set $\activeS_i$ for each block (as in the BCPFW algorithm). We first relate the batch PFW direction with the block PFW directions, as well as their respective batch and blockwise \emph{PFW gaps} (the PFW gap is replacing the FW gap~\eqref{eq:duality_gap} in the analysis of PFW).

\begin{definition}[Block PFW gaps] \label{def:blockPFW}
Consider the problem~\eqref{eq:blockPFWproblem} and suppose that the point $\bx$ has each
of its block $\bx_{(i)}$ with current active set $\activeS_i \subseteq \atoms^{(i)}$.\footnote{That is, $\bx_{(i)}$ is a convex combination of \emph{all} the elements of $\activeS^{(i)}$ with non-zero coefficients.}
We define the corresponding \emph{batch PFW gap} at~$\bx$ with active set $\activeS := \activeS_1 \times \cdots \times \activeS_n$ as:
\begin{align}
g^{\PFW}(\bx ; \activeS) &:= \max_{\sv \in \domain, \vv \in \activeS}  \,\langle -\nabla f(\bx) \,  , \, \sv - \vv  \rangle \label{eq:batchPFW} \\
&= \max_{\sv \in \domain, \vv \in \activeS}  \sum_i \langle  -\nabla_{(i)} f(\bx) \, , \, \sv_{(i)} - \vv_{(i)} \rangle \notag \\
&= \sum_i  \underbrace{\max_{\substack{ \sv_{(i)} \in \domain^{(i)} \\ \vv_{(i)} \in \activeS_i } } \langle   -\nabla_{(i)} f(\bx) \, , \, \sv_{(i)} - \vv_{(i)}\rangle} \notag \\
&=: \qquad   \sum_i \qquad g_i^\PFW(\bx \, ; \, \activeS_i) , \label{eq:blockPFW} 
\end{align}
where $g_i^\PFW$ is the PFW gap for block~$i$. We recognize that the maximizing arguments for $g_i^\PFW$ are the FW corner $\sv_{(i)}$ and the away corner $\vv_{(i)}$ for block $i$ that one would obtain when running BCPFW on this block.
\end{definition}

We note that by maintaining independent active sets $\activeS_i$ for each block, the number of potential away corner combinations is exponential in the number of blocks, yielding many more possible directions of movement than in the batch PFW algorithm where the number of away corners is bounded by the number of iterations. Moreover, suppose that we have an explicit expansion for each block $\bx_{(i)}$ as a convex combination of atoms in the active set: $\bx_{(i)} = \sum_{\vv_{(i)} \in \activeS_i} \beta_i(\vv_{(i)}) \, \vv_{(i)}$, where $\beta_i(\vv_{(i)}) > 0$ is the convex combination coefficient associated with atom~$\vv_{(i)}$. Then we can also express $\bx$ as an explicit convex combination of the (exponential size) active set $\activeS$ as follows:
$\bx = \sum_{\vv \in \activeS} \beta(\vv) \, \vv$, where $\beta(\vv) := \prod_{i=1}^n \beta_i(\vv_{(i)})$.

We can now prove an analog of the expected block descent lemma (Lemma~\ref{lemma:convGapSampling:expectedDescent} for BCFW) in the case of BCPFW and BCAFW\@. %
For technical reasons, we need a slightly different block curvature constant~$\CfAi$ (cf.~Eq.~(26) in~\citet{LacosteJulien2015linearFW}).

\begin{lemma}[Expected BCPFW descent lemma] \label{lem:BCPFWdescent}
Consider running the BCPFW algorithm on problem~\eqref{eq:blockPFWproblem}. Let~$\bx^{(k)}$ be the current iterate, and suppose that $\activeS_j$ is the current active set for each block $\bx_{(j)}^{(k)}$. Let $\activeS^{(k)} := \activeS_1 \times \cdots \times \activeS_n$ be the current (implicit) active set for $\bx^{(k)}$. Suppose that \emph{there is no drop set at~$\bx^{(k)}$}, that is, that for each possible block~$i$ that could be picked at this stage, the PFW step with line-search on block~$i$ will not have its step size truncated (we say that the line-search \emph{will succeed}). Then, conditioned on the current state,
in expectation over the random choice of block~$i$ with uniform probability and for any $\stepsize \in [0,1]$, it holds for the next iterate 
$\bx^{(k+1)}$ of BCPFW:
    \begin{multline}   \label{eq:BCPFWlemma}
    \E\big[ f(\bx^{(k+1)})\,|\, \bx^{(k)}, \activeS^{(k)} \big]
    \le \\ 
    f(\bx^{(k)}) -  \frac{\stepsize}{n} g^\PFW(\bx^{(k)} \,;\, \activeS^{(k)})
      + \frac{\stepsize^2}{2n} \CfATotal,
    \end{multline}
where $\CfATotal := \sum_{i=1}^n \CfAi$ is the total (away) curvature constant, and where $\CfAi$ is defined as in Definition~\ref{def:curvature_const}, but allowing the reference point~$\dualvarv_{(i)}$ in~\eqref{eqn:CfProduct} to be any point $\vv_{(i)} \in \domain^{(i)}$ instead, thus allowing a pairwise FW direction $\sv_{[i]} - \vv_{[i]}$ to be used in its definition.

Moreover,~\eqref{eq:BCPFWlemma} also holds for BCAFW (again under the assumption of \emph{no drop step}), but with an extra $\nicefrac{1}{2}$ factor in front of $g^\PFW(\bx^{(k)} \,;\, \activeS^{(k)})$ in the bound.
\end{lemma}
\begin{proof}
Let block~$i$ be the chosen one that defined $\bx^{(k+1)}$ and let $\bx_\stepsize := \bx^{(k)}+\stepsize(\sv_{[i]}-\vv_{[i]})$, where $\mathbf{s}_{(i)} \in \domain^{(i)}$ is the FW corner on block~$i$ with~$\mathbf{s}_{[i]} \in \R^m$ its zero-padded version, and similarly $\vv_{(i)} \in \activeS_i$ is the chosen away corner on block~$i$. By assumption, we have that the line-search succeeds, i.e., the minimum of $\min_{\stepsize \in [0,\stepmax]} f(\bx_\stepsize)$ is achieved for $\stepsize^* < \stepmax$, where $\stepmax$ is the maximum step size for this block for the PFW direction (this is because the optimal step size for the line-search cannot be truncated at $\stepmax$, as otherwise it would be a drop step). As $f$ is a convex function, this means that $f(\bx^{(k+1)}) = \min_{\stepsize \in [0,\stepmax]} f(\bx_\stepsize) = \min_{\stepsize \geq 0}  f(\bx_\stepsize)$ (removing inactive constraints does not change its minimum). By definition of~$\CfAi$, we thus have for any $\stepsize \in [0,1]$:
\begin{align}
f({\scriptstyle \bx^{(k+1)}}) &\leq  f(\bx_\stepsize) \notag \\
&= f({\scriptstyle \bx^{(k)}} + \gamma (\sv_{[i]} - \vv_{[i]})) \notag \\
& \leq f({\scriptstyle \bx^{(k)}})\! +\! \gamma \langle  \nabla_{(i)} f(\bx) , \, \sv_{(i)} \!\! -\! \vv_{(i)}\rangle \!+\! \frac{\gamma^2}{2}\CfAi \notag \\
&=  f({\scriptstyle \bx^{(k)}}) - \gamma \, g_i^\PFW({\scriptstyle \bx^{(k)}} ;  \activeS_i) + \frac{\gamma^2}{2}\CfAi \!\!\!\!\!\!. \label{eq:blockDescentPFW}
\end{align}
Taking the expectation of the bound with respect to~$i$, conditioned on~$\bx^{(k)}$ and~$\activeS^{(k)}$, yields~\eqref{eq:BCPFWlemma} by using the block-decomposition relationship~\eqref{eq:blockPFW} in the definition of~$g^\PFW(\bx^{(k)} \,;\, \activeS^{(k)})$. This completes the proof for BCPFW.

In the case of BCAFW, let $\dd_i$ be the chosen direction for block~$i$ (either a FW direction or an away direction). Then since $\dd_i$ is chosen to maximize the inner product with $-\nabla_{(i)} f(\bx^{(k)})$, we have
$\langle -\nabla_{(i)} f(\bx^{(k)}), \dd_i \rangle \geq \frac{1}{2} g_i^\PFW(\bx^{(k)} \, ; \, \activeS_i)$ (with a similar argument as used to get Eq.~(6) in~\citet{LacosteJulien2015linearFW} for~AFW). We then follow the same argument to derive~\eqref{eq:blockDescentPFW}, but using $\dd_i$ instead of $(\sv_{(i)} \!\! -\! \vv_{(i)})$, which gives an extra $\nicefrac{1}{2}$ factor as $\langle -\nabla_{(i)} f(\bx^{(k)}), \dd_i \rangle$ is potentially only half of $g_i^\PFW(\bx^{(k)} \, ; \, \activeS_i)$. Taking again the expectation of~\eqref{eq:blockDescentPFW} completes the proof.
\end{proof}

\begin{remark}
The important condition that there is \emph{no drop step} at $\bx^{(k)}$ in the BCPFW descent lemma~\ref{lem:BCPFWdescent} is to allow the bound~\eqref{eq:blockDescentPFW} to hold for \emph{any} $\stepsize \in [0,1]$. Otherwise, let $I$ be the (non-empty) set of blocks for which there would be a drop step at $\bx^{(k)}$ and let $\stepsize_I := \min_{i \in I}  \stepmax^{(i)}$, where $\stepmax^{(i)}$ is the maximum step size for block~$i$. Then in this case we could only show the bound~\eqref{eq:blockDescentPFW} for $\stepsize \leq \stepsize_I$. But $\stepsize_I$ could be arbitrarily small,\footnote{Small maximum step sizes happen when the current coordinate value for an away corner is small (perhaps because a small step size was used by the line-search when they were added as a FW corner previously).} and so no significant progress is guaranteed in expectation in this case. 

We also note that $\CfAi$ is used instead of $\Cf^{(i)}$ in the lemma because $\Cf^{(i)}$ can only be used with a feasible step from $\bx^{(k)}$, and thus again, the bound would only be valid for $\stepsize \leq \stepmax$ (as bigger step sizes can take you outside of $\domain^{(i)}$). If the gradient of $f$ is Lipschitz continuous, one can bound $\CfAi \leq \tilde{L}_i \big(\diam_{\|\cdot\|_i} \domain^{(i)} \big)^2$, which is almost the same bound as for~$\Cf^{(i)}$ explained in footnote~\ref{foot:Cfi}, but with $\tilde{L}_i$ being the Lipschitz constant of $\nabla_{(i)} f$ for variations in the slightly extended domain $\domain^{(i)} + (\domain^{(i)} - \domain^{(i)})$ (with set addition in the Minkowski sense).
\end{remark}

\begin{theorem}[Geometric convergence of BCPFW] \label{thm:BCPFW}
Consider running BCPFW (or BCAFW) on problem~\eqref{eq:blockPFWproblem} where $q$ is a strongly convex function and $\domain$ is a block-separable polytope. Let $h_k := f(\bx^{(k)}) - f(\bx^*)$ be the suboptimality of the iterate~$k$, where $\bx^*$ is any optimal solution to~\eqref{eq:blockPFWproblem}. Conditioned on any iterate $\bx^{(k)}$ with active set $\activeS^{(k)}$ such that \emph{no block could give a drop set} (as defined in the conditions for Lemma~\ref{lem:BCPFWdescent}), then the expected new suboptimality decreases geometrically, that is:
\begin{equation} \label{eq:BCPFWgeometric}
\E\big[ h_{k+1} \,|\, \bx^{(k)}, \activeS^{(k)} \big] \leq (1-\rho) h_k , 
\end{equation}
with rate:
\begin{align}
\rho &:= \frac{1}{2n} \min \{1, 2 \frac{\strongConvGeneralized}{\CfATotal} \} \,\, \text{for the BCPFW algorithm}, \label{eq:BCPFWrate} \\ 
\rho &:= \frac{1}{4n} \min \{1, \frac{\strongConvGeneralized}{\CfATotal} \} \,\,\,\, \text{for the BCAFW algorithm}, \label{eq:BCAFWrate} 
\end{align}
where $\CfATotal := \sum_{i=1}^n \CfAi$ is the total (away) curvature constant for problem~\eqref{eq:blockPFWproblem} as defined in Lemma~\ref{lem:BCPFWdescent}, and $\strongConvGeneralized$ is the generalized strong convexity constant for problem~\eqref{eq:blockPFWproblem} as defined in~Eq.~(39) of~\citet{LacosteJulien2015linearFW} ($\strongConvGeneralized$ is strictly greater than zero when $q$ is strongly convex and $\domain$ is a polytope).
\end{theorem}
\begin{proof}
We first do the argument for BCPFW. Let $g_k := g^\PFW(\bx^{(k)} \,;\, \activeS^{(k)})$, and notice that $g_k \geq h_k$ always. Because we assume that there is no drop step at $\bx^{(k)}$, we can use the expected BCPFW descent lemma~\ref{lem:BCPFWdescent}. By subtracting $f(\bx^*)$ on both side of the descent inequality~\eqref{eq:BCPFWlemma}, we get (for any $\stepsize \in [0,1]$):
\begin{equation} \label{eq:hkNextBCPFW}
\E\big[h_{k+1} \,|\, \bx^{(k)}, \activeS^{(k)} \big]
\le h_k -  \frac{\stepsize}{n} g_k
  + \frac{\stepsize^2}{2n} \CfATotal . 
\end{equation}
We can minimize the RHS of~\eqref{eq:hkNextBCPFW} with $\stepsize^* = \frac{g_k}{\CfATotal}$. If $g_k > \CfATotal$ (i.e. $\stepsize^* > 1$), then use $\stepsize = 1$ in~\eqref{eq:hkNextBCPFW} to get:
\begin{equation} \label{eq:rhoBigBCPFW}
\E\big[h_{k+1} \,|\, \bx^{(k)}, \activeS^{(k)} \big]
\le h_k -  \frac{1}{2n} g_k \leq (1-\frac{1}{2n}) h_k .
\end{equation}
This gives a geometric rate of $\rho = \frac{1}{2n}$. So now suppose that $g_k \leq \CfATotal$ (so that $\stepsize^* \leq 1$); putting $\stepsize = \stepsize^*$ in~\eqref{eq:hkNextBCPFW}, we get:
\begin{equation} \label{eq:hkBCPFWprogress}
\E\big[h_{k+1} \,|\, \bx^{(k)}, \activeS^{(k)} \big]
\le h_k -  \frac{1}{2n \CfATotal} \, {g_k}^2 .
\end{equation}
We now use the key relationship between the suboptimality~$h_k$ and the PFW gap~$g_k$ derived in inequality~(43) of~\citet{LacosteJulien2015linearFW} (which is true for any function~$f$ by definition of~$\strongConvGeneralized$ if we allow it to be zero):
\begin{equation} \label{eq:hkPFWgap}
h_k \leq \frac{{g_k}^2}{2 \strongConvGeneralized} .\textsl{}
\end{equation}
Substituting~\eqref{eq:hkPFWgap} into~\eqref{eq:hkBCPFWprogress}, we get:
\begin{equation} \label{eq:rhoBCPFW}
\E\big[h_{k+1} \,|\, \bx^{(k)}, \activeS^{(k)} \big] \leq (1-\frac{\strongConvGeneralized}{n  \CfATotal}) h_k ,
\end{equation}
which gives the $\rho = \nicefrac{\strongConvGeneralized}{n  \CfATotal}$ rate. Taking the worst rate of~\eqref{eq:rhoBigBCPFW} and~\eqref{eq:rhoBCPFW} gives the rate~\eqref{eq:BCPFWrate}, completing 
the proof for BCPFW.

In the case of BCAFW, Lemma~\ref{lem:BCPFWdescent} yields the inequality~\eqref{eq:hkNextBCPFW} but with an extra $\nicefrac{1}{2}$ factor in front of $g_k$. Re-using the same argument as above, we get a rate of $\rho = \nicefrac{1}{4n}$ when $\stepsize^* > 1$, and $\rho = \nicefrac{\strongConvGeneralized}{4n  \CfATotal}$ when $\stepsize^* \leq 1$, showing~\eqref{eq:BCAFWrate} as required.

Finally, the fact that $\strongConvGeneralized > 0$ when $q$ is $\mu$-strongly convex and $\domain$ is a polytope comes from the lower bound given in Theorem~10 of~\citet{LacosteJulien2015linearFW} in terms of the \emph{pyramidal width} of~$\domain$ (a strictly positive geometric quantity for polytopes), and the \emph{generalized strong convexity} of $f$ as defined in Lemma~9 of~\citet{LacosteJulien2015linearFW}. The generalized strong convexity of $f$ is simply $\mu$ if $f$ is $\mu$-strongly convex. In the more general case of problem~\eqref{eq:blockPFWproblem} where only~$q$ is $\mu$-strongly convex, the generalized strong convexity depends both on~$\mu$ and the \emph{Hoffman constant}~\citepsup{hoffman1952constant} associated with the linear system of problem~\eqref{eq:blockPFWproblem}. See~\citet{LacosteJulien2015linearFW} for more details, as well as Lemma~2.2 of~\citetsup{beck2015AFW}.
\end{proof}

\paragraph{Interpretation.} Theorem~\ref{thm:BCPFW} only guarantees progress of BCPFW or BCAFW when there would not be any drop step for \emph{any} block~$i$ for the current iterate. For the batch AFW algorithm, one can easily lower bound the number of times that these ``good steps'' can happen as a drop step reduces the size of the active set and thus cannot happen more than half of the time. On the other hand, in the block coordinate setting, we can be unlucky and always have one block that could give a drop step (while we pick other blocks during the algorithm, this bad block affects the expectation). This means that without a refined analysis of the drop step possibility, we cannot guarantee any progress in the worst case for BCPFW or BCAFW. As a safeguard, one can modify BCPFW or BCAFW so that it also has the option to do a standard BCFW step on a block if it yields better progress on $f$ -- this way, the algorithm inherits at least the (sublinear) convergence guarantees of BCFW.

\paragraph{Empirical linear convergence.} In our experiments, we note that BCPFW always converged empirically, and had an empirical linear convergence rate for the SSVM objective when $\lambda$ was big enough ($q(\cdot) = \frac{\lambda}{2} \| \cdot \|^2$ for the SSVM objective~\eqref{eq:svmstruct_nslack_dual}). See Figure~\ref{fig:app:exp1_dataset1} for OCR-large~(c) for example. We also tried the modified BCPFW algorithm where a choice is made between a FW step, a pairwise FW step or an away step on a block by picking the one which gives the biggest progress. We did not notice any significant speed-up for this modified method.

\paragraph{On the dimension of SSVM.} Finally, we note that the rate constant~$\rho$ in Theorem~\ref{thm:BCPFW} has an implicit dependence on the dimensionality (in particular, through the pyramidal width of~$\domain$). \citet{LacosteJulien2015linearFW} showed that the largest possible pyramidal width of a polytope in dimension~$m$ (for a fixed diameter) is achieved by the probability simplex and is $\Theta(1/\sqrt{m})$. For the SSVM in the general form~\eqref{eq:svmstruct_nslack_dual}, the dimensionality of~$\domain^{(i)}$ is the number of possible structured outputs for input~$i$, which is typically an exponentially large number, and thus the pyramidal width lower bound would be useless in this case. Fortunately, the matrix~$A$ (feature map) and vector~$\bv$ (loss function) are highly structured, and thus many~$\dualvarv$'s are mapped to the same objective value. For a feature mapping~$\featuremapdiffv_i(\outputvarv)$ representing the sufficient statistics for an energy function associated with a graphical model (as for a conditional random field~\citepsup{laffery2001CRF}), then the SSVM objective is implicitly optimizing over the \emph{marginal polytope} for the graphical model~\citepsup{wainwright2008graphical}. More specifically, let $A_i$ be the $d \times m_i$ submatrix of~$A$ associated with example~$i$. Then we can write $A_i = B_i M_i$ where $M_i$ is a $p \times m_i$ \emph{marginalization} matrix, that is, $\bm{\mu} = M_i \dualvarv$ is an element of the marginal polytope for the graphical model, where $p$ is the dimensionality of the marginal polytope -- which is a polynomial number in the size of the graph, rather than exponential. By the affine invariance property of the FW-type algorithms, we can thus instead use the pyramidal width of the marginal polytope for the convergence analysis (and similarly for the Hoffman constant). \citet{LacosteJulien2015linearFW} conjectured that the pyramidal width of a marginal polytope in dimension $p$ was also~$\Theta(1/\sqrt{p})$, thus giving a more reasonable bound for the convergence rate of BCPFW for SSVM.

\section{BCFW for SSVM with box constraints}
\label{app:pos_const}

\subsection{Problem with box constraints}
Problem~\eqref{eq:svmstruct_nslack_primal_nonsmooth} can be equivalently rewritten as a quadratic program (QP) with an exponential number of constraints:
\begin{align}
    \label{eq:svmstruct_nslack_primal}
    \min_{\weightv,\, \bm{\xi}} \quad &   \frac{\regularizerweight}{2}\norm{\weightv}^2 +
    \frac1n \sum_{i=1}^n \xi_i\\
    \text{s.t.} \quad &
    \langle \weightv, \featuremapdiffv_i(\outputvarv) \rangle
    \geq \errorterm(\outputvarv_i,\outputvarv) - \xi_i \quad  \forall i, \, \forall \outputvarv \in \outputdomain_i
      \notag
\end{align}
where the slack variable~$\xi_i$ measures the surrogate loss for the $i$-th datapoint.
Problem~\eqref{eq:svmstruct_nslack_primal} is often referred to as the $n$-slack structured SVM with margin-rescaling~\citep[Optimization Problem 2]{Joachims:2009ex}.

In this section, we consider the problem~(\ref{eq:svmstruct_nslack_primal}) with additional box constraints on the parameter vector~$\weightv$:
\begin{align}
    \label{eq:ssvn_nslack_primal_bounded}
    \min_{\weightv,\, \bm{\xi}} \quad &   \frac{\regularizerweight}{2}\norm{\weightv}^2 +
    \frac1n \sum_{i=1}^n \xi_i\\
    \text{s.t.} \quad &
    \langle \weightv, \featuremapdiffv_i(\outputvarv) \rangle
    \geq \errorterm(\outputvarv_i,\outputvarv) - \xi_i \quad  \forall i, \, \forall \outputvarv \in  \outputdomain_i, \notag \\
    \quad & \lb \preccurlyeq \weightv \preccurlyeq \ub
      \notag,
\end{align}

where $\lb \in \R^d$ and $\ub \in \R^d$ denote the lower and upper bounds, respectively, and the symbol ``$\preccurlyeq$'' is the element-wise ``less or equal to'' sign.
In the following, we assume that the box constraints are feasible, i.e., $\lb \preccurlyeq \ub$.
Note that the following discussion can be directly extended to the case where only some dimension of the weight vector have to respect the box constraints.
The Lagrangian of problem~\eqref{eq:ssvn_nslack_primal_bounded} can be written as
\begin{multline}
 \label{eq:ssvn_nslack_lagrang_bounded}
L(\weightv,\bm{\xi},\dualvarv, \duallv, \dualuv)
  = \frac\regularizerweight2 \langle\weightv,\weightv\rangle+\frac1n\sum_{i=1}^n \xi_i   \\
  + \sum_{i\in[n],\,\outputvarv \in \outputdomain_i}
  \frac1n\dualvar_i(\outputvarv)
  \left( -\xi_i +
     \langle \weightv, -\featuremapdiffv_i(\outputvarv) \rangle
     + \errorterm_i(\outputvarv)
  \right)  \\
  +\regularizerweight \langle \dualuv, \weightv - \ub \rangle + \regularizerweight \langle \duallv, -\weightv + \lb \rangle
\end{multline}
where $\duallv \in \R^d$ and  $\dualuv \in \R^d$ are the dual variables associated with the lower and upper bound constraints, respectively.
From the KKT conditions, we obtain
\begin{align}
\label{eq:kkt_box_w}
\weightv =& \ A\dualvarv-(\dualuv-\duallv), \\
\label{eq:kkt_box_alp}
\sum_{\outputvarv \in \outputdomain_i}
  \dualvar_i(\outputvarv) = & \ 1 ~~~\forall i\in[n].
\end{align}
Finally, the dual of problem~\eqref{eq:ssvn_nslack_primal_bounded} (here written in a minimization form) can be written as follows:
\begin{align}
\label{eq:ssvn_nslack_dual_bounded}
    \min_{\substack{ \dualvarv\in\R^{m} \\  \dualvarv \succcurlyeq 0}}  f(\dualvarv,\duallv,\dualuv) :=&  \;
    \frac{\regularizerweight}{2}
    \big\|A \dualvarv - (\dualuv - \duallv)\big\|^2 - \bv^\transpose \dualvarv \notag \\[-3mm]
    & + \regularizerweight ( \dualuv^{\transpose} \ub - \duallv^{\transpose} \lb )
    \notag \\[1mm]
    \text{s.t. } & \sum_{\outputvarv \in \outputdomain}  \dualvar_i(\outputvarv) = 1 ~~~ \forall i\in[n], \notag \\
      \text{and } & \dualuv \succcurlyeq 0 , \duallv \succcurlyeq 0.
\end{align}

\paragraph{A modified block optimization method.}
Ideally, we should optimize $f(\dualvarv,\duallv,\dualuv)$ jointly w.r.t.\ all the dual variables.
This task is not directly suitable for the Frank-Wolfe approach as the domain for $\duallv$ and $\dualuv$ is unbounded.
However, joint optimization w.r.t.~$\duallv$ and $\dualuv$ with $\dualvarv$ kept fixed can be done in closed form.
After that, optimization w.r.t.~$\dualvarv$ can be performed using the Frank-Wolfe blockwise approach.
Therefore, we resort to optimizing in a blockwise fashion: we iterate either a batch FW or a BCFW step on $\dualvarv$ with an exact block-update on $(\dualuv,\duallv)$. As we will see below, this principled approach is similar to a commonly used heuristic of truncating the value of $\weightv$ to make it feasible during an algorithm which works on the dual. In fact, our approach will be equivalent to run FW or BCFW with a truncation making $\weightv(\dualvarv)$ feasible after each FW step, but with a \emph{change in the optimal step-size computation} (line~\ref{alg:lineStepsizeFWpositive} in Algorithm~\ref{alg:FW_SVM_box} for FW; line~\ref{alg:lineStepsizeBCFWpositive} in Algorithm~\ref{alg:FW_product_SVM_box} for BCFW) due to the different nature of the optimization problem.

\begin{algorithm}[t!]
    \caption{Batch Frank-Wolfe algorithm for structured SVM with box constraints} %
    \label{alg:FW_SVM_box}
    \begin{algorithmic}[1]
        \STATE Let $\notruncv^{(0)}:= \0$; $\ell^{(0)}:=0$
        \STATE $\weightv^{(0)}:= [ \notruncv^{(0)} ]_{\lb}^{\ub}$
        \COMMENT{truncation to the feasible set}
        \FOR{$k:=0,\dots,\infty$}
        \FOR{$i:=1,\dots,n$}
        \STATE Solve $\outputvarv_i^* := \displaystyle\argmax_{\outputvarv\in\outputdomain_i} \ H_i(\outputvarv;\weightv^{(k)})$ \vspace{-2mm}%
        \ENDFOR %
        \STATE Let $\notruncv_{\sv} := {\displaystyle\sum_{i=1}^n} \frac1{\regularizerweight n} \featuremapdiffv_i(\outputvarv_i^*)$ \;
        and \; $\ell_{\sv} := \frac1n \displaystyle\sum_{i=1}^n \errorterm_i(\outputvarv_i^*)$
        \STATE Let $\stepsize := \frac{\regularizerweight (\notruncv^{(k)}-\notruncv_{\sv})^\transpose \weightv^{(k)} - \ell^{(k)} + \ell_{\sv} }{ \regularizerweight \|\notruncv^{(k)}-\notruncv_{\sv}\|^2 }$ ~and clip  to $[0,1]$ \label{alg:lineStepsizeFWpositive}
        \STATE Update $\notruncv^{(k+1)}:= (1-\stepsize)\notruncv^{(k)}+\stepsize\, \notruncv_{\sv}$
        \STATE {\small~~~~~~~and~ $\ell^{(k+1)}:= (1-\stepsize)\ell^{(k)}+\stepsize\, \ell_{\sv}$}
        \STATE {\small~~~~~~~and~ $\weightv^{(k+1)}:= [ \notruncv^{(k+1)} ]_{\lb}^{\ub}$}
        \ENDFOR
    \end{algorithmic}
\end{algorithm}

\begin{algorithm}[t]
    \caption{Block-coordinate Frank-Wolfe algorithm for structured SVM with box constraints}%
    \label{alg:FW_product_SVM_box}
    \begin{algorithmic}[1]
        \STATE Let $\notruncv^{(0)}:= {\notruncv_i}^{(0)} := \0$; \; $\ell^{(0)}:={\ell_i}^{(0)}:=0$;
        \STATE $\weightv^{(0)}:= [ \notruncv^{(0)} ]_{\lb}^{\ub}$
        \COMMENT{truncation to the feasible set}
        \FOR{$k:=0,\dots,\infty$}
        \STATE Pick $i$ at random in $\{1,\ldots,n\}$
        \STATE Solve $\outputvarv_i^* := \displaystyle\argmax_{\outputvarv\in\outputdomain_i} \ H_i(\outputvarv;\weightv^{(k)})$ %
        \STATE Let $\notruncv_{\sv} := \frac1{\regularizerweight n} \featuremapdiffv_i(\outputvarv_i^*)$
        \; and \;  $\ell_{\sv} := \frac1n \errorterm_i(\outputvarv_i^*)$
        \STATE Let $\stepsize := \frac{ \regularizerweight (\notruncv_i^{(k)}-\notruncv_{\sv})^\transpose \weightv^{(k)} - \ell_i^{(k)} + \ell_{\sv} }{ \regularizerweight \|\notruncv_i^{(k)}-\notruncv_{\sv}\|^2}$~and clip to $[0,1]$ \label{alg:lineStepsizeBCFWpositive}
        \STATE Update ${\notruncv_i}^{(k+1)}:= (1-\stepsize){\notruncv_i}^{(k)}+\stepsize \,\notruncv_{\sv}$
        \STATE {\small~~~~~~~and~ ${\ell_i}^{(k+1)}:= (1-\stepsize){\ell_i}^{(k)}+\stepsize\, \ell_{\sv}$}
        \STATE Update $\notruncv^{(k+1)}\;:= \notruncv^{(k)} + {\notruncv_i}^{(k+1)} - {\notruncv_i}^{(k)}$
        \STATE {\small~~~~~~~and~~ $\ell^{(k+1)}:= ~\ell^{(k)}+{\ell_i}^{(k+1)} \ \  - {\ell_i}^{(k)}$}
        \STATE Let $\weightv^{(k+1)}:= [ \notruncv^{(k+1)} ]_{\lb}^{\ub}$
        \ENDFOR
    \end{algorithmic}
\end{algorithm}

\subsection{Optimizing w.r.t \texorpdfstring{$\bm{\beta}_u$}{betaU} and \texorpdfstring{$\duallv$}{betaL} while fixing \texorpdfstring{$\dualvarv$}{alpha}}

The optimization w.r.t.~$\dualuv$ with $\dualvarv$ and $\duallv$ fixed can be easily solved in closed form via a simple thresholding operation:
\begin{equation}
\label{eq:optimalBetaU}
\bm{\beta}_u^* = [A\dualvarv + \duallv - \ub ]_+ \ .
\end{equation}
The optimization w.r.t.~$\duallv$ with $\dualvarv$ and $\dualuv$ fixed is analogous:
\begin{equation}
\label{eq:optimalBetaL}
\duallv^* = [ - A\dualvarv + \dualuv + \lb ]_+ \ .
\end{equation}

Denote the $p$-th variable of $\dualuv$ and $\duallv$ with $\dualuv(p)$ and $\duallv(p)$, respectively.
For any index $p$, both $\dualuv(p)$ and $\duallv(p)$ cannot be nonzero simultaneously, because if one of the constraints is violated (either the upper or the lower bound), then the other constraint must be satisfied.
Hence, $\dualuv(p) \neq 0$ implies $\duallv(p) = 0$ and vice versa. %
Therefore, the final update equations can equivalently be written as
\begin{align}
\label{eq:optimalBetaULfinal}
\dualuv^*(\dualvarv) &= [A\dualvarv - \ub ]_+ \ , \\
\duallv^*(\dualvarv) &= [-A\dualvarv + \lb ]_+ \ .
\end{align}

Introducing $\notruncv(\dualvarv):=A\dualvarv$, we get $\bm{\beta}_u^* = [\notruncv - \ub ]_+$ and $\bm{\beta}^*_l = [-\notruncv + \lb ]_+$.
Hence, the operation $\weightv = \notruncv - (\bm{\beta}^*_u - \bm{\beta}^*_l)$ is simply the projection (truncation) of $\notruncv$ on the feasible set defined by the upper and lower bounds.
In the final algorithm, we maintain $\notruncv(\dualvarv)$ and directly update the primal variables~$\weightv$ without updating $\dualuv$ and $\duallv$.

\subsection{Batch setting: optimizing w.r.t \texorpdfstring{$\dualvarv$}{alpha} with  \texorpdfstring{$\bm{\beta}_u$}{betaU} and \texorpdfstring{$\duallv$}{betaL} fixed}
When the variables $\dualuv$ and $\duallv$ are fixed, the convex problem~\eqref{eq:ssvn_nslack_dual_bounded}
has a compact domain and so we can use the Frank-Wolfe algorithm on it.
In the following, we highlight the differences with the setting without box constraints.
We denote by $\weightv(\dualvarv)$ the truncation of $\notruncv(\dualvarv)$ on the box constraints, i.e.,
\begin{equation} \label{eq:wTruncation}
\weightv(\dualvarv) := \notruncv(\dualvarv) - \left(\bm{\beta}^*_u(\dualvarv) - \bm{\beta}^*_l (\dualvarv)\right).
\end{equation}
The derivations below assume that $\dualuv$ and $\duallv$ are fixed to their optimal values $\bm{\beta}^*_u(\dualvarv)$ and $\bm{\beta}^*_l (\dualvarv)$ for a specific $\dualvarv$.

\paragraph{Linear subproblem.}
The Frank-Wolfe linear subproblem can be written as
\begin{align}
\label{eq:linearizationProblemFWbox}
\bm{s} = \argmin_{\bm{s}' \in \mathcal{M}} \langle \bm{s}', \nabla_{\dualvarv} f(\dualvarv, \bm{\beta}_l, \bm{\beta}_u) \rangle
\end{align}
where $\nabla_{\dualvarv} f(\dualvarv, \bm{\beta}_l, \bm{\beta}_u)$ can be easily computed:
\begin{align}
\label{eq:gradient_ssvmdual_bounded}
\nabla_{\dualvarv} f(\dualvarv, \bm{\beta}_l, \bm{\beta}_u) &= \regularizerweight A^{\transpose} (A\dualvarv - (\bm{\beta}_u - \bm{\beta}_l)) - \bm{b} \notag \\
&= \regularizerweight A^{\transpose} \weightv - \bm{b}.
\end{align}
Analogously to the problem without box constraints, the linear subproblem used by the Frank-Wolfe algorithm is equivalent to the loss-augmented decoding subproblem~\eqref{eq:subproblem_loss_augm}.
The update of $\dualvarv$ can be made using the corner $\bm{s}$.
In what follows, we show that this update can be performed without explicitly keeping the dual variables at the cost of storing the extra vector~$\notruncv$.

\paragraph{The duality gap.} The Frank-Wolfe gap for problem~\eqref{eq:ssvn_nslack_dual_bounded} can be written as
\[
\begin{array}{rl}
\label{eq:FWdualgap_box}
  g(\dualvarv) :=& \displaystyle\max_{\sv' \in \domain} \,\langle \dualvarv - \sv',\nabla_{\dualvarv} f(\dualvarv, \bm{\beta}_l, \bm{\beta}_u) \rangle \\
    =& (\dualvarv - \sv)^\transpose (\regularizerweight A^{\transpose} (A\dualvarv - (\bm{\beta}_u - \bm{\beta}_l)) - \bv) \\
    =& \regularizerweight (\notruncv-\notruncv_{\sv})^\transpose \weightv - \bv^\transpose \dualvarv + \bv^\transpose \sv
   \vspace{-1mm}
\end{array}
\]
where $\notruncv_{\sv}:=A\sv$.
Below, we prove that the Frank-Wolfe duality gap~$g(\dualvarv)$ for the problem~\eqref{eq:ssvn_nslack_dual_bounded} when $\dualuv$ and $\duallv$ are fixed at their current optimal value for the current $\dualvarv$ equals to a Lagrange duality gap, analogously to the case without box constraints~\citep[Appendix~B.2]{lacosteJulien13bcfw}.\footnote{We stress that this relationship is only valid for the pair $\weightv = \weightv(\dualvarv)$ in the primal, and $\dualuv = \dualuv^*(\dualvarv), \duallv = \duallv^*(\dualvarv)$ in the dual.}

\begin{proof}
Consider the difference between the primal objective of~\eqref{eq:ssvn_nslack_primal_bounded} at $\weightv := A\dualvarv-(\dualuv-\duallv)$  with the optimal slack variables~$\bm{\xi}$ and the dual objective of~\eqref{eq:ssvn_nslack_dual_bounded} at $\dualvarv$ (in the maximization form).
We get
\begin{align*}
g_{\text{\tiny Lag.}}(\weightv,\dualvarv)&=\frac\regularizerweight2 \weightv^\transpose \weightv + \frac1n \sum_{i=1}^n \tilde{H}_i(\weightv) \\
& \qquad - \left(\bv^\transpose \dualvarv - \frac\regularizerweight2 \weightv^\transpose \weightv - \regularizerweight ( \dualuv^{\transpose} \ub - \duallv^{\transpose} \lb ) \right) \\
&= \regularizerweight \weightv^\transpose \weightv - \bv^\transpose \dualvarv + \frac1n \sum_{i=1}^n \max_{\outputvarv\in\outputdomain_i} H_i(\outputvarv;\weightv) \\
&\qquad + \regularizerweight ( \dualuv^{\transpose} \ub - \duallv^{\transpose} \lb ) \ .
\end{align*}

Recalling
\begin{align*}
\frac1n \sum_{i=1}^n \max_{\outputvarv\in\outputdomain_i} H_i(\outputvarv;\weightv) &=  \max_{\sv'\in \domain} -\sv'^\transpose \nabla_{\dualvarv} f(\dualvarv, \bm{\beta}_l, \bm{\beta}_u) \\ &= -\sv^\transpose \nabla_{\dualvarv} f(\dualvarv, \bm{\beta}_l, \bm{\beta}_u),
\end{align*}
we can write
\begin{align*}
g_{\text{\tiny Lag.}}(\weightv,\dualvarv)  &=  \regularizerweight \weightv^\transpose (A\dualvarv-(\dualuv-\duallv)) - \bv^\transpose \dualvarv  \\
&\qquad - \sv^{\transpose}\nabla_{\dualvarv} f(\dualvarv, \bm{\beta}_l, \bm{\beta}_u) + \regularizerweight ( \dualuv^{\transpose} \ub - \duallv^{\transpose} \lb ) \\
&= (\regularizerweight \weightv^\transpose  A - \bv^\transpose ) \dualvarv - \sv^\transpose\nabla_{\dualvarv} f(\dualvarv, \bm{\beta}_l, \bm{\beta}_u)  \\
&\qquad -\regularizerweight ( \weightv^\transpose  \dualuv - \weightv^\transpose \duallv)
+ \regularizerweight ( \dualuv^{\transpose} \ub - \duallv^{\transpose} \lb )  \\
	&= (\dualvarv - \sv)^\transpose \nabla_{\dualvarv} f(\dualvarv, \bm{\beta}_l, \bm{\beta}_u)
	\\&\qquad +  \regularizerweight(\dualuv^{\transpose}(\ub-\weightv)
	-\duallv^\transpose (\lb-\weightv)).
\end{align*}

As we assumed that $\dualuv = \dualuv^{*}(\dualvarv)$ and $\duallv = \duallv^{*}(\dualvarv)$, we have $\dualuv^{*\transpose}(\ub-\weightv)=0$ and $\duallv^{*\transpose}(\lb-\weightv)=0$, and thus
\begin{align*}
g_{\text{\tiny Lag}}(\weightv,\dualvarv,\dualuv^*,\duallv^*) = g(\dualvarv).
\end{align*}
\end{proof}

\paragraph{Line-Search.} Line search can be performed efficiently using 
$\stepsize_{\text{\tiny opt}s} := \frac{\langle \dualvarv - \sv, \nabla f(\dualvarv) \rangle}{\regularizerweight\norm{A(\dualvarv-\sv)}^2}
= \frac{g(\dualvarv)}{\regularizerweight\norm{\notruncv-\notruncv_{\sv}}^2}$.

\paragraph{Algorithm.}
Algorithm~\ref{alg:FW_SVM_box} contains the batch Frank-Wolfe algorithm with box constraints.
The main idea consists in maintaining the vector $\notruncv := A\dualvarv$ in order to perform all the updates of $\dualvarv$ using only the primal variables.
Optimization w.r.t.\ $\duallv$ and $\dualuv$ corresponds to the truncation of $\notruncv$.
Given these variables, the gap and the optimal step size are easy to compute.

\subsection{The block-coordinate setting}
Algorithm~\ref{alg:FW_product_SVM_box} describes the block-coordinate version of the Frank-Wolfe algorithm with box constraints.
The algorithm is obtained from the batch version (Algorithm~\ref{alg:FW_SVM_box}) in exactly the same way as BCFW (Algorithm~\ref{alg:FW_product_SVM}) is obtained from the batch Frank-Wolfe method~\citep[Algorithm~2]{lacosteJulien13bcfw}.

\newlength{\citationLengthTable}
\setlength{\citationLengthTable}{\widthof{Citation Citation More}} %

\begin{table*}[t]
    \centering
    \caption{ Statistics of the datasets used in the experimental evaluation. For the oracle time, we report the sum of the average running times for both the max oracle and the joint feature map computation (the starred numbers indicate that the input features were stored on disk instead of RAM, thus slowing down the computation). \vspace{3mm}\label{tab:dataset}}
    \resizebox{\textwidth}{!}{ %
        \begin{tabular}{@{}lllccrrccl@{}}
        \toprule
            Dataset                   & Task & Sructure (Oracle)                             & Citation          & Version              & $n$     & $d$                           & Sparsity & 
            \parbox{\widthof{constraints}}{\centering Box \\ constraints} & 
            \parbox{\widthof{Oracle time}}{\centering Oracle time \\ (in s)} \\ \midrule
            \multirow{2}{*}{OCR}      & \multirow{2}{*}{character recognition}  & \multirow{2}{*}{chain (Viterbi)}  & \multirow{2}{\citationLengthTable}{\centering \citep{Taskar2003}} & small                & $626$   & $4,082$                        & No       & No              & $5\times 10^{-4}$ 
            \\
            &                   &                   &                   & large                & $6,251$  & $4,082$                        & No       & No              & $5\times 10^{-4}$ 
            \\ \midrule
            CoNLL                     & text chunking      & \multirow{1}{*}{chain (Viterbi)}             & \parbox{\citationLengthTable}{\centering \citep{Sang2000}}                  & \multicolumn{1}{l}{} & $8,936$  & \multicolumn{1}{l}{$1,643,026$} & Yes      & No              &         $2\times 10^{-2}$                            \\   %
            \midrule
            \multirow{3}{*}{HorseSeg} & \multirow{3}{*}{binary segmentation} & \multirow{3}{*}{grid (graph cut)} & \multirow{3}{\citationLengthTable}{\centering \citep{kolesnikov2014closed}} & small                & $147$   & $1,969$                        & No       & Yes             & $1\times 10^{-3} $
            \\
            &                   &                   &                   & medium               & $6,121$  & $1,969$                        & No       & Yes             & $1\times 10^{-3} $
            \\
            &                   &                   &                   & large                & $25,438$ & $1,969$                        & No       & Yes             & $2 \times 10^{-2}$ (*)
            \\ \midrule
            LSP                       & pose estimation          & \multirow{1}{*}{tree (max sum)}            &   \parbox{\citationLengthTable}{\centering \citep{Johnson10}}                & small                & $100$   & $2,676$                        & No       & Yes             &  $2$ (*) \\ %
            \bottomrule
        \end{tabular}
    }
\end{table*}

\section{Dataset description \label{app:dataset_table}}
In our experiments, we use four structured prediction datasets: OCR~\cite{Taskar2003} for character recognition; CoNLL~\cite{Sang2000} for text chunking; HorseSeg~\cite{kolesnikov2014closed} for binary image segmentation; LSP~\cite{Johnson10} for pose estimation.
In this section, we provide the description of the datasets and the corresponding models.
Table~\ref{tab:dataset} summarizes quantitative statistics for all the datasets.
For the OCR and CoNLL datasets, the features and models described below are exactly the same as used by~\citet{lacosteJulien13bcfw}; we give a detailed description for reference. For HorseSeg and LSP, we had to build the models ourselves from previous work referenced in the relevant section.

\subsection{OCR}
The Optical Character Recognition (OCR) dataset collected by~\citet{Taskar2003} poses the task of recognizing English words from sequences of handwritten symbols represented by binary images.
The average length of sequences equals~$7.6$ symbols.
For a sequence of length~$T$,  the input feature representation~$\inputvarv$ consists of~$T$ binary images of size~$16 \times 8$.
The output object~$\outputvarv$ is a sequence of length~$T$ with each symbol taking 26 possible values.

The OCR dataset contains $6,877$ words.
In the small version, $626$ words are used for training and the rest for testing.
In the large version, $6,251$ words are used for training and the rest for testing.

The prediction model is a chain. The feature map~$\featuremapv(\inputvarv, \outputvarv)$ contains features of three types: emission, transition and bias.
The $16 \times 8 \times 26$ emission features count the number of times along the chain a specific position of the $16\times8$ binary image equals~$1$ when associated with a specific output symbol.
The $26 \times 26$ transition features count the number of times one symbol follows another.
The $26 \times 3$ bias features represent three biases for each element of the output alphabet: one model bias, and a bias for when the letter appears at the beginning or at the end of the sequence.

As the structured error between output vectors $\errorterm(\outputvarv_i,\outputvarv)$, the OCR dataset uses the Hamming distance normalized by the length of the sequences.
The loss-augmented structured score $H_i(\outputvarv;\weightv)$ is a function with unary and pairwise potentials defined on a chain and is exactly maximized with the dynamic programming algorithm of~\citet{Viterbi67}.

\subsection{CoNLL}
The CoNLL dataset released by~\citet{Sang2000} poses the task of text chunking.
Text chunking, also known as shallow parsing~\citesup{sha2003shallow}, consists in dividing the input text into syntactically related non-overlapping groups of words, called phrase or chunks.
The task of text chunking can be cast as a sequence labeling where a sequence of labels $\outputvarv$ is predicted from an input sequence of tokens $\inputvarv$.
For a given token $\inputvar_t$ (a word with its corresponding part-of-speech tag), the associated label $\outputvar_t$ gives the type of phrase the token belongs to, i.e., says whether or not it corresponds to the beginning of a chunk, or encodes the fact that the token does not belong to a chunk.

The CoNLL dataset contains $8,936$ training English sentences extracted from the Wall Street Journal part of the Penn Treebank II~\citesup{marcus1993pennTreebank}. Each output label $\outputvar_t$ can take up to $22$ different values.

We use the feature map~$\featuremapv(\inputvarv, \outputvarv)$ proposed by~\citetsup{sha2003shallow}.
First, for each position~$t$ of the input sequence~$\inputvarv$, we construct a unary feature representation, containing the local information.
We start with extracting several attributes representing the words and the part-of-speech tags at the positions neighboring to~$t$.\footnote{We extract the attributes with the CRFsuite library~\citesup{CRFsuite} and refer to its documentation for the exact list of attributes: \url{http://www.chokkan.org/software/crfsuite/tutorial.html}.}
Each attribute is encoded with an indicator vector of length equal to either the dictionary size or the number of part-of-speech tags.
We concatenate these vectors to get a unary feature representation, which is a sparse binary vector of dimensionality~$74,658$.
Note that these representation can be precomputed outside of the training process.

Given a labeling $\outputvarv$ and the unary representations, the feature map $\featuremapv$ is constructed by concatenating features of three types (as in the chain model for OCR): emission, transition and bias.
The $74,658 \times 22$ emission features count the number of times each coordinate of the unary representation of token~$x_t$ is nonzero and the corresponding output variable~$y_t$ is assigned a particular value.
The transition map of size $22 \times 22$ encodes the number of times one label follows another in the output~$\outputvarv$.
The $22 \times 3$ bias features encode biases for all the possible values of the output variables and, specifically, biases for the first and the last variables.

As in the OCR task, the structured error~$\errorterm(\outputvarv_i,\outputvarv)$ is the normalized Hamming distance and thus the max-oracle can be efficiently implemented using the dynamic programming algorithm of~\citet{Viterbi67}.

\subsection{HorseSeg}
The HorseSeg dataset\footnote{\url{https://pub.ist.ac.at/~akolesnikov/HDSeg/HDSeg.tar}} released by~\citet{kolesnikov2014closed} poses the task of object/background segmentation of images containing horses, i.e., assigning a label ``horse'' or ``background'' to each pixel of the image.
HorseSeg contains $25,438$ training images, $147$ of which are manually annotated, $5,974$~annotations are constructed from object bounding boxes by the automatic method of~\citetsup{guillaumin2014autoAnnotation}, while the remaining $19,317$~annotations were constructed by the same method but without any human supervision.
The test set of HorseSeg consists of $241$~images with manual annotations.
In our experiments, we use training sets of three different sizes: $147$ images for HorseSeg-small, $6,121$~images for HorseSeg-medium and $25,438$ for HorseSeg-large.

In addition to images and their pixel-level annotations, \citet{kolesnikov2014closed} released\footnote{\url{https://pub.ist.ac.at/~akolesnikov/HDSeg/data.tar}} oversegmentations (superpixels) of the images precomputed with the SLIC algorithm~\citesup{achanta2012slicSuperpixels} and the unary features of each superpixel computed similarly to the work of~\citetsup{lempitsky2011pylon}.
On average, each image contains $147$~superpixels. The $1,969$ unary features include $1$~constant feature, $512$-bin histograms of densely sampled visual SIFT words~\citesup{lowe2004sift}, $128$-bin histograms of RGB colors, $16$-bin histograms of locations (each pixel of a region of interest is matched to a cell of the $4 \times 4$ uniform grid).
The three aforementioned histograms are computed on the superpixels themselves, on the superpixels together with their neighboring superpixels, and on the second-order neighborhoods.

For each pair of adjacent superpixels, we construct $100$ pairwise features: a constant feature; and
quantities $\exp{(-\eta d_{pq})}$ where~$d_{pq}$ is a $\chi^2$-distance between $9$ pairs of the corresponding histograms of each type for neighbors $p$ and $q$, and $\eta$ is a parameter taking $11$~values from the set $2^{-5}, 2^{-4}, \dots, 2^5$.

The structured feature map is defined in such a way that the corresponding structured score function contains unary and pairwise Potts potentials\\[-4mm]
\begin{align*}
\langle \weightv, \featuremapv(\inputvarv_i,\outputvarv) \rangle
&=
\sum_{p \in \mathcal{V}_i} \langle \weightv_U, \inputvarv_{i, p}^U \rangle ([y_p = 1] - [y_p=0]) \\
&+
\sum_{\{p,q\} \in \mathcal{E}_i} \langle \weightv_P,\inputvarv_{i, pq}^P \rangle [y_p \neq y_p]
\end{align*}
where the vector $\inputvarv_i = ( (\inputvarv_{i, p}^U)_{p\in\mathcal{V}_i}, (\inputvarv_{i, pq}^P)_{\{p,q\} \in \mathcal{E}_i})$ denotes all the features of image~$i$, the vector~$\outputvarv = (\outputvar_p)_{p\in \mathcal{V}_i} \in \{0,1\}^{\mathcal{V}_i}$ is a feasible labeling, the set~$\mathcal{V}_i$ is the set of the superpixels of the image~$i$, and the set $\mathcal{E}_i$ represents the adjacency graph.

The structured error is measured with a Hamming loss with class-normalized penalties 
\begin{equation*}
\errorterm(\outputvarv_i,\outputvarv)
=
\sum_{p \in \mathcal{V}_i} \omega_{y_{i,p}} [y_{i,p} \neq y_{p}]
\end{equation*}
where $\outputvarv_i = (\outputvar_{i,p})_{p\in \mathcal{V}_i}$ is the labeling of superpixels closest to the ground-truth annotation and the weights~$\omega_{0}$ and $\omega_{1}$ are proportional to the ground-truth area of each class.

The loss-augmented score function
$$
\errorterm(\outputvarv_i,\outputvarv) - \langle \weightv, \featuremapv(\inputvarv_i,\outputvarv)\rangle
$$
is a discrete function defined w.r.t.\ a cyclic graph and can be maximized in polynomial time when it is supermodular. By construction, all our pairwise features are nonnegative, so we can ensure supermodularity by adding positivity constraints on the weights corresponding to the pairwise features~$\weightv_P \succcurlyeq 0$. 
The version of BCFW with positivity constraints is described in Appendix~\ref{app:pos_const}.

The discrete optimization problem arising in the max-oracle is solved by the min-cut/max-flow algorithm of~\citet{boykov2004graphcut}.
The running time or the max-flow is small compared to the operations with the features required to compute the potentials.

\subsection{LSP}
The Leeds Sports Pose (LSP) dataset introduced by~\cite{Johnson10} poses the tasks of full body pose estimation from still images containing sports activities.
Based on the input image with a centered prominent person, the task is to predict the locations of $14$ body-parts (joints), e.g., ``left knee'' or ``right ankle''.

We cast the task of pose estimation as a structured prediction problem and build our model based on the work of~\citet{Chen_NIPS14}, which is one of the state-of-the-art methods for pose estimation.
First, we construct an acyclic graph  where the nodes $p\in\mathcal{V}$ correspond to the different body-parts. The set of body parts is extended from the original $14$ parts of interest by the midway points to get the $26$ nodes of the graph. Second, the graph is converted into the directed one by utilizing the arcs of both orientations for each original edge. We denote the resulting graph by $\mathcal{G}=(\mathcal{V},\vec{\mathcal{E}})$.

In the model of~\citet{Chen_NIPS14}, each node~$p\in\mathcal{V}$ has a variable $\bm{l}_p \in \mathcal{P} \subset \R^2$ denoting the spatial position of the corresponding joint that belongs to a finite set of possibilities~$\mathcal{P}$; each arc~$(p,q) \in \vec{\mathcal{E}}$ has a variable $t_{pq} \in \mathcal{T} = \{1,\dots,13\}$ representing the type of spacial relationship between the two nodes. The output variable~$\outputvarv$ is constructed by concatenating the unary and pairwise variables $\outputvarv = \bigl( (\bm{l}_p)_{p \in \mathcal{V}}, \: (t_{pq})_{(p,q) \in \vec{\mathcal{E}}} \bigr)$.

The structure score function is a function of discrete variables $\bm{l}_p$ and $t_{pq}$ that is defined w.r.t.\ the graph~$\mathcal{G}$:
\begin{align*}
\langle \weightv, \featuremapv(\inputvarv_i,\outputvarv) \rangle
&=
\sum_{p \in \mathcal{V}} \weight_{U,p} \phi^U_p( \bm{I}_i, \bm{l}_p ) \\
&+
\sum_{(p,q) \in \mathcal{E}} 
\weight_{T,pq}
\phi^P_{pq}(  \bm{I}_i, \bm{l}_p, t_{pq} )
\\
&+
\sum_{(p,q) \in \mathcal{E}} 
\langle \weightv_{P,pq, t_{pq}},
\Delta( \bm{l}_p - \bm{l}_q - \bm{r}_{pq}^{t_{pq}})\rangle.
\end{align*}
Here, the input
$
\inputvarv_i = \bigl(  \bm{I}_i, (\bm{r}_{pq}^{t})_{(p,q) \in \vec{\mathcal{E}}}^{t\in \mathcal{T}} \bigr)
$
consists of the original image~$\bm{I}_i$ and the mean relative positions $\bm{r}_{pq}^{t} \in \R^2$ of each type of spatial relationship corresponding to each arc~$(p,q) \in \vec{\mathcal{E}}$.
Functions~$\phi^U_p(\bm{I}_i, \cdot)$ and $\phi^P_{pq}(\bm{I}_i, \cdot, \cdot)$ compute the scores for each possible value of the discrete variables $\bm{l}_p$ and $(\bm{l}_p,t_{pq})$, respectively.
The vector-valued function $\Delta(\bm{l}) = (l_1, l_1^2, l_2, l_2^2)$, $\bm{l} \in \R^2$, measures different types of mismatch between the preferred relative displacement~$\bm{r}_{pq}^{t_{pq}}$ and the displacement $\bm{l}_p - \bm{l}_q$ coming from the labeling~$\outputvarv$.
The vector~$\weightv = \bigl(  (\weight_{U,p})_{p \in \mathcal{V}}, (\weight_{T,pq})_{(p,q)\in \vec{\mathcal{E}}}, (\weightv_{P,pq, t})_{(p,q)\in \vec{\mathcal{E}}}^{t \in \mathcal{T}} \bigr)$ is the joint vector of parameters learned by structured SVM. Overall, this setup has $2,676$ parameters.

Displacements $\bm{r}_{pq}^{t}$ and functions~$\phi^U_p$, $\phi^P$ are computed at the preprocessing stage of the structured SVM.
We follow~\citet{Chen_NIPS14} and obtain the displacements $\bm{r}_{pq}^{t}$ with the K-means clustering of the displacements of the training set. Functions~$\phi^U_p$, $\phi^P$ consists in a Convolutional Neural Network (CNN) and are also learned from the training set. We refer the reader to the work of~\citet{Chen_NIPS14} and their project page\footnote{\url{http://www.stat.ucla.edu/~xianjie.chen/projects/pose_estimation/pose_estimation.html}} for further details.
Note that the last training stage of~\citet{Chen_NIPS14} is different from ours, i.e., it consists in binary SVM on the carefully sampled sets of positive and negative examples. 

To run SSVM, we define the structured error $\errorterm(\outputvarv_i,\outputvarv)$ as a decomposable function w.r.t. the positions of the joints
\begin{equation*}
\errorterm(\outputvarv_i,\outputvarv) = \frac{1}{|\mathcal{V}|}\sum_{p \in \mathcal{V}} \max \left(1,\frac{\Vert \bm{l}_{i,p}-\bm{l}_{p}\Vert_2^2}{s_i^2}\right)
\end{equation*}
where $\bm{l}_{i,p}$ belongs to the ground-truth labeling~$\outputvarv_i$ and $\bm{l}_{p}$ belongs to the labeling~$\outputvarv$.
The quantity $s_i\in\R$ is the scaling factor and is defined by the distance between the left shoulder and the right hip in the ground-truth labeling~$\outputvarv_i$.
Similarly to~\citetsup{osokin14_ECCV}, the complex dependence of the loss on the ground-truth labeling does not influence the complexity of the max oracle.

For the defined structured score and loss, the optimization problem of the max oracle can be exactly solved by the max-sum belief propagation algorithm on an acyclic graph with messages computed with the generalized distance transform (GDT)~\cite{felzenszwalb2005distTrans}.
The usage of GDTs allows to significantly reduce the oracle running time, but requires positivity constraints on the connection weights~$\weightv_{P,pq, t}$.
BCFW with positivity constraints is described in Appendix~\ref{app:pos_const}.

The resulting max oracle is quite slow (2 seconds per image) even when the belief propagation algorithm is optimized with GDTs.
Slow running time made it intractable for us to run the experiments on the full original training set consisting of $1,000$ images.
We use only the first $100$ images and refer to this dataset as LSP-small. However, both the CNN training and clustering of the displacements were done on the original training set.

\newpage

\section{Full experimental evaluation: comparing BCFW variants}
\label{app:exp1}
In Figures~\ref{fig:app:exp1_dataset1} and~\ref{fig:app:exp1_dataset2}, we give the detailed results of the experiments described in section~\ref{subsec:exp1}.
As a reminder, we are comparing different methods (caching vs no caching, gap sampling vs uniform sampling, pairwise FW steps vs regular FW steps) on different datasets in three main regimes w.r.t.~$\regularizerweight$.
For each dataset, we use the good value of~$\regularizerweight$ (``good" meaning the smallest possible test error) together with its smaller and larger values.
The three regimes are displayed in the middle~(b), top~(a) and bottom~(c) of each subfigure, respectively.

\section{Full experimental evaluation: regularization path}
\label{app:exp2}
In this section, we evaluate the regularization path method proposed in Section~\ref{sec:reg_path}.
Our experiments are organized in three stages.
First, we choose the $\factorRegPath$ parameter for Algorithm~\ref{alg:rpAlgorithmSSVM} computing the $\epsilon$-approximate regularization path.
Second, we define and evaluate the heuristic regularization path.
Finally, we compare both $\epsilon$-approximate and heuristic paths against the grid search on multiple datasets.

\paragraph{$\epsilon$-approximate path.}
Algorithm~\ref{alg:rpAlgorithmSSVM} for computing the $\epsilon$-approximate regularization path has one parameter~$\factorRegPath$ controlling how large are the induction steps in terms of~$\regularizerweight$.
This parameter provides the trade-off between the number of breakpoints and the accuracy of optimization for each breakpoint.
We explore this trade-off in Figure~\ref{fig:regPath_epsApprox_ocrSmall}.
For several values of $\factorRegPath$, we report the cumulative number of effective passes (bottom plots) and cumulative time (top plots) required to get  $\epsilon$-approximate solution for each $\regularizerweight$. 
We report both plots for the two methods used as the SSVM solver: BCPFW with gap sampling and caching; BCFW with gap sampling and without caching.
We conclude that both too small ($< 0.5$) and too large ($\approx 1$) values of parameter~$\factorRegPath$ result in slower methods, but overall the method is not too sensitive to~$\factorRegPath$. In all remaining experiments, we use $\factorRegPath=0.9$ when computing the $\epsilon$-approximate path.

\paragraph{Heuristic path.}
When computing an $\epsilon$-approximate regularization path, Algorithm~\ref{alg:rpAlgorithmSSVM} needs to perform at least one pass over the dataset at each breakpoint to check the convergence criterion, i.e., that the gap is not larger than~$\epsilon$.
When having many breakpoints, this extra pass can be of significant cost, especially for large values of~$\regularizerweight$ where the SSVM solver converges very quickly.
However, in practice, we observe that the stale gap estimates are often good enough to determine convergence and actually checking the convergence criterion is not necessary.
At the cost of loosing the guarantees, we can expect computational speed-up.
We refer to the result of Algorithm~\ref{alg:rpAlgorithmSSVM} when the SSVM solver (at line~\ref{alg:rpAlgorithmSSVM:SSVM}) uses stale gap estimates instead of the exact gaps to check the convergence as a \emph{heuristic path}.
In the case of heuristic path, the parameter~$\factorRegPath$ provides a trade-off between the running time and accuracy of the path.
In Figure~\ref{fig:regPath_heuristic_ocrSmall}, we illustrate this trade-off on the OCR-small dataset.
For several values of~$\factorRegPath$, we report the true value of the duality gap at each breakpoint (the SSVM solver terminates when the stale gap estimate is below~$\factorRegPath\epsilon$, $\epsilon=0.1$) and the cumulative time.
We use BCFW + gap sampling and BCPFW + gap sampling + cache as the SSVM solvers.
For large $\factorRegPath$, we observe that the solutions for small $\regularizerweight$ are not $\epsilon$-approximate and that the method in this regime requires less number of passes over the data, i.e., runs faster.
In what follows we always use $\factorRegPath =0.7$ for the heuristic path as it is the largest value providing accurate enough results.

\paragraph{Comparison of paths against grid search.}
Finally, we compare the regularization path methods to the standard grid search approach.
We define a grid of 31 values chosen to cover all the values of $\regularizerweight$ used in the experiments of Section~\ref{app:exp1} ($2^{15},2^{14},\dots,2^{-15}$).
For each value of the grid, we optimize the SSVM objective optimize either independently or by warm-starting from the nearest larger value.
We have considered two variants of warm start: keeping the primal variables and rescaling the dual variables, or keeping the dual variables and rescaling the primal variables.
In our experiments, we do not notice any consistent difference between the two approaches, so we only use the first type of warm start.

Figure~\ref{fig:regPath_4datasets} presents the results on the four datasets: HorseSeg-small (Figure~\ref{fig:regPath_horseSmall}),
OCR-small (Figure~\ref{fig:regPath_ocrSmall}),
HorseSeg-medium (Figure~\ref{fig:regPath_horseMedium}),
OCR-large (Figure~\ref{fig:regPath_ocrLarge}).

For both the OCR-large and HorseSeg-medium datasets, neither heuristic, nor $\epsilon$-approximate path methods did not reach their stopping criterion and were terminated at the time limit of 24 hours.
In the case of OCR-large, the regularization path reached a value of $\regularizerweight$ smaller than the lower limit of the predefined grid, i.e., $2^{-15}$. 
In the case of HorseSeg-medium, the grid search methods did not reach the lower bound of the grid and were terminated at the 24-hour time limit.

\newpage
\def \tableHeight {3cm}
\begin{figure*}
\resizebox{0.98\textwidth}{!}{
\begin{tabular}{c@{\; \; }c@{\qquad \quad}c@{\; \; }c}
\multicolumn{4}{c}{\includegraphics[width=1.1\textwidth]{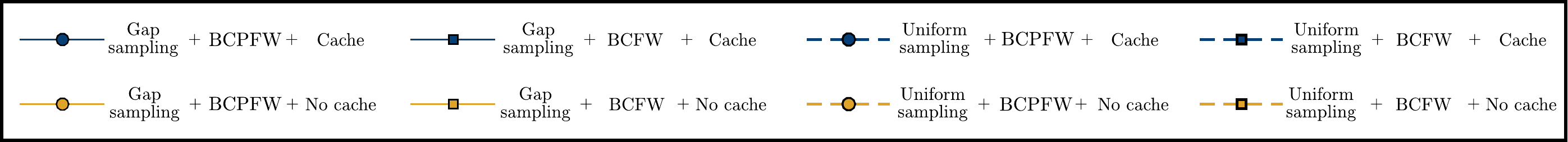}}  \\[0.1cm]
\includegraphics[height=\tableHeight]{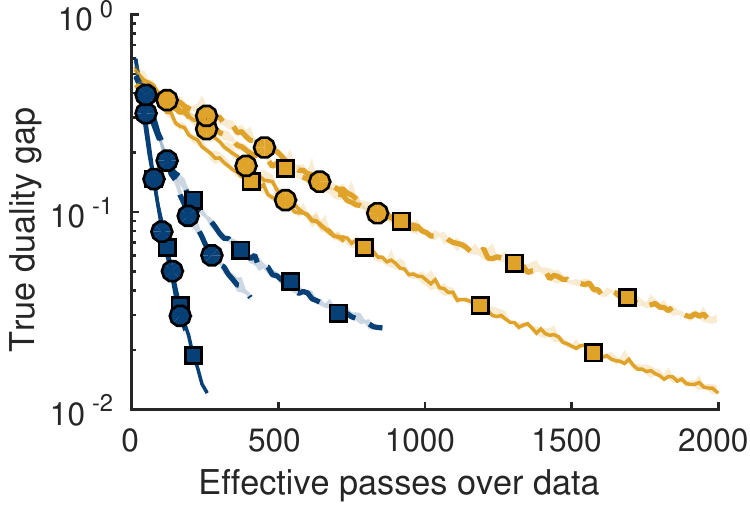} &
\includegraphics[height=\tableHeight, trim=0.6cm 0cm 0cm 0cm , clip]{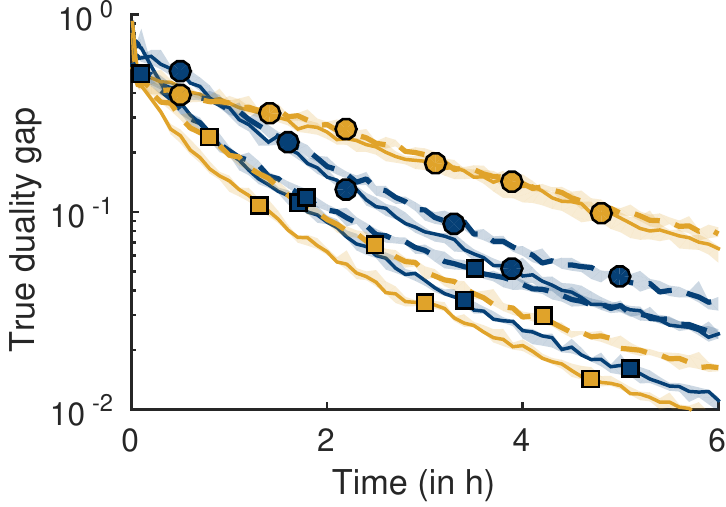}
&
\includegraphics[height=\tableHeight]{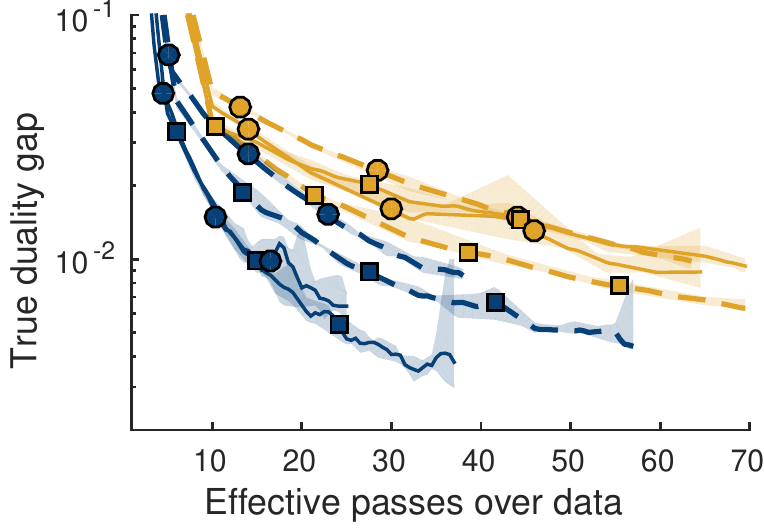} &
\includegraphics[height=\tableHeight, trim=0.6cm 0cm 0cm 0cm, clip]{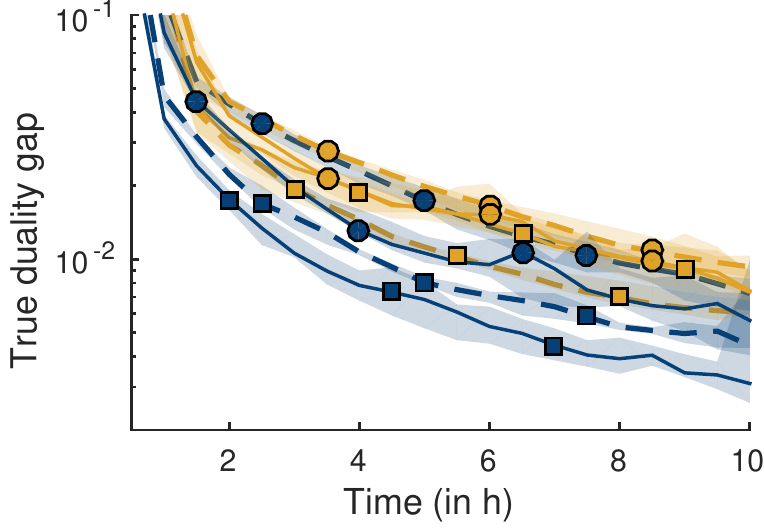}
\\
\multicolumn{2}{c}{(a) OCR-large, $\regularizerweight=0.0001$} &  \multicolumn{2}{c}{(a) CoNLL, $\regularizerweight=0.0001$}
\\
\includegraphics[height=\tableHeight]{figures/gapVsPass_ocrLarge_lambda_1e-03.pdf} &
\includegraphics[height=\tableHeight, trim=0.6cm 0cm 0cm 0cm, clip]{figures/gapVsTime_ocrLarge_lambda_1e-03.pdf}
&
\includegraphics[height=\tableHeight]{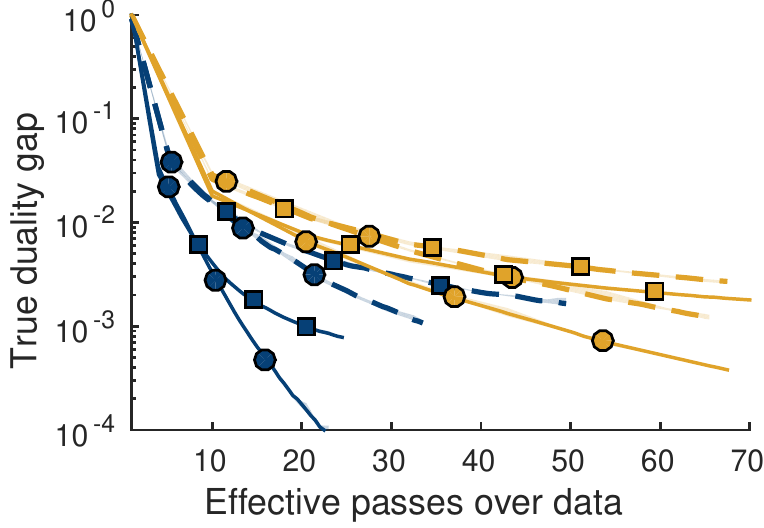} &
\includegraphics[height=\tableHeight, trim=0.6cm 0cm 0cm 0cm, clip]{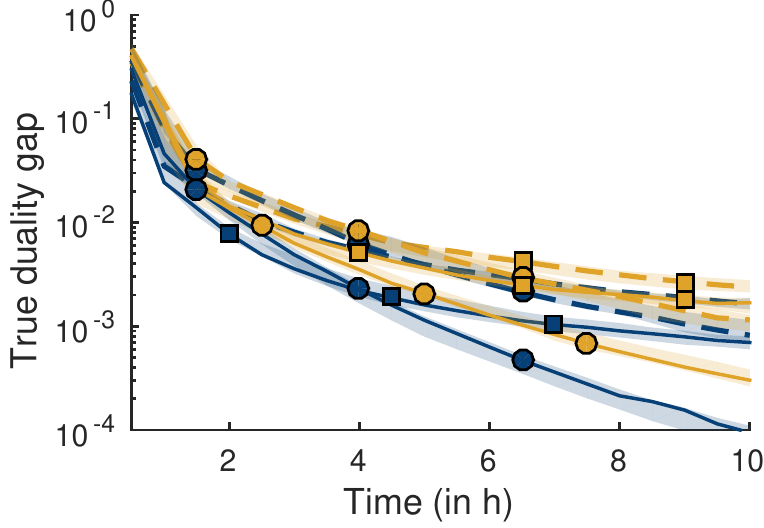}
\\
\multicolumn{2}{c}{(b) OCR-large, $\regularizerweight=0.001$} &  \multicolumn{2}{c}{(b) CoNLL, $\regularizerweight=0.01$}
\\
\includegraphics[height=\tableHeight]{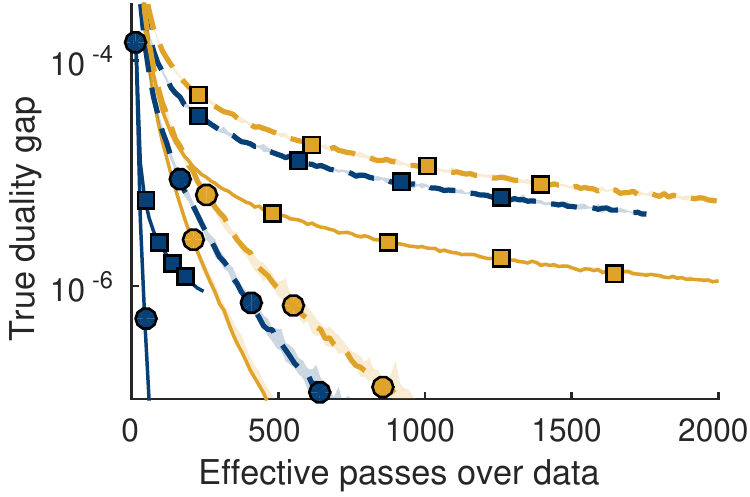} &
\includegraphics[height=\tableHeight, trim=0.6cm 0cm 0cm 0cm, clip]{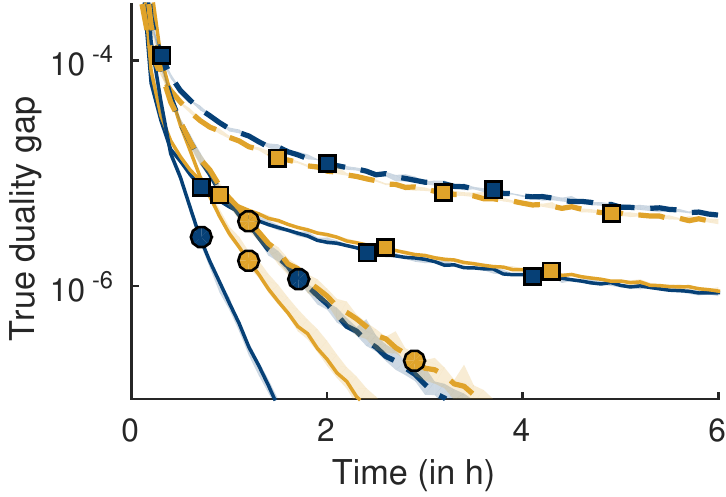}
&
\includegraphics[height=\tableHeight]{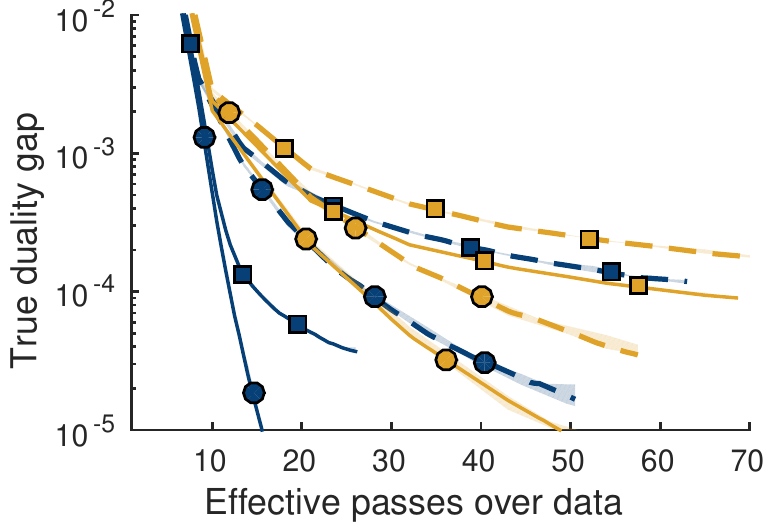} &
\includegraphics[height=\tableHeight, trim=0.6cm 0cm 0cm 0cm, clip]{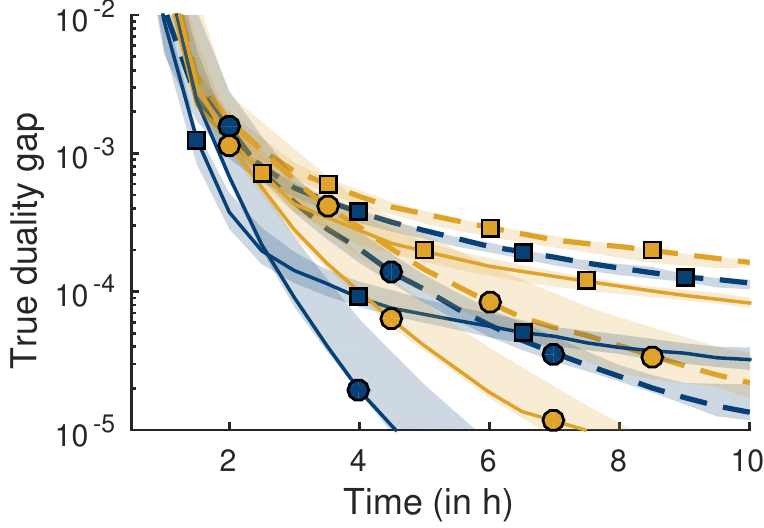}
\\
\multicolumn{2}{c}{(c) OCR-large, $\regularizerweight=0.1$} &  \multicolumn{2}{c}{(c) CoNLL, $\regularizerweight=0.1$}
\\[2mm]\hline \\[-1mm]

\includegraphics[height=\tableHeight]{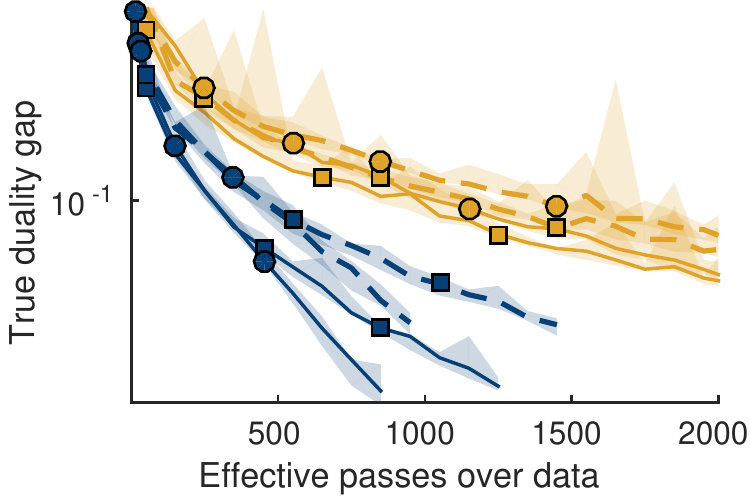} &
\includegraphics[height=\tableHeight, trim=0.6cm 0cm 0cm 0cm, clip]{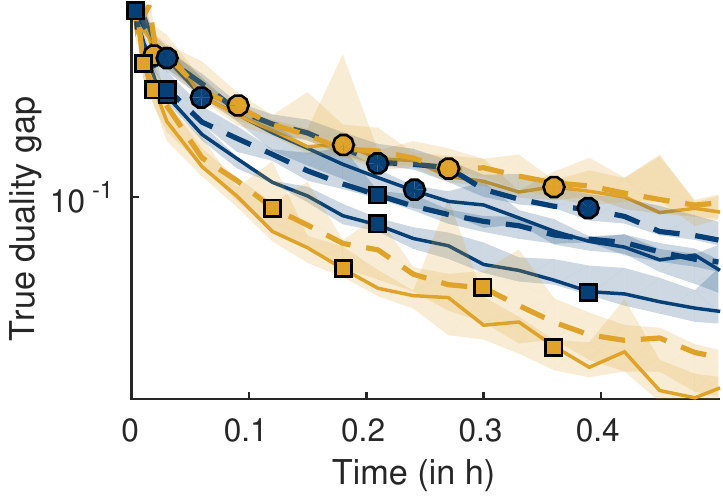}
&
\includegraphics[height=\tableHeight]{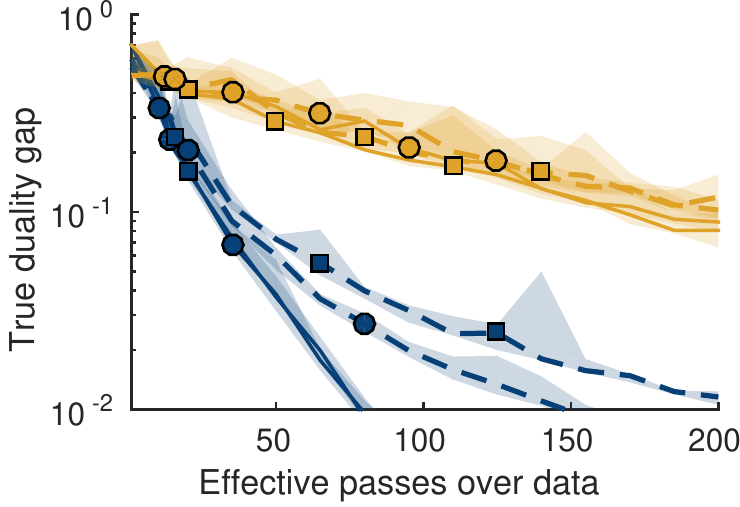} &
\includegraphics[height=\tableHeight, trim=0.6cm 0cm 0cm 0cm, clip]{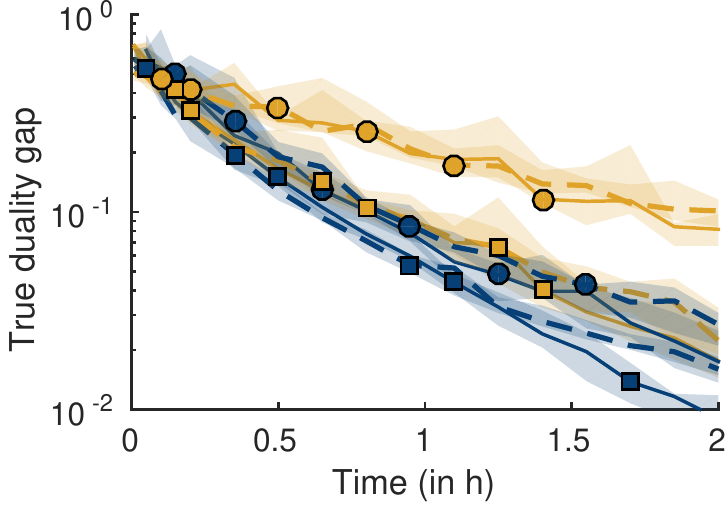}
\\
\multicolumn{2}{c}{(a) HorseSeg-small, $\regularizerweight=0.1$} &  \multicolumn{2}{c}{(a) HorseSeg-medium, $\regularizerweight=0.1$}
\\
\includegraphics[height=\tableHeight]{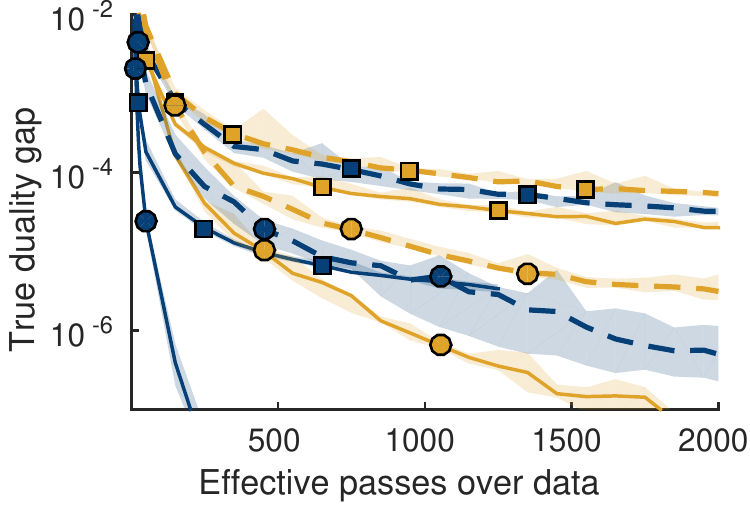} &
\includegraphics[height=\tableHeight, trim=0.6cm 0cm 0cm 0cm, clip]{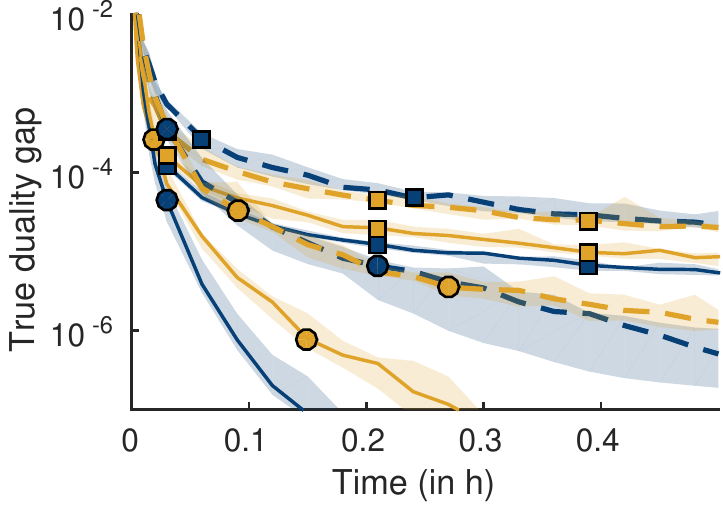}
&
\includegraphics[height=\tableHeight]{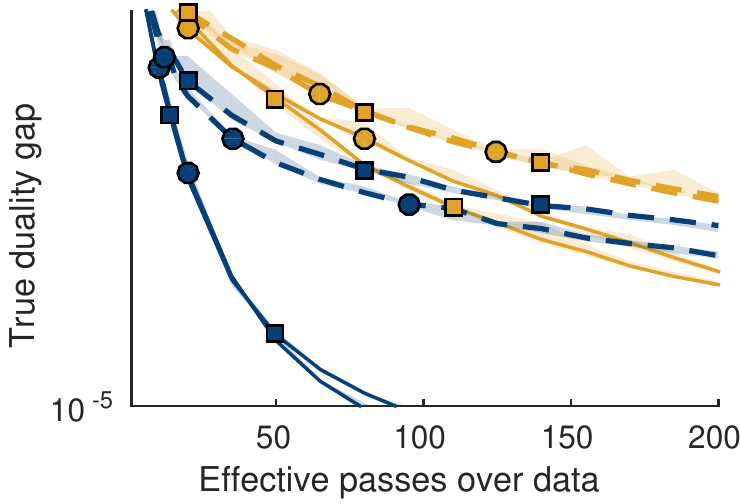} &
\includegraphics[height=\tableHeight, trim=0.6cm 0cm 0cm 0cm, clip]{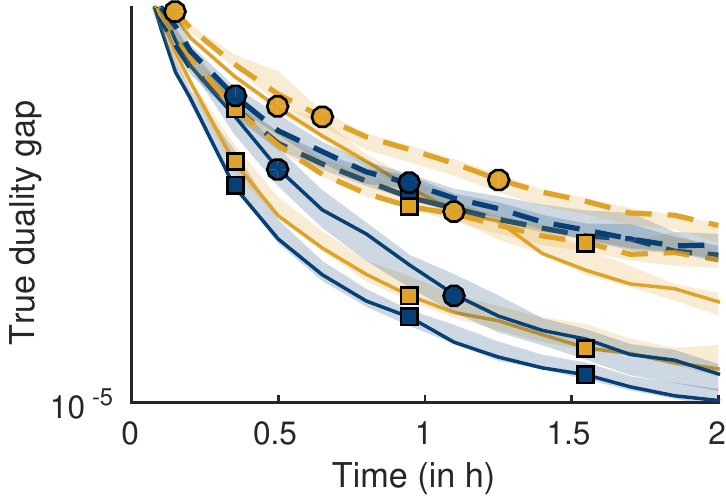}
\\
\multicolumn{2}{c}{(b) HorseSeg-small, $\regularizerweight=100$} &  \multicolumn{2}{c}{(b) HorseSeg-medium, $\regularizerweight=1$}
\\
\includegraphics[height=\tableHeight]{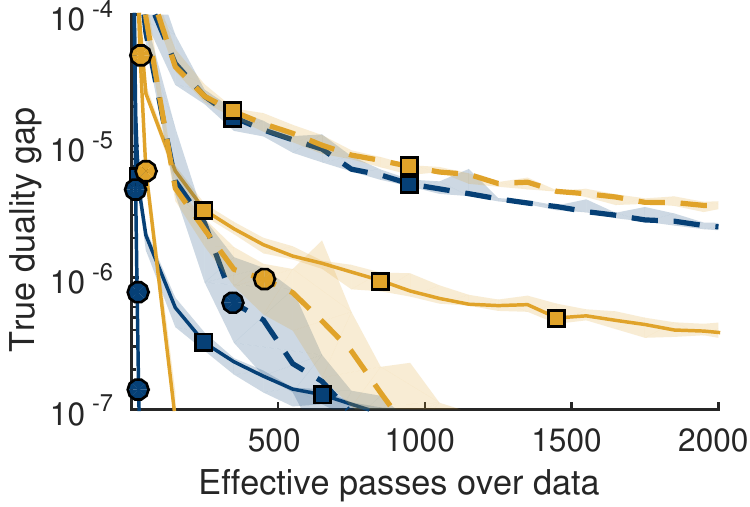} &
\includegraphics[height=\tableHeight, trim=0.6cm 0cm 0cm 0cm, clip]{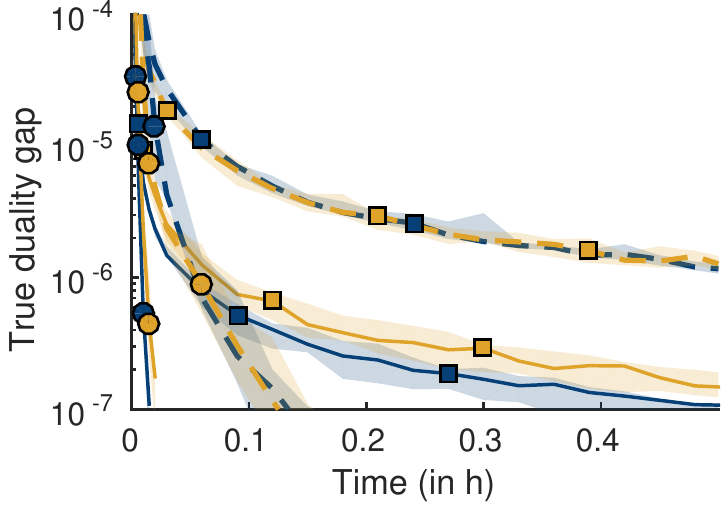}
&
\includegraphics[height=\tableHeight]{figures/gapVsPass_horseMedium_lambda_1e+01.pdf} &
\includegraphics[height=\tableHeight, trim=0.6cm 0cm 0cm 0cm, clip]{figures/gapVsTime_horseMedium_lambda_1e+01.pdf}
\\
\multicolumn{2}{c}{(c) HorseSeg-small, $\regularizerweight=1000$} &  \multicolumn{2}{c}{(c) HorseSeg-medium, $\regularizerweight=10$}
\\
\end{tabular}
}
\caption{Comparison of the variants of BCFW.
    We compare 8 different methods that can be represented by 3 binary dimensions: object sampling, caching, type of FW steps.
    We represent these dimensions in different ways:
    the dimension of caching (in blue) versus no caching (in orange) is represented through colors,
    the dimension of gap sampling (solid lines) versus uniform sampling (dashed lines) is represented through line style,
    the dimension of pairwise FW steps (circle markers) versus regular FW steps (square markers) is represented through markers.
    For each method, we report both the number of effective passes over data ($n$ oracle calls) and the running time against obtained duality gap (computed offline). \textit{The figure is continued in Figure~\ref{fig:app:exp1_dataset2}.}
  }
  \label{fig:app:exp1_dataset1}
\end{figure*}

\clearpage

\begin{figure*}[t!]
\resizebox{0.98\textwidth}{!}{
\begin{tabular}{c@{\; \; }c@{\qquad \quad}c@{\; \; }c}
\multicolumn{4}{c}{\includegraphics[width=1.1\textwidth]{figures/legend_long_frame_white_new_BCPFW.pdf}}  \\[0.1cm]
\includegraphics[height=\tableHeight]{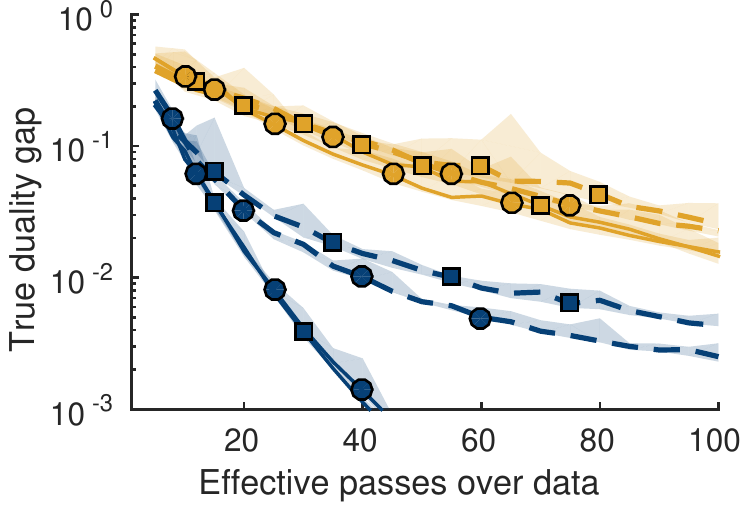} &
\includegraphics[height=\tableHeight, trim=0.6cm 0cm 0cm 0cm , clip]{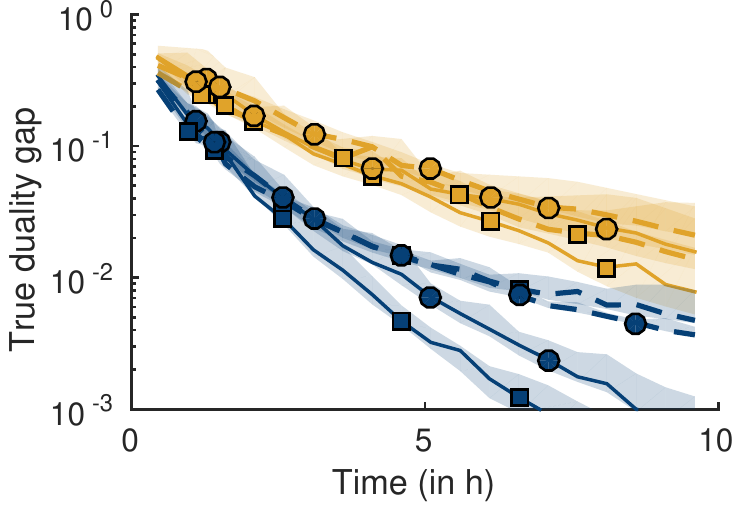}
&
\includegraphics[height=\tableHeight]{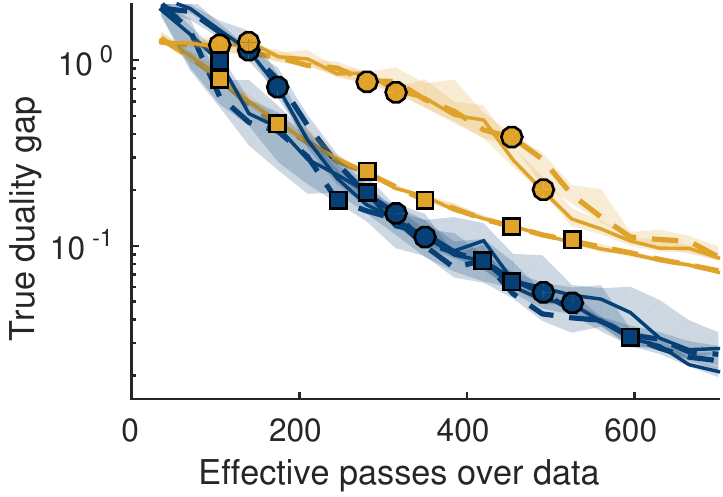} &
\includegraphics[height=\tableHeight, trim=0.6cm 0cm 0cm 0cm, clip]{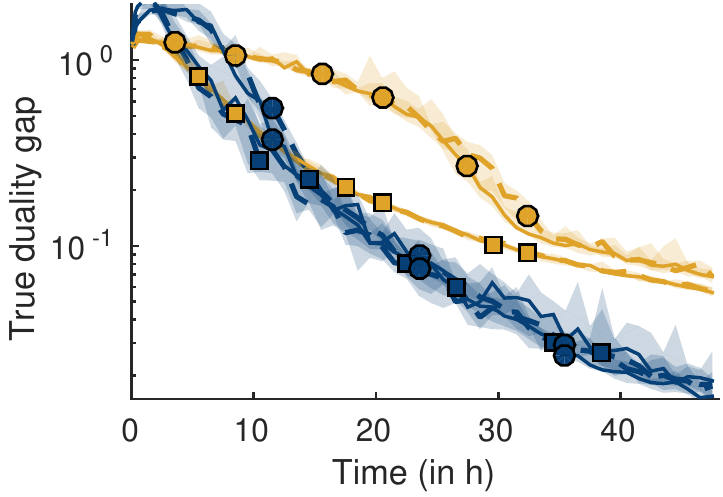}
\\
\multicolumn{2}{c}{(a) HorseSeg-large, $\regularizerweight=0.1$} &  \multicolumn{2}{c}{(a) LSP-small, $\regularizerweight=10$}
\\
\includegraphics[height=\tableHeight]{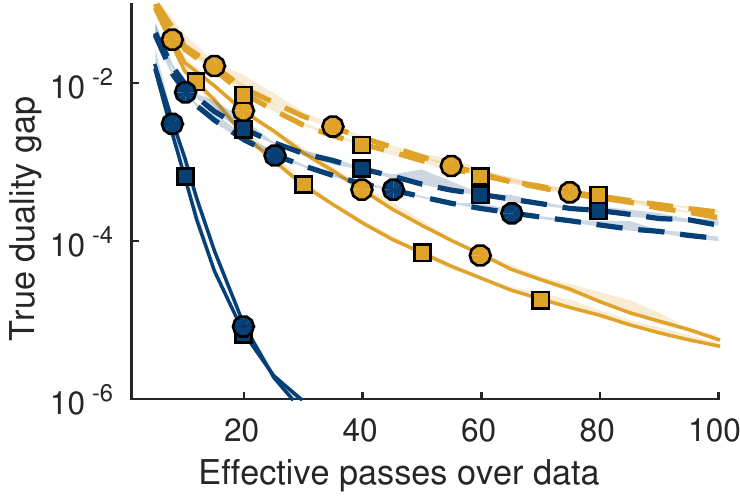} &
\includegraphics[height=\tableHeight, trim=0.6cm 0cm 0cm 0cm, clip]{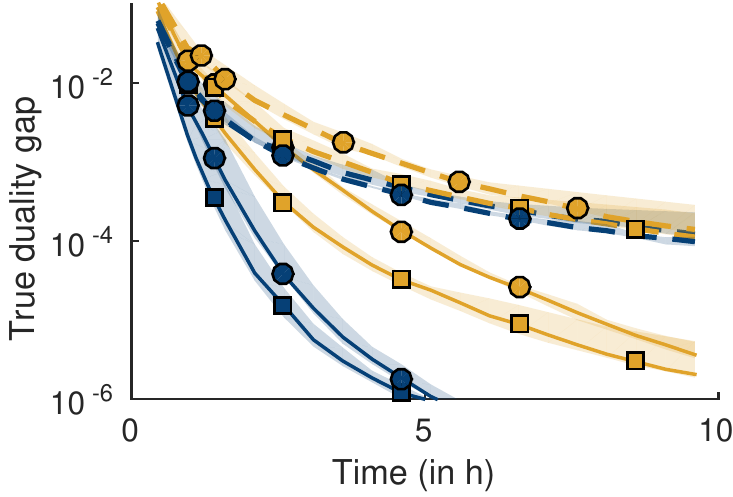}
&
\includegraphics[height=\tableHeight]{figures/gapVsPass_lspSmall_lambda_1e+02.pdf} &
\includegraphics[height=\tableHeight, trim=0.6cm 0cm 0cm 0cm, clip]{figures/gapVsTime_lspSmall_lambda_1e+02.pdf}
\\
\multicolumn{2}{c}{(b) HorseSeg-large, $\regularizerweight=1$} &  \multicolumn{2}{c}{(b) LSP-small, $\regularizerweight=100$}
\\
\includegraphics[height=\tableHeight]{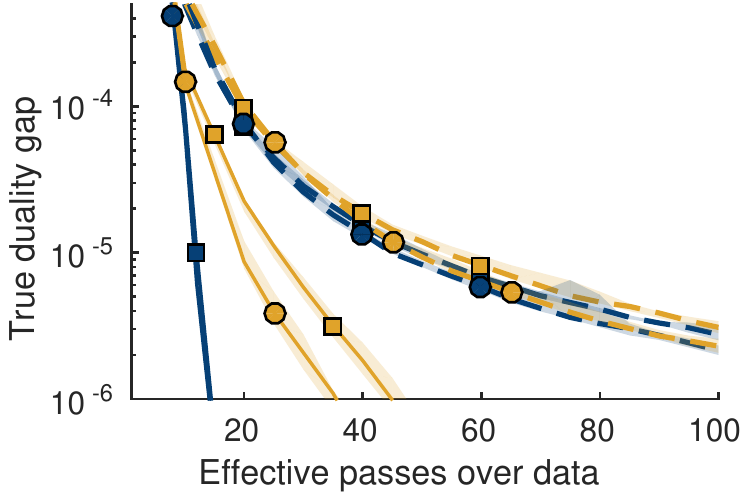} &
\includegraphics[height=\tableHeight, trim=0.6cm 0cm 0cm 0cm, clip]{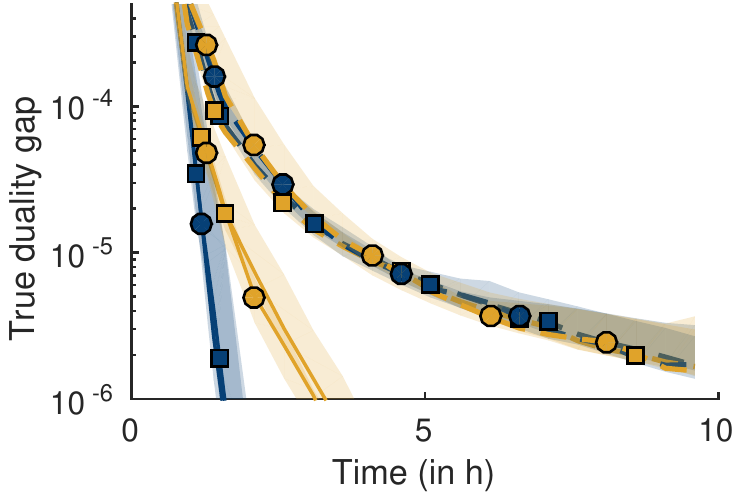}
&
\includegraphics[height=\tableHeight]{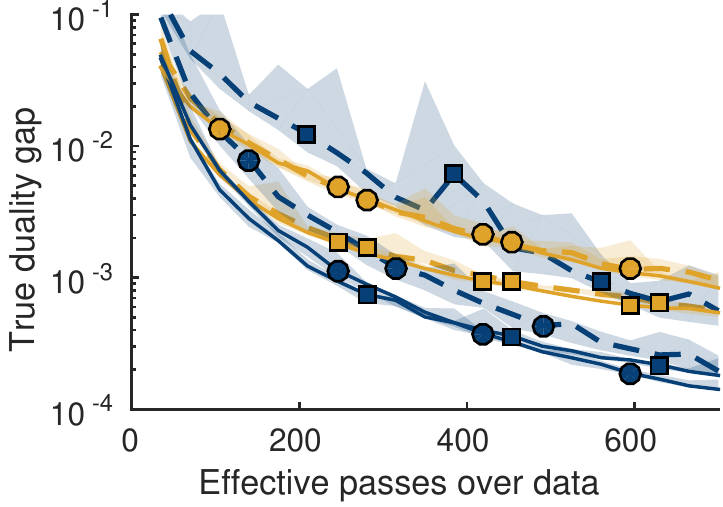} &
\includegraphics[height=\tableHeight, trim=0.6cm 0cm 0cm 0cm, clip]{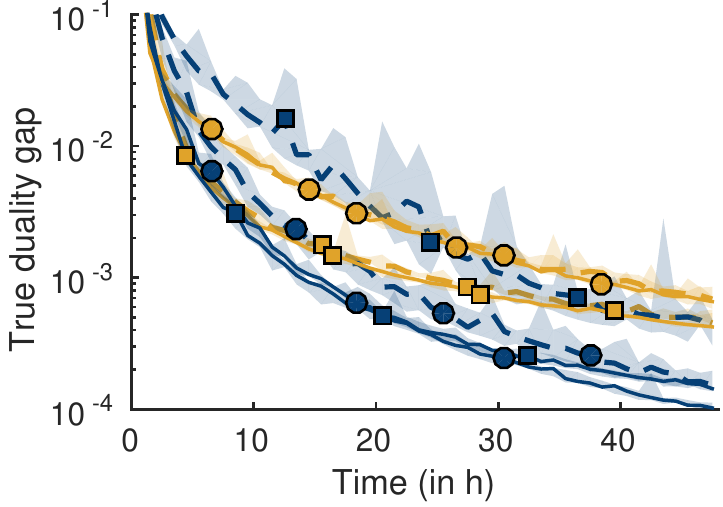}
\\
\multicolumn{2}{c}{(c) HorseSeg-large, $\regularizerweight=10$} &  \multicolumn{2}{c}{(c) LSP-small, $\regularizerweight=1000$}
\\[-0cm]
\end{tabular}
}
\caption{\label{fig:app:exp1_dataset2} \textit{Continuation of Figure~\ref{fig:app:exp1_dataset1}.}
    Comparison of the variants of BCFW on HorseSeg-large and LSP-small.
    }
\end{figure*}

\begin{figure*}[h!]
    \centering
    \begin{subfigure}[b]{0.48\textwidth}
	\resizebox{\textwidth}{!}{
    \begin{tabular}{c@{$\:$}c}
        \includegraphics[width=0.5\textwidth]{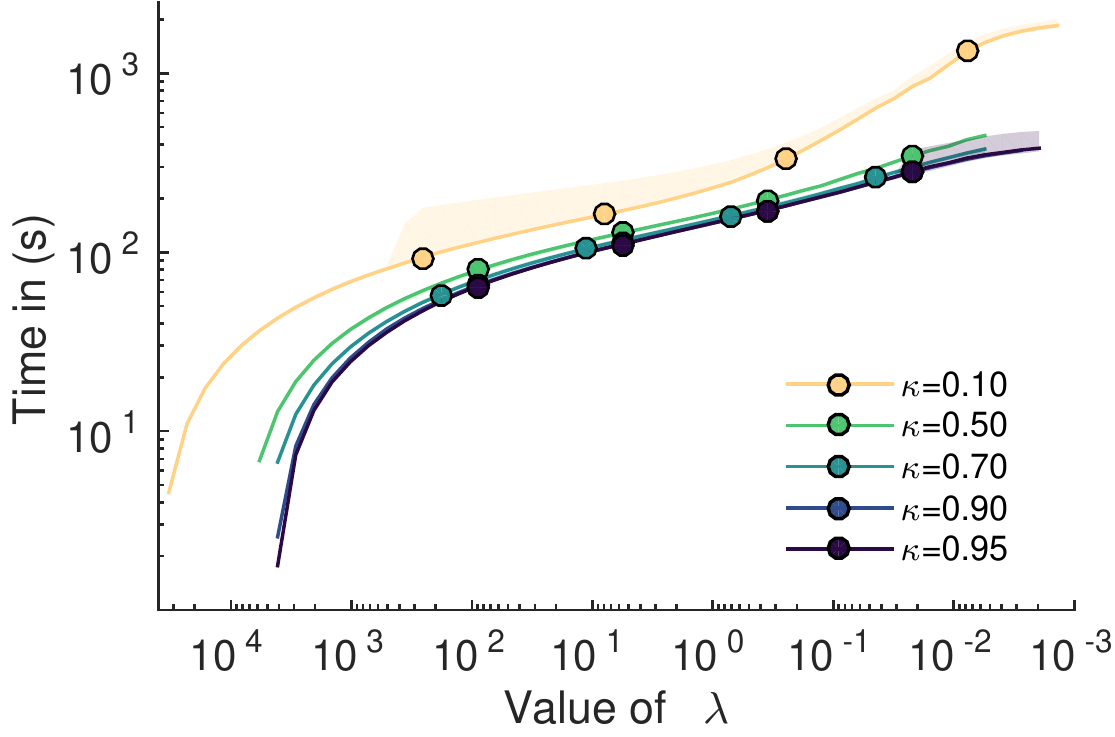} &
        \includegraphics[width=0.5\textwidth]{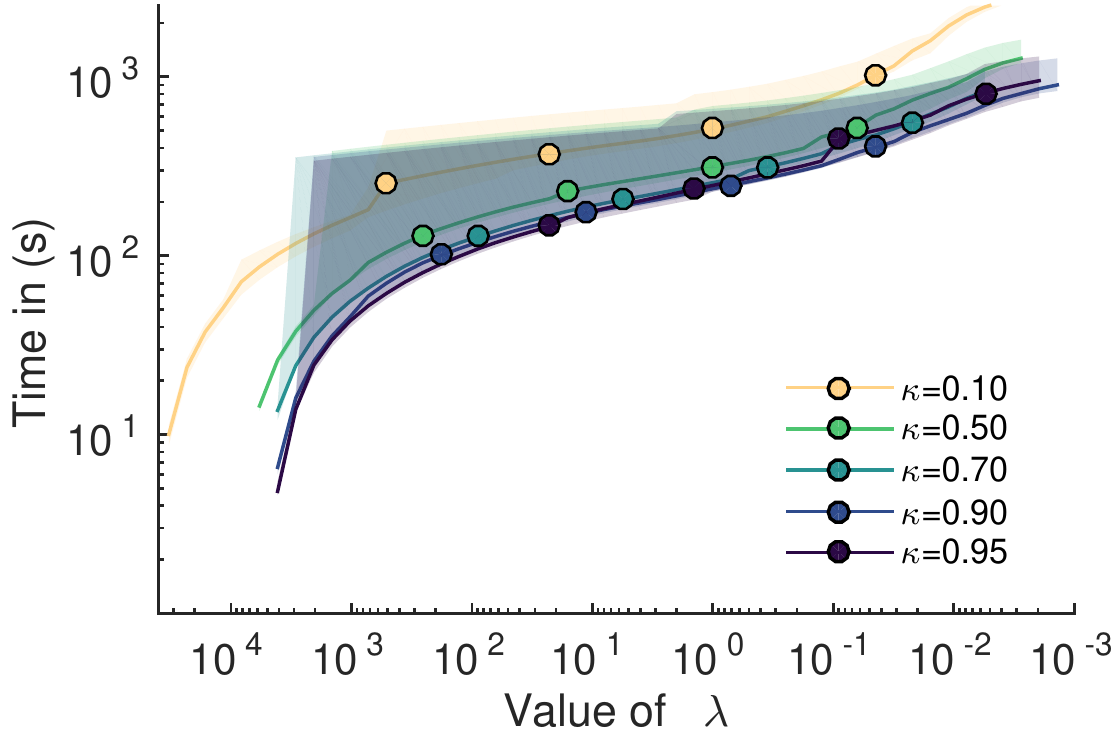}\\
        \includegraphics[width=0.5\textwidth]{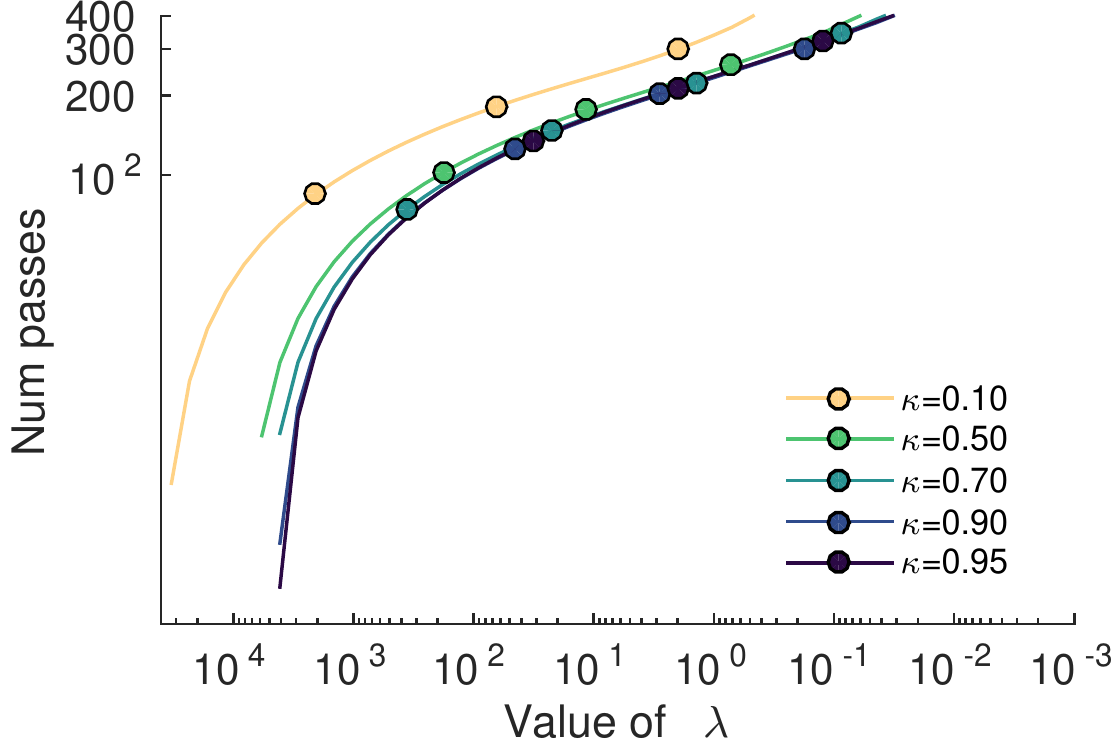} &
        \includegraphics[width=0.5\textwidth]{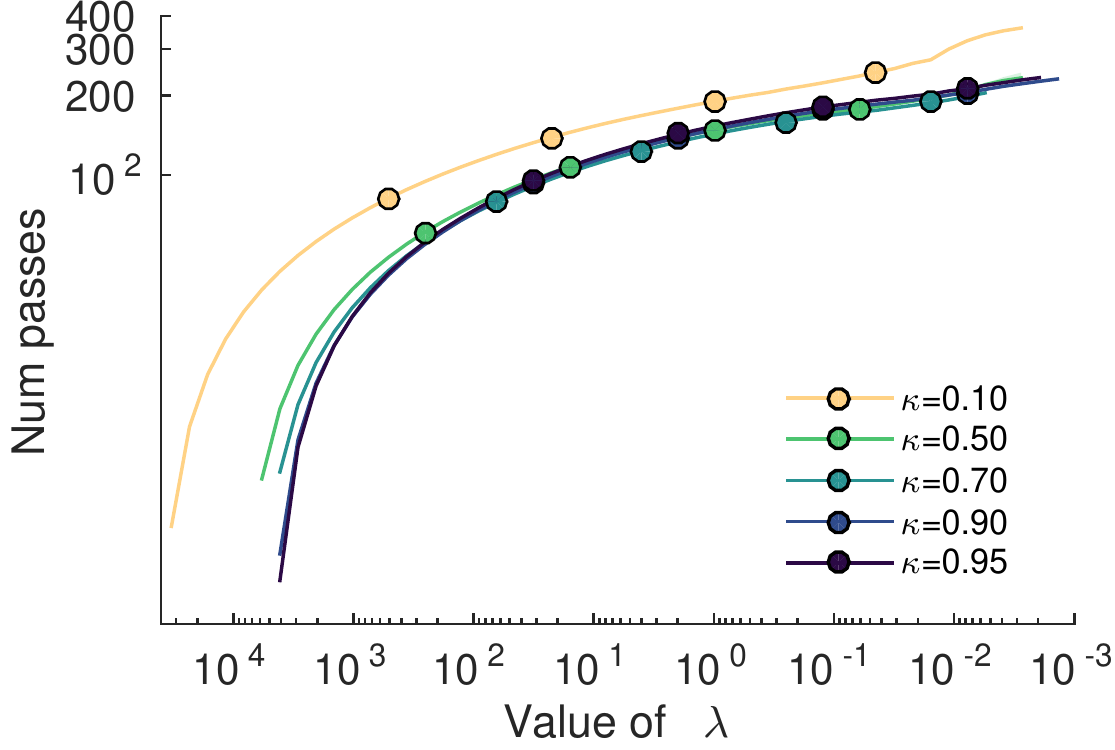}\\
        {\scriptsize BCFW + gap sampling}
        &
        {\scriptsize BCPFW + gap sampling + caching}
        \\[-0.0cm]
    \end{tabular}
    }
    \caption{
        \label{fig:regPath_epsApprox_ocrSmall} $\epsilon$-approximate regularization paths. 
        For different values of $\factorRegPath$, we report the cumulative number of effective passes (bottom) and the cumulative time (top) required to get an $\epsilon$-approximate solution for each $\regularizerweight$. We analyze the two methods: BCFW + gap sampling (left) and BCPFW + gap sampling + caching (right).
    }
    \end{subfigure}
     \hfill
    \begin{subfigure}[b]{0.48\textwidth}
\resizebox{\textwidth}{!}{
    \begin{tabular}{c@{$\:$}c}
        \includegraphics[width=0.5\textwidth]{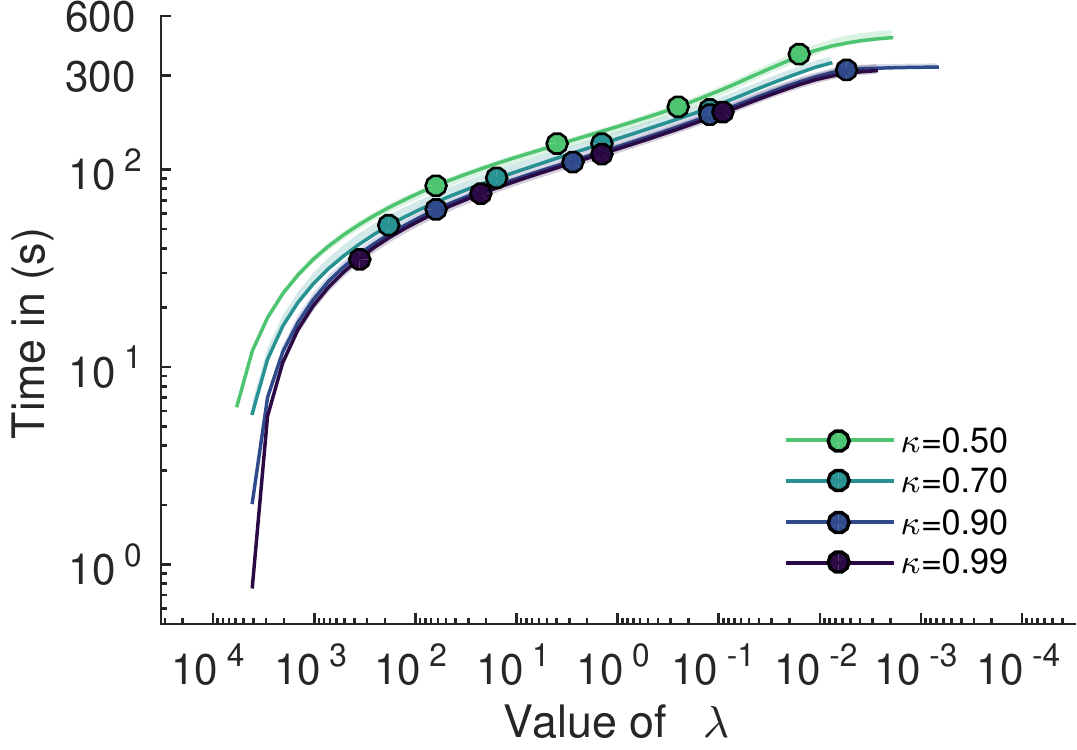} &
        \includegraphics[width=0.5\textwidth]{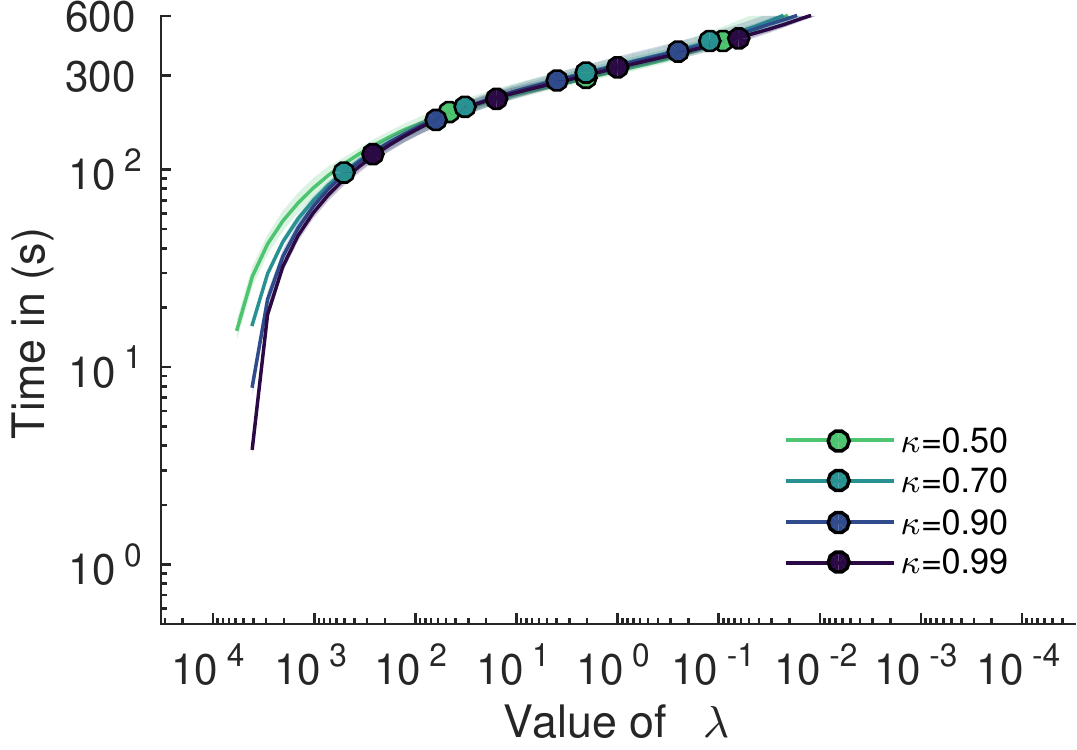}\\
        \includegraphics[width=0.5\textwidth]{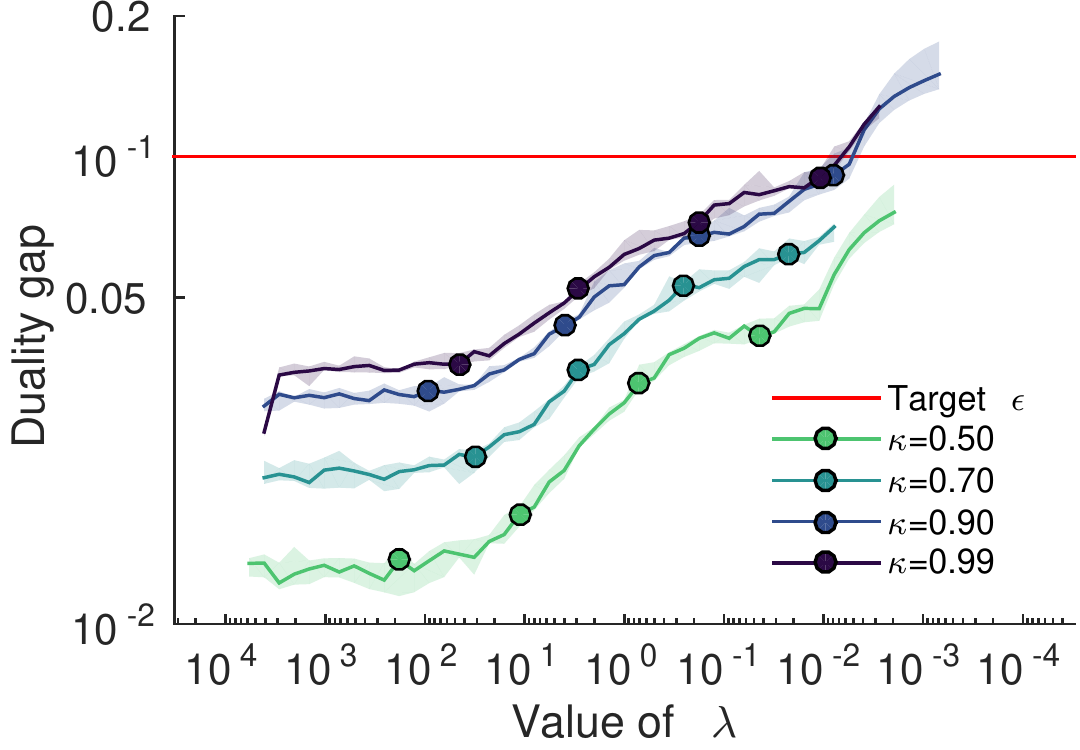} &
        \includegraphics[width=0.5\textwidth]{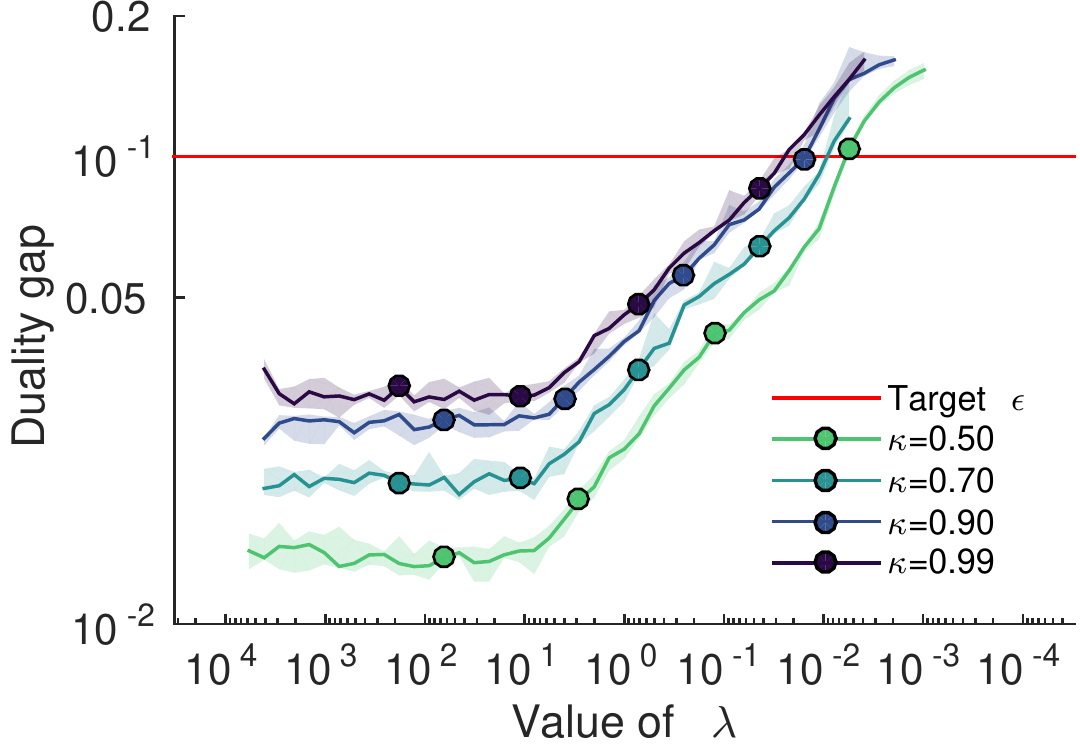}\\
        {\scriptsize BCFW + gap sampling}
        &
        {\scriptsize BCPFW + gap sampling + caching}
        \\[-0.0cm]
    \end{tabular}
    }
    \caption{
        \label{fig:regPath_heuristic_ocrSmall} Heuristic regularization paths.
        For different values of $\factorRegPath$, we report the cumulative time required to compute a path until $\regularizerweight$ (top) along with the true value of the duality gap obtained for each $\regularizerweight$ (bottom).We analyze the two methods: BCFW + gap sampling (left) and BCPFW + gap sampling + caching (right).
    }
    \end{subfigure}
	\caption{Experiment exploring the effect of $\factorRegPath$ for regularization paths computed on the OCR-small dataset.}
	\label{fig:bigFigure}
\end{figure*}

\clearpage
\begin{figure*}
\centering
\begin{subfigure}[b]{0.48\textwidth}
    \begin{tabular}{c@{$\:$}c}
        \includegraphics[width=0.49\textwidth]{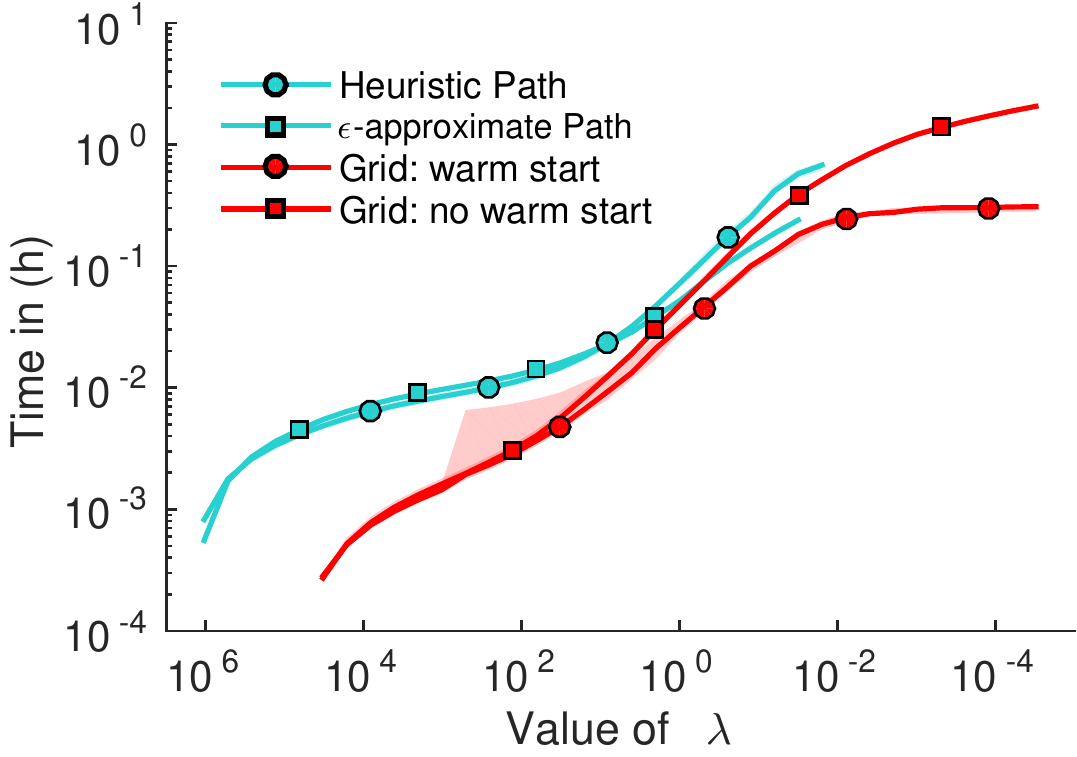} &
        \includegraphics[width=0.49\textwidth]{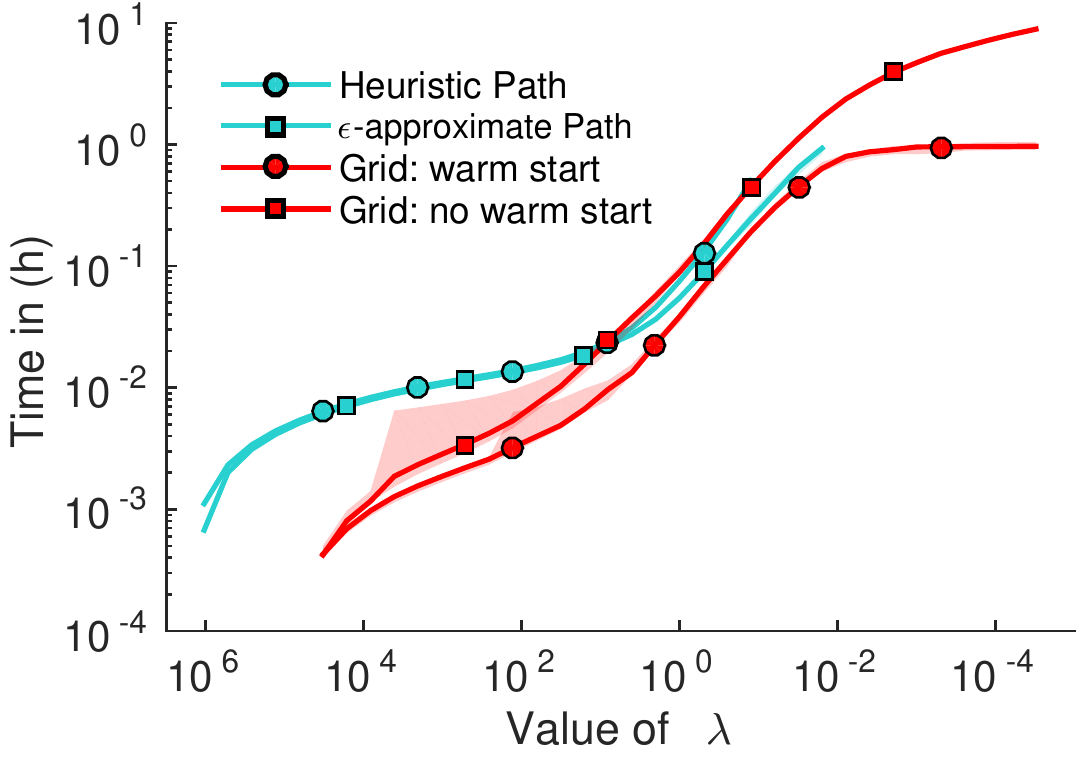}\\
        \includegraphics[width=0.49\textwidth]{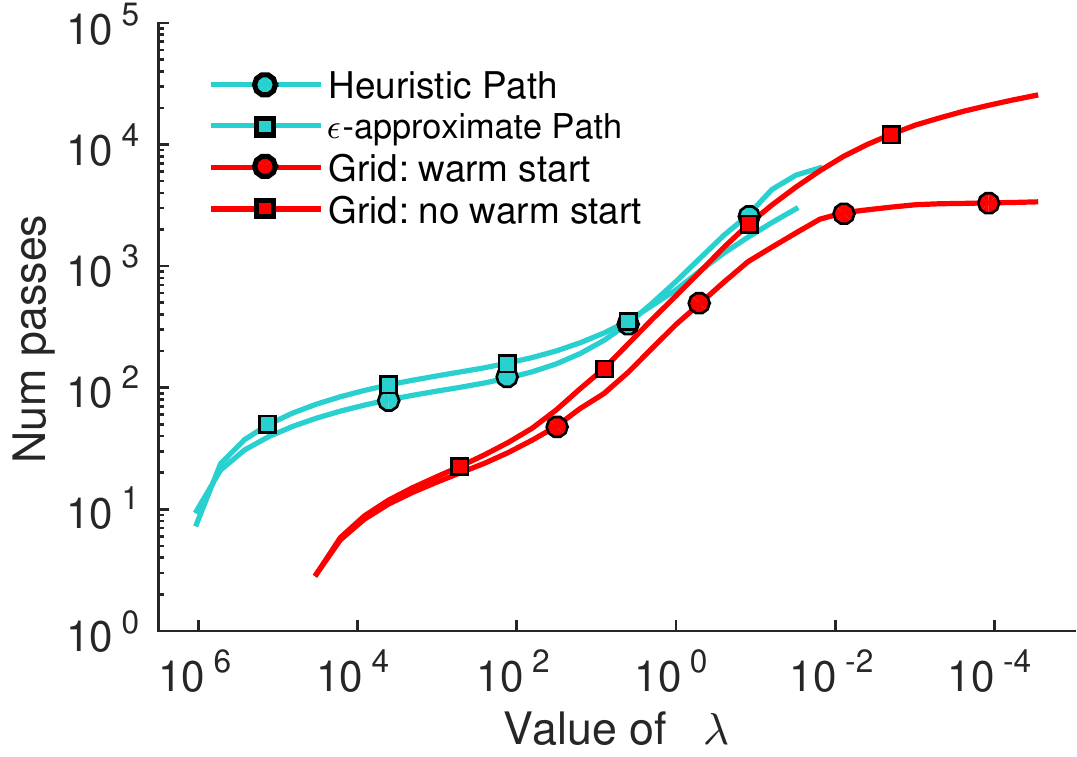} &
        \includegraphics[width=0.49\textwidth]{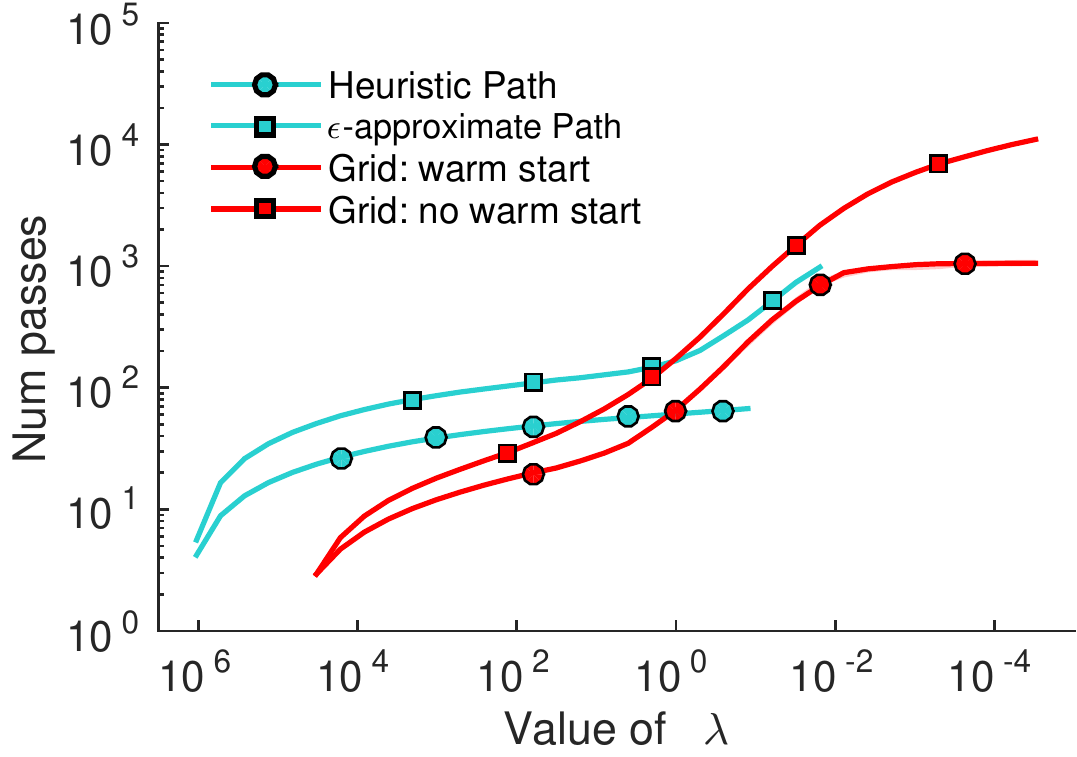}\\
        {\scriptsize BCFW + gap sampling}
        &
        {\scriptsize BCPFW + gap sampling + caching}
        \\[-0.0cm]
    \end{tabular}
    \caption{
        \label{fig:regPath_horseSmall} HorseSeg-small
        }
\end{subfigure}
\quad
\begin{subfigure}[b]{0.48\textwidth}
    \begin{tabular}{c@{$\:$}c}
        \includegraphics[width=0.49\textwidth]{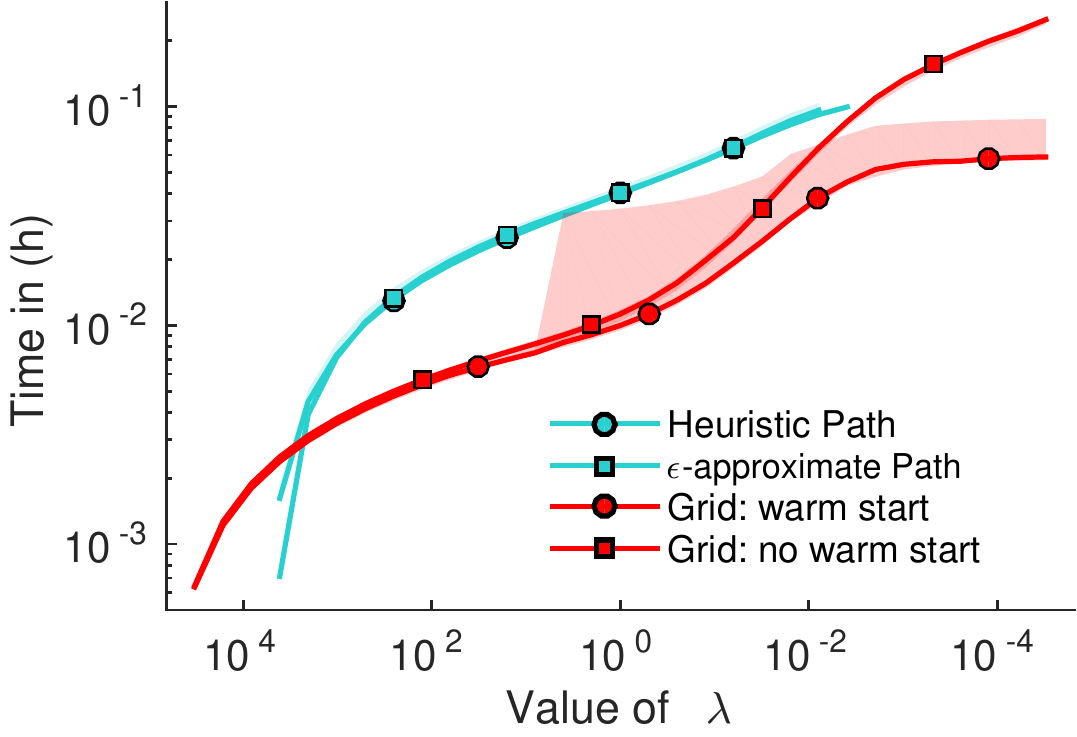} &
        \includegraphics[width=0.49\textwidth]{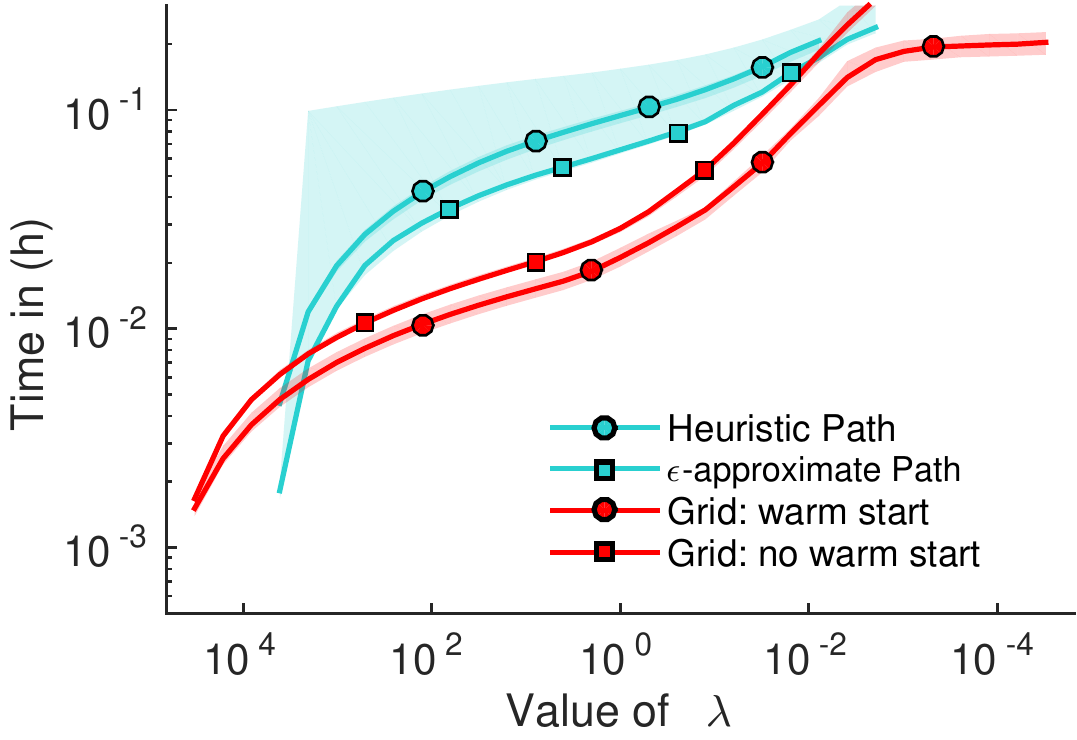}\\
        \includegraphics[width=0.49\textwidth]{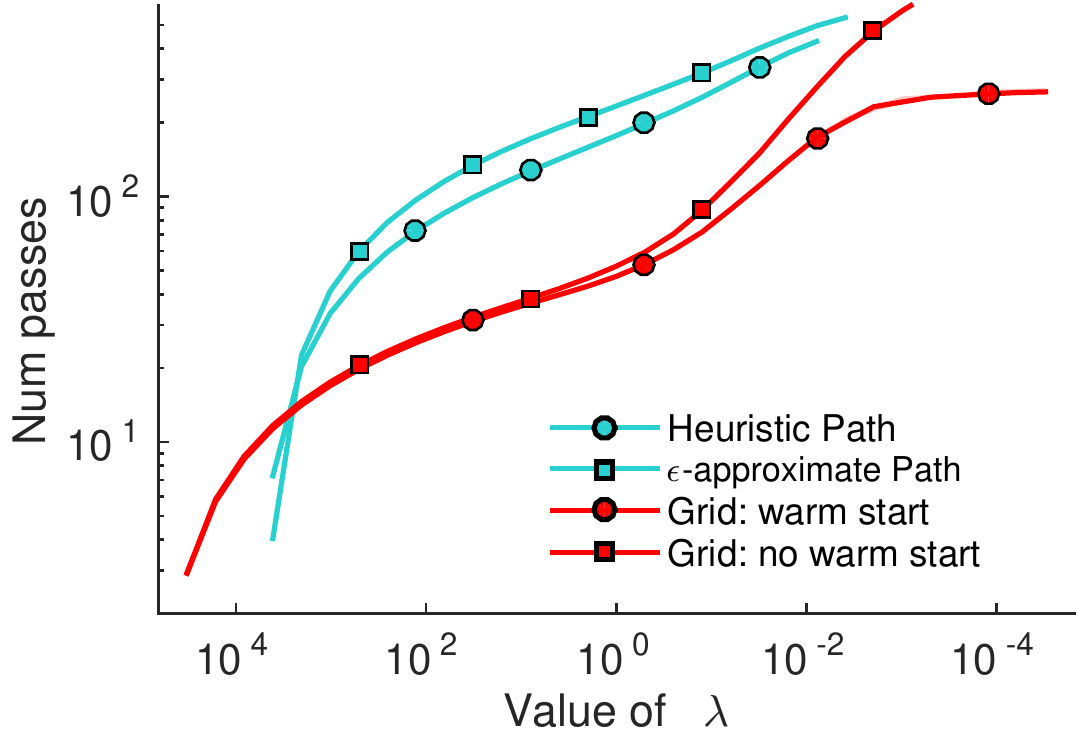} &
        \includegraphics[width=0.49\textwidth]{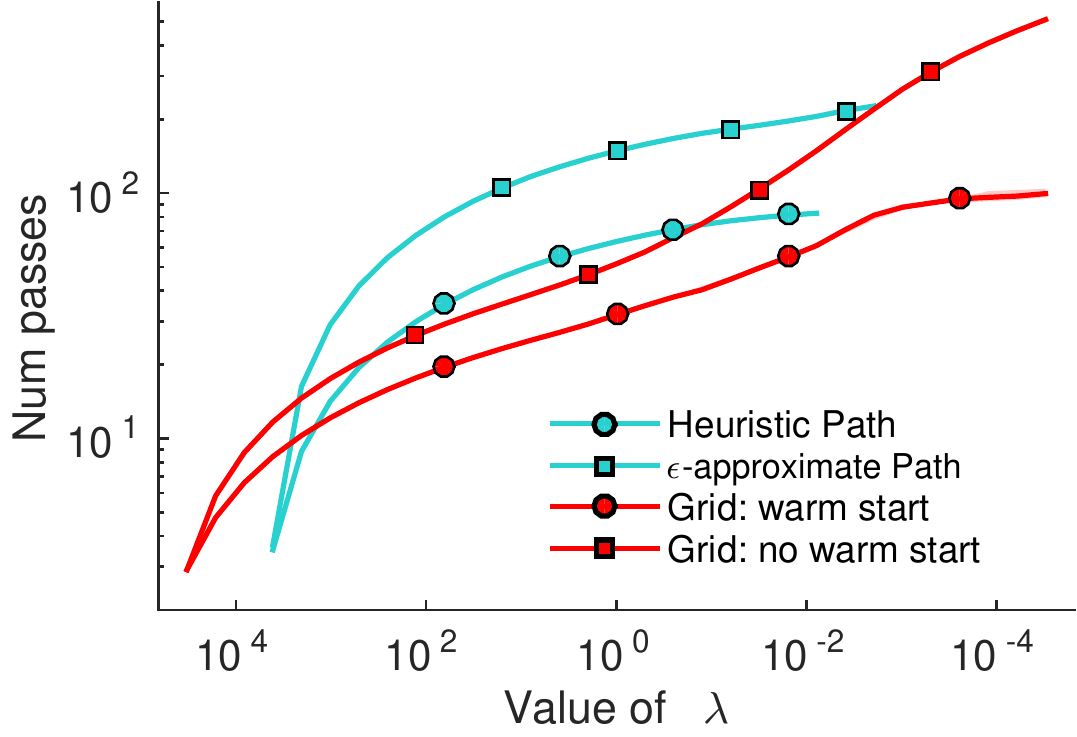}\\
        {\scriptsize BCFW + gap sampling}
        &
        {\scriptsize BCPFW + gap sampling + caching}
        \\[-0.0cm]
    \end{tabular}
    \caption{
        \label{fig:regPath_ocrSmall} OCR-small
        }
\end{subfigure}
\\[5mm]
\begin{subfigure}[b]{0.48\textwidth}
    \begin{tabular}{c@{$\:$}c}
        \includegraphics[width=0.49\textwidth]{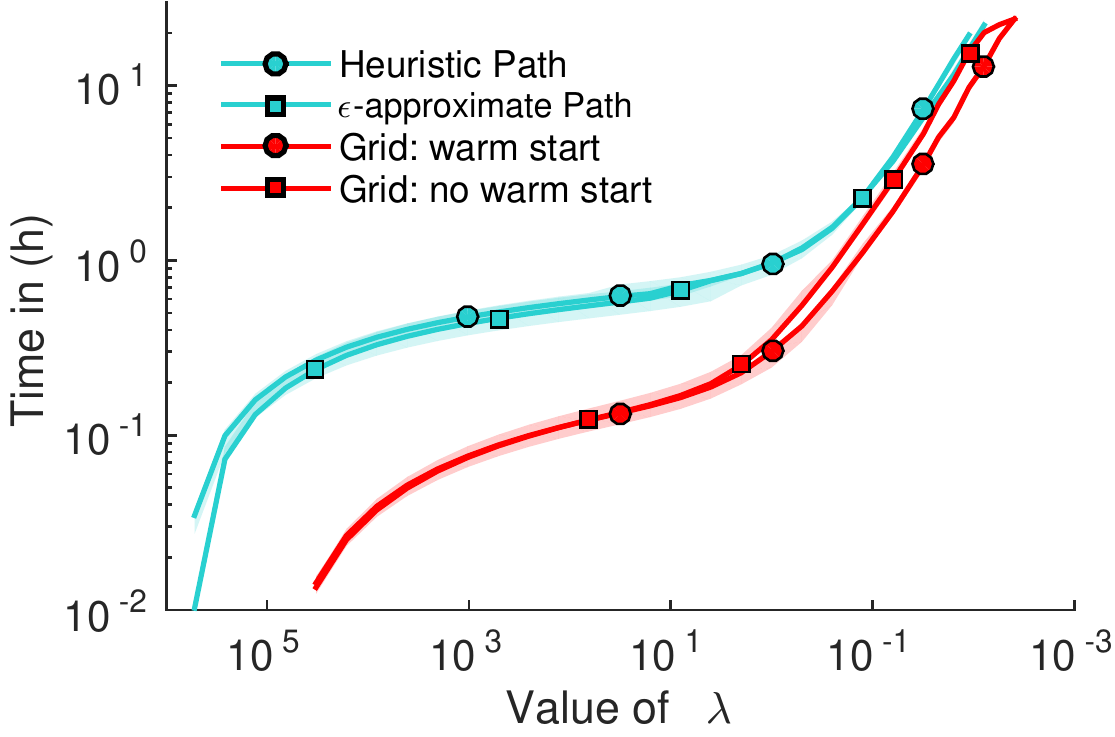} &
        \includegraphics[width=0.49\textwidth]{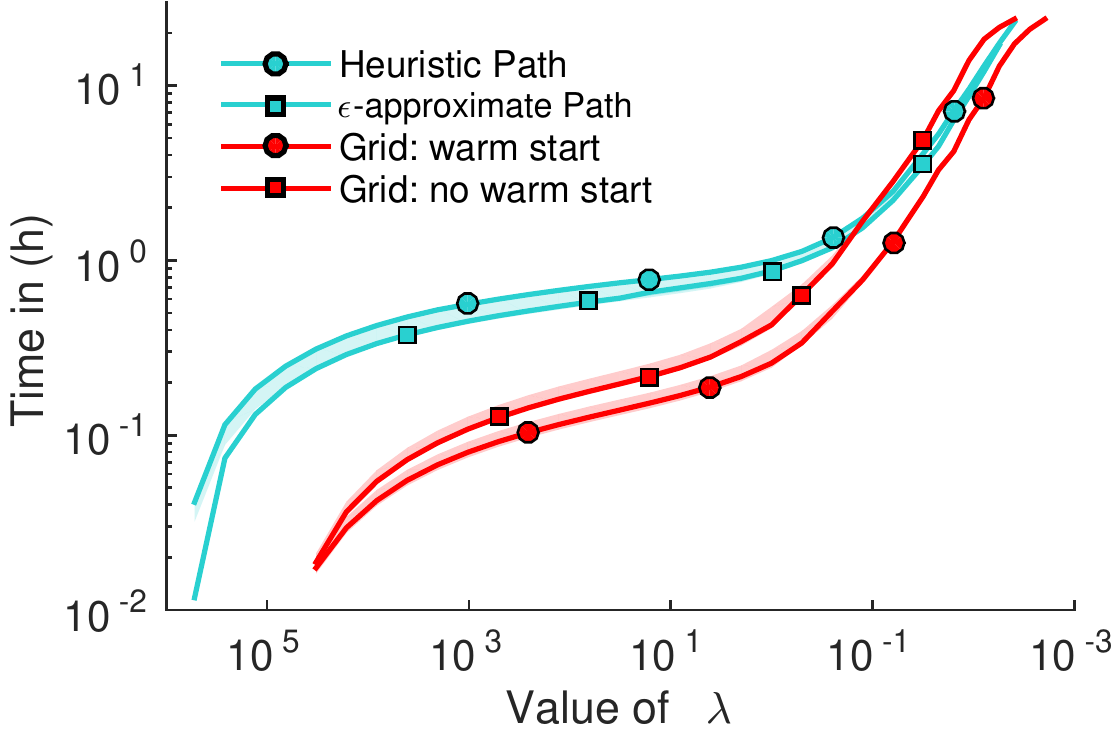}\\
        \includegraphics[width=0.49\textwidth]{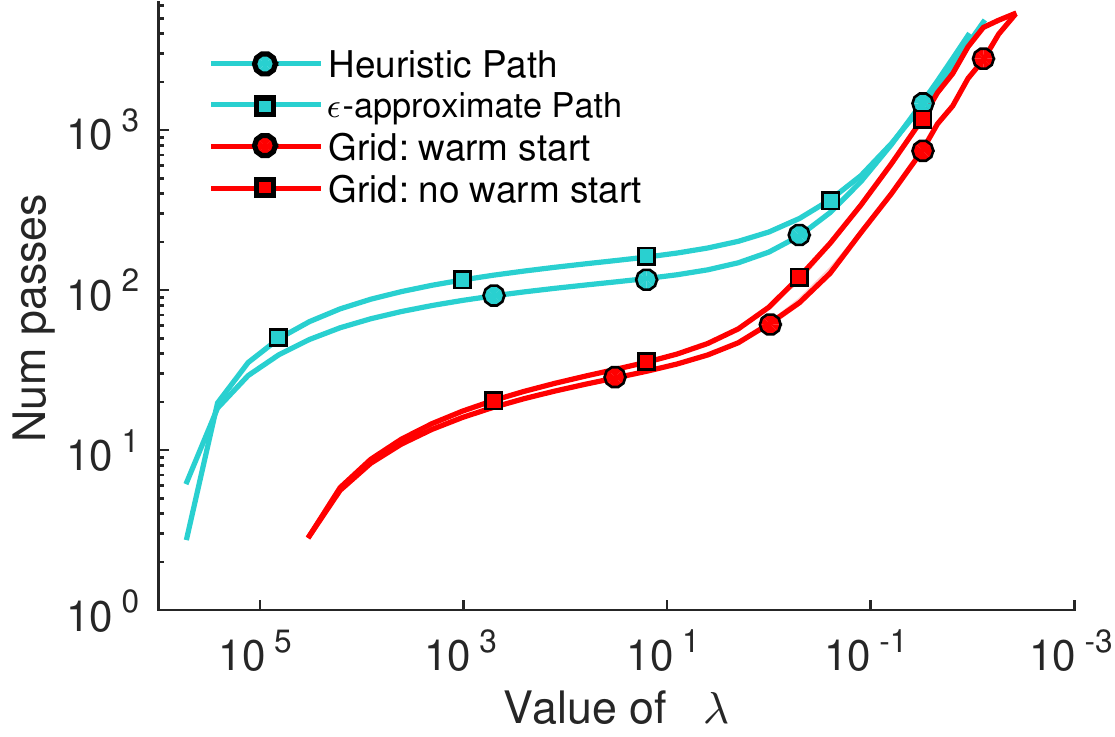} &
        \includegraphics[width=0.49\textwidth]{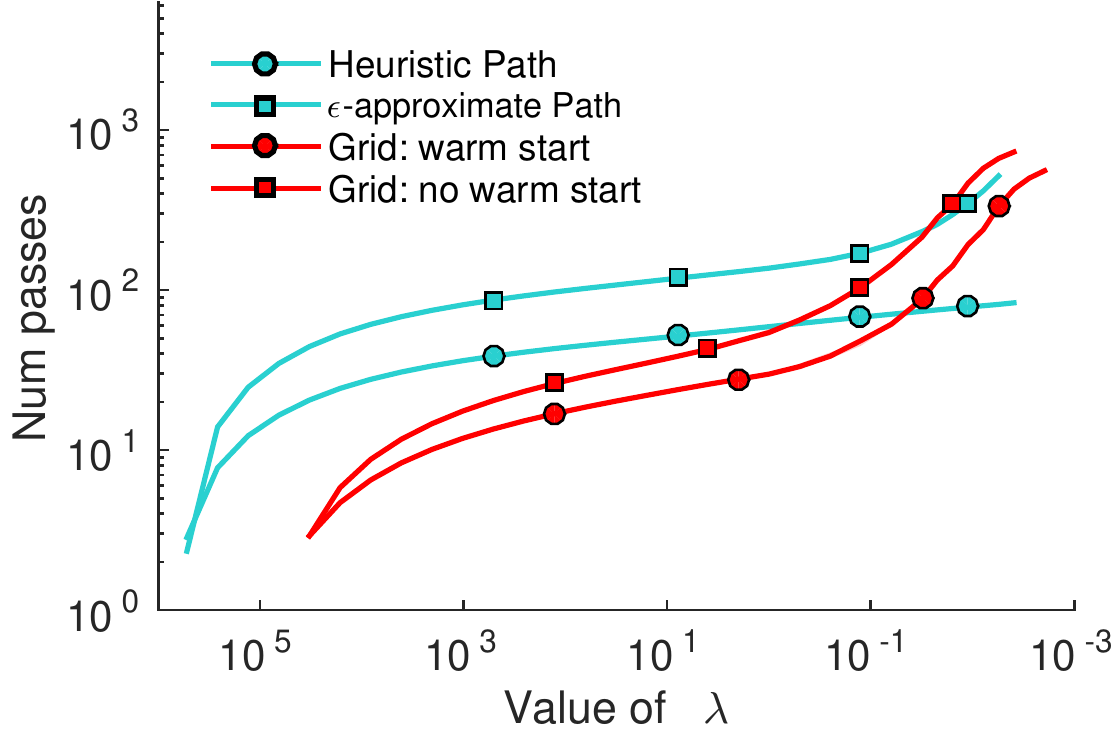}\\
        {\scriptsize BCFW + gap sampling}
        &
        {\scriptsize BCPFW + gap sampling + caching}
        \\[-0.0cm]
    \end{tabular}
    \caption{
        \label{fig:regPath_horseMedium} HorseSeg-medium
        }
\end{subfigure}
\quad
\begin{subfigure}[b]{0.48\textwidth}
    \begin{tabular}{c@{$\:$}c}
        \includegraphics[width=0.49\textwidth]{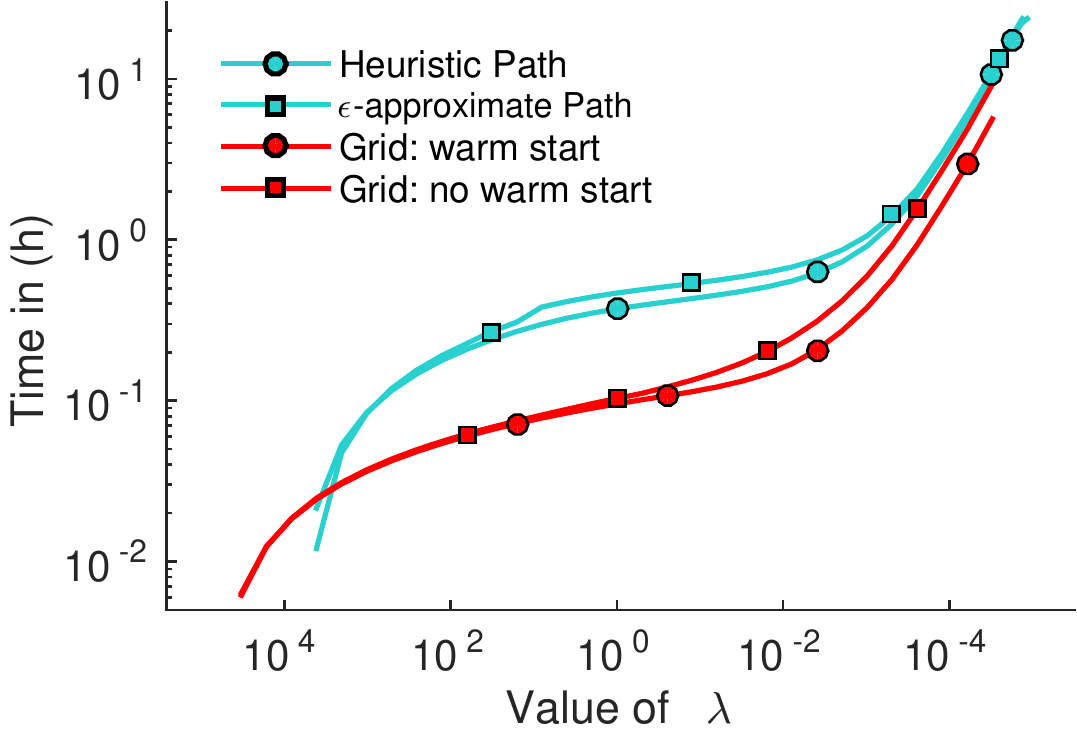} &
        \includegraphics[width=0.49\textwidth]{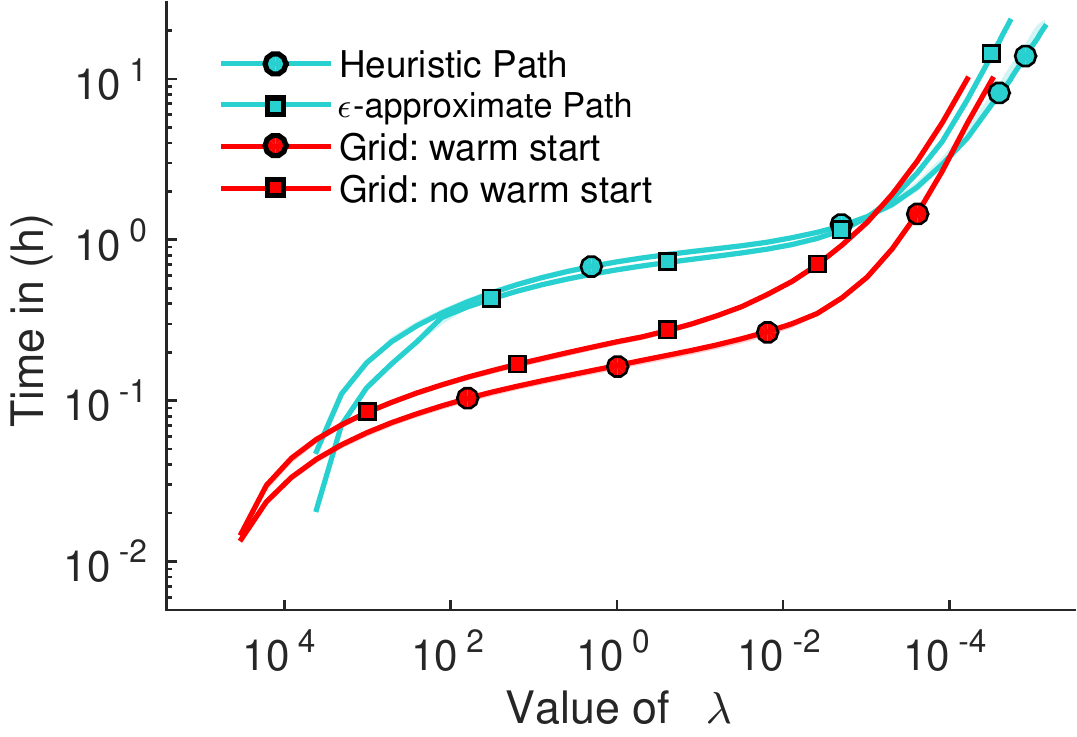}\\
        \includegraphics[width=0.49\textwidth]{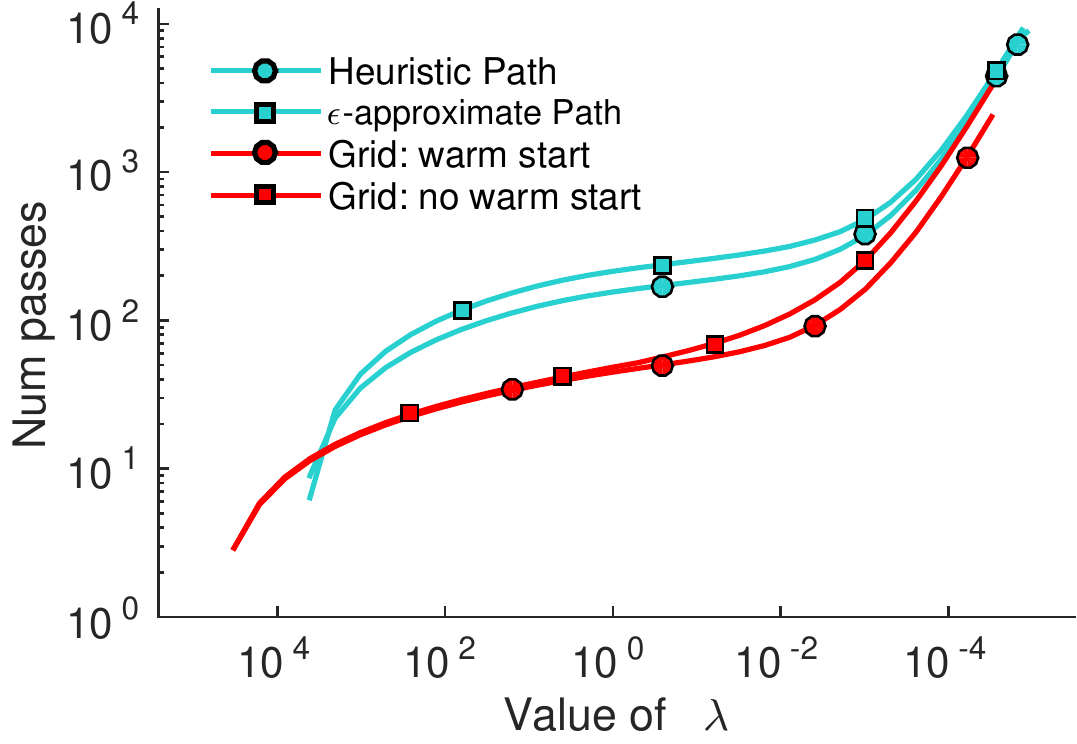} &
        \includegraphics[width=0.49\textwidth]{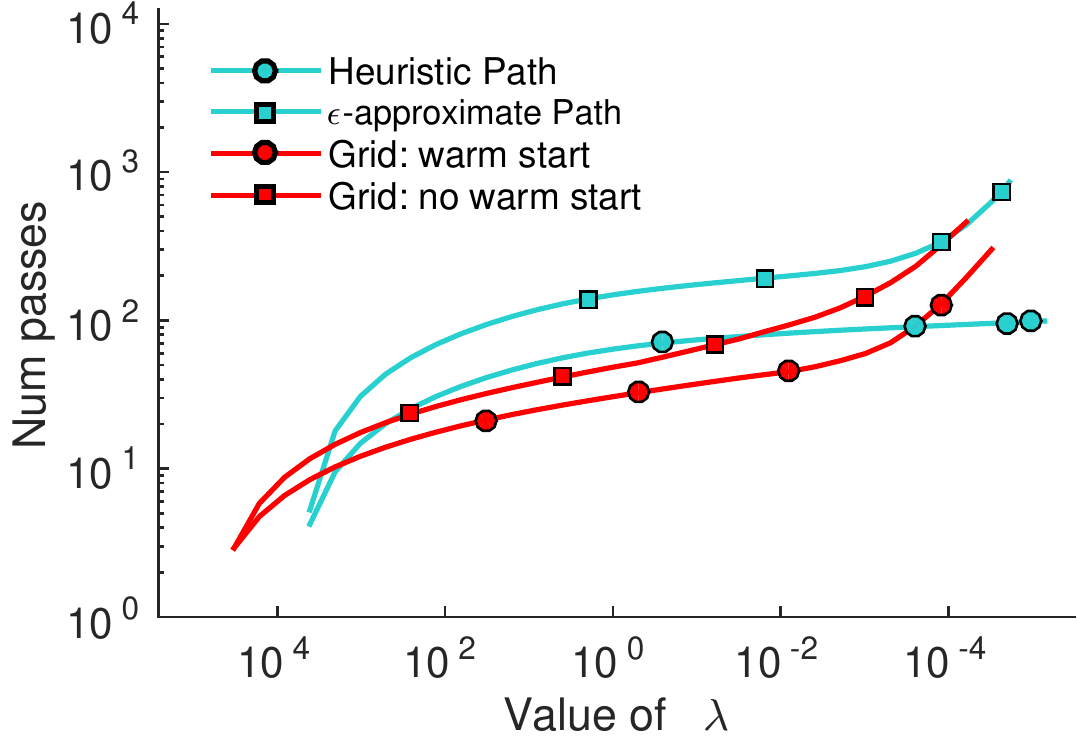}\\
        {\scriptsize BCFW + gap sampling}
        &
        {\scriptsize BCPFW + gap sampling + caching}
        \\[-0.0cm]
    \end{tabular}
    \caption{
        \label{fig:regPath_ocrLarge} OCR-large
        }
\end{subfigure}
\caption{Comparison of regularization path methods. In each subfigure, we compare the $\epsilon$-approximate regularization path against the heuristic path and the grid search with/without warm start for a specific dataset. 
In each subfigure, we report the cumulative running time (top) and the cumulative effective number of passes (bottom) required to get to each value of the regularization parameter~$\regularizerweight$. We report results using two different methods as the SSVM solver: BCFW + gap sampling (left) and BCPFW + gap sampling + caching (right).
    Note that for OCR-large, the time limit of 24 hours was reached for the regularization path methods.
    However, the reached value of $\regularizerweight$ is smaller than the lower boundary of the grid $2^{-15}$.
    For HorseSeg-medium, the time limit of 24 hours was reached for both the regularization path and grid search methods.
    \label{fig:regPath_4datasets}}
\end{figure*}

\clearpage
\bibliographysup{icml2016_bcfw}
\bibliographystylesup{icml2016}

\end{document}